\documentclass[english,a4paper,11pt]{article}
\usepackage[T1]{fontenc}
\usepackage[utf8]{inputenc}
\usepackage{graphicx}
\usepackage{authblk}
\usepackage{hyperref,comment}
\usepackage{amsmath,amssymb,amsfonts,amsthm}%
\usepackage{algorithm}%
\usepackage{algorithmicx}%
\usepackage{algpseudocode}%
\usepackage{siunitx}
\usepackage{adjustbox}
\usepackage{hyperref}
\usepackage{subcaption}
\usepackage{float}

\DeclareMathOperator*{\minimize}{\text{minimize}}
\DeclareMathOperator*{\subjto}{\text{subject to}}

\usepackage{xcolor}

\usepackage{todonotes}
\usepackage{bm}

\newtheorem{lemma}{Lemma}
\newtheorem{theorem}[lemma]{Theorem}%
\makeatletter
\@ifundefined{date}{}{\date{}}
\makeatother
\usepackage{babel}
\begin{document}
\title{\vspace{-4cm}Can Complexity and Uncomputability Explain Intelligence? SuperARC: A Test for Artificial Super Intelligence Based on  Recursive Compression}
\author[1,2]{Alberto Hern\'andez-Espinosa}
\author[1,2]{Luan Ozelim}
\author[1,2,3,4]{Felipe S. Abrah\~ao}
\author[1,2,5,6]{Hector Zenil\thanks{Corresponding author: hector.zenil@kcl.ac.uk}}
\affil[1]{\normalsize\text{ }Oxford Immune Algorithmics, Oxford University Innovation \& London Institute for Healthcare Engineering, U.K.}
\affil[2]{\normalsize\text{ }Algorithmic Dynamics Lab, Center of Molecular Medicine, Karolinska Institute \& King's College London, U.K.}
\affil[3]{\normalsize\text{ }Centre for Logic, Epistemology and the History of Science, University of Campinas (UNICAMP), Brazil.}
\affil[4]{\normalsize\text{ }DEXL, National Laboratory for Scientific Computing (LNCC), Brazil.}
\affil[5]{\normalsize\text{ }Algorithmic Dynamics Lab, Department of Biomedical Computing, School of Biomedical Engineering and Imaging Sciences}
\affil[6]{\normalsize\text{ }King's Institute for Artificial Intelligence, King's College London, U.K.\vspace{-1cm}}
\date{ }
\maketitle

\begin{abstract}
%% Shortened abstract to 150 words for NatCompSci
%\textbf{150-word abstract}\\ %Currently, 162 words
We introduce an increasing-complexity, open-ended, and human-agnostic metric to evaluate foundational and frontier AI models in the context of Artificial General Intelligence (AGI) and Artificial Super Intelligence (ASI) claims. Unlike other tests that rely on human-centric questions and expected answers, or on pattern-matching methods, the test here introduced is grounded on fundamental mathematical areas of randomness and optimal inference. We argue that human-agnostic metrics based on the universal principles established by Algorithmic Information Theory (AIT) formally framing the concepts of model abstraction and prediction offer a powerful metrological framework.  When applied to frontiers models, the leading LLMs outperform most others in multiple tasks, but they do not always do so with their latest model versions, which often regress and appear far from any global maximum or target estimated using the principles of AIT defining a Universal Intelligence (UAI) point and trend in the benchmarking. Conversely, a hybrid neuro-symbolic approach to UAI based on the same principles is shown to outperform frontier specialised prediction models in a simplified but relevant example related to compression-based model abstraction and sequence prediction. Finally, we prove and conclude that predictive power through arbitrary formal theories is directly proportional to compression over the algorithmic space, not the statistical space, and so further AI models' progress can only be achieved in combination with symbolic approaches that LLMs developers are adopting often without acknowledgement or realisation. \\

\end{abstract}

\section{Introduction}\label{sectionIntro}

As we enter an AI test saturation phase, where all AI models claim to be the best or front runners on all available AI tests, new tests ideally orthogonal to current ones have to continue to be developed that can keep up with new models to keep increasingly challenging them. 

%Even in instances where we may be limited, flawed, irrational, or engaged in often shallow human activities like chatting or evaluating artificial systems on their (in)capability to perform mundane chores like washing dishes. 
%Even in the case where there may be inherent cognitive biases toward overrating our own intelligence to the detriment of efforts to devise a possibly more objective and quantitative measure of intelligence, 

Historically, humans have been heavily biased to believe that the way humans think and act represents the acme of intelligence, and there remains the philosophical and scientific question of the extent to which we can achieve more objective or less human-centric measures of intelligence.

The impressive performance of Large Language Models (LLMs) as language processing and generation tools evinces that language and other areas of human intelligence may be overrated and can be, in fact, more dependent than we thought on aspects of memorisation and statistical pattern matching.

%One of the first metrics for intelligence was introduced by Charles Spearman in 1904~\cite{spearman}. He proposed specific tests called `s' that would each contribute to a general intelligence test under the name `g', representing the common cognitive ability underlying performance in various mental tasks. Specific intelligences that contribute to the estimation of the g factor are verbal comprehension, perceptual reasoning, working memory, processing speed, quantitative reasoning, abstract reasoning, spatial ability, memory retrieval, auditory processing, and fluid reasoning. Some LLM benchmarks test for different factors, with several benchmarks based on correct answers versus hallucinations; some of which are very human-centric metrics related to human's biological peculiarities and shared history.

A common psychological perspective sees intelligence through the lens of IQ tests; one of the first was the g-factor, a psychometric construct introduced by Spearman~\cite{spearman} that quantifies the positive correlations between cognitive abilities. 
This framework is consistently linked to a human-centric perspective of what intelligence is and, therefore, biased towards circular reasoning. 
In the context of AI, some LLM benchmarks test for different factors, with several benchmarks based on correct answers versus hallucinations; some of which are also very human-centric metrics related to humans' biological peculiarities and shared history.

%The concept of intelligence testing has been explored by researchers in different fields, including starting with machine intelligence rather than biological or human intelligence~\cite{ploscomp,zenilanimal,zenilrobots,bibid,orallo1998}. 
Some scholars argue that intelligence can be objectively defined through tests that evaluate specific computational abilities essential to demonstrate intelligent behaviour, rather than trying to define intelligence itself in absolute terms~\cite{ploscomp,orallo1998,zenilanimal,zenilrobots,bibid}. 
This perspective shifts the focus from an abstract or philosophical definition to a practical, measurable framework assessing an entity's capacity for problem-solving, pattern recognition, and adaptive learning within a structured system.

This reflects an operational turn in the study of intelligence, emphasising the design of formal benchmarks and quantifiable metrics.  However, this approach is not without philosophical challenges. 
By reducing intelligence to observable outputs, it risks overlooking the role of internal representation, consciousness, or semantic understanding---dimensions emphasised in critiques like Searle's. 

An approach toward tackling such issues is to ground intelligence metrics in more fundamental notions of computation and mathematics. 
For example, the concepts of randomness, prediction and inference as defined by Gregory Chaitin, Andrey Kolmogorov, and Ray Solomonoff. Chaitin~\cite{Chaitin1982} proposed that formal definitions of intelligence and its components should be based on algorithmic complexity, a formal mathematical theory able to define the concept of randomness as opposed to intelligence.
Similarly, Solomonoff~\cite{Solomonoff1986} advanced the idea of evaluating intelligence through algorithmic probability, laying the foundation for optimal prediction frameworks (or universal ``Bayesian'' inference). These approaches further motivated approaches such as Hutter's AIXI~\cite{hutter2005universal} in an attempt to reconcile objective evaluation with theoretical generality in the context of learning.
Algorithmic complexity, algorithmic probability, and algorithmic randomness comprise the most important concepts in algorithmic information theory (AIT)~\cite{Calude2002,Downey2010,kolmobook,Chaitin2004} and provide the accepted mathematical definitions of randomness and optimal inference (induction and abduction) going beyond simplistic statistical tests based on methods such as GZIP or LZW, popular pattern-matching compression methods that are more closely related to Shannon entropy than to model synthetization and predictive inference.%kolmogorov,chaitin,solo,Calude2002,Downey2010,kolmobook}.

At recent public events, speaking about the foundations of AI and AGI, some leaders in the AI industry have drawn strong parallels between algorithmic complexity, data compression, and AI~\cite{sutskever2023video,ainews2021ai}. 
Although these terminologies, such as AGI and ASI, are currently loosely defined in the scientific literature, these claims and the current understanding make the connection between LLMs (or any other generative AI), algorithmic complexity, and data compression clearer and more explicit, even calling it fundamental for general and super intelligence, artificial or natural. 

Based on these arguments connecting intelligence to recursive compression~\cite{bibid}, some tests for machine, human, and non-human entities have been proposed in~\cite{Legg2007,zenil2013turingtest,zenilrobots}.
Section~\ref{sectionCompCompre} presents a reflection on that property of intelligence to involve the identification of recursive patterns, planning from prediction, and the generation of concise explanations for observed complex phenomena. 
%Such an approach provides a unified framework for understanding both human and artificial intelligence, moving beyond traditional tests and philosophical debates to a measurable and practical foundation. 
Recursive compression here means the ability to represent an observation in a condensed manner by taking advantage of aspects of the data's regularities beyond statistical pattern matching. This is, by selecting and keeping as many as possible the features that make the explanation executable and predictive of the explanandum future states.

%One idea expressed by Sutskever~\cite{sutskever2023video}, is that Stochastic Gradient Descent (SGD), a main iterative optimisation algorithm for optimising an objective function used to train models in machine learning (ML) and artificial intelligence (AI), is a practical approximation to finding a computer program that compresses the encoding data in the search space and performs a type of `Kolmogorov search' to find an implicit small computer program embedded in the weights of a `soft computer' or a neural network such as a large Transformer.
Despite the interesting theoretical arguments that could be drawn from these connections, one argument is that they would only be valid under idealised conditions (unbounded data access/storage, perfect optimisation, appropriate inductive biases), which are rarely met in practice. As seen in real-world problems, even simple datasets with specific distributions can lead to optimisation toward local minima that do not correspond to minimal algorithmic descriptions.
%Encoders are effectively (lossy or lossless) compression heuristics and, therefore, deeply connected to algorithmic complexity via compression. 

Closely related ideas are also in evidence in Schmidhuber's G\"{o}del machines~\cite{schmidhuber2007godel} work and Hutter's AIXI~\cite{hutter2005universal} based on Levin's search~\cite{levin} and the principles of algorithmic probability~\cite{solo,mdl,kolmobook}.
Similarly to a test proposed in~\cite{HernndezOrallo2015}, a benchmark designed to evaluate conceptual understanding in machine learning models was proposed, consisting of a diverse set of tasks that indirectly assess a model's capacity for abstraction, requiring it to generalise beyond memorisation~\cite{fact}. 
These tasks challenge models to reason both interpolatively (by making sense of patterns within observed data) and extrapolatively (by extending learnt principles to novel scenarios). 
Although interesting and a first approach, the test lacked robust foundations of algorithmic information, nor were they applied to frontier models.
See also Section~\ref{sectionCompCompre} for further theoretical challenges and developments.
In a previous work, we successfully explored some of these ideas, proving that we can perform this search on non-differentiable spaces using metrics purely based on algorithmic complexity to search for those programs in model space, making the previously considered fundamental requirement of differentiability redundant~\cite{hernandez2021algorithmicml}.

Building on previous work reporting applications to various fields ranging from cell and molecular biology to genetics~\cite{iscience,nmi} to biosignatures to animal and human behaviour~\cite{ploscomp,zenilanimal,zenilrobots}, here we introduce a quantitative test for any AI model that aims at universal and agnostic optimisation with an application to LLMs fully framed in terms of the principles and foundations of AIT together with perturbation analysis from Algorithmic Information Dynamics~\cite{zenil2020algorithmic,zenilbook1,zenilbook2} (see Section~\ref{sectionAIDforSuperARC} in the Supplementary Information and Section~\ref{sectionCompCompre}). 
Our framework is related to tests such as the ARC-AGI tests/ challenge~\cite{Chollet2019MeasureIntelligence}, but it is agnostic to: 
the chosen set of problems, since it does not pick specific test cases;
the underlying formal theories that define or characterise the evaluation tools; the chosen observer/evaluating agents;
and the chosen interacting agents or external input.
It avoids the devise of a metric that is fixed whose theoretical principles are not assumed to evolve together with the tested subjects, thereby allowing the test to become the target and no longer useful (see also Section~\ref{sectionBenchContami} in the Sup. Inf. and Section~\ref{sectionResults})---therefore, a test for what can be understood as AGI and ASI. 

%As we discuss further in Sections~\ref{sectionAIDforSuperARC} and~\ref{sectionAGIandASI}, when we use terms like AGI and ASI in this paper, we generally refer to systems capable of both abstraction, generalisation, and prediction across arbitrary formal domains. 
%In this manner, we particularly focus on the aspects of intelligent systems that enable them to devise \emph{novel} scientific or mathematical theories.
While we do not assume that connections to compression as model abstraction and prediction as planning are necessary but not necessarily sufficient for general intelligence, these qualities have recently been strongly associated with AI, AGI, and ASI~\cite{lecun,lecun2}.
%\todo{(R1)(R2) On compression as necessary (but not necessarily sufficient) condition for comprehension.} 

Our framework adopts a specific theoretical perspective on intelligence but surely does not capture all aspects of human cognition. The test here introduced is meant to challenge aspects of AI (e.g., LLMs) by putting forward mathematical theory and methods related to the properties of intelligence believed to be key for intelligence, in particular AGI or ASI such as model abstraction and planning as in model synthesis (new explanatory models) and as in its recursive prediction capabilities. Although we propose a method in simplified contexts, this is without any loss of generality to any other type of data.

%Unlike other algorithmic universal methods---as discussed in Section~\ref{sectionAIDforSuperARC} in the Sup. Inf. and discussed in Section~\ref{sectionCompCompre}---that assume a fixed observer in the application of universal induction, data compression over the whole algorithmic space, abstraction of irreducibly new theories, and prediction of future outcomes are inextricably linked as a result of the external input given by (or through the lens of) any formal theory an arbitrary observer or interacting agent may devise.

We claim that the feature of increasing complexity makes the test robust to benchmark contamination and test targeting and can account for improvements due to external intervention, while its ultimate uncomputability nature provides the desired open-endedness for a feature likely to be as complex as what is trying to capture and evaluate. In other words, an equally complex and open test for an equally complex and open attribute, intelligence.

\section{Results}\label{sectionResults}

The tests were inspired by, and based on, two methods called \emph{Coding Theorem} and \emph{Block Decomposition} methods (CTM and BDM)~\cite{ctm3,soler2014calculating,bdm}, which use a composition of pattern-matching and running a very large set of small computer programs to approximate a distribution allowing the estimation of the algorithmic complexity of short objects providing insights into the minimal description length of objects such as strings, sequences or images.

These methods have been applied to humans before on similar tasks and the same tools (CTM/BDM), showing that the methods can capture aspects that other intelligence tests based on pattern matching fail or require much more ad hoc information and assumptions to reproduce them~\cite{ploscomp}. Humans showed some predictive capabilities that could be quantified only with CTM/BDM and not with other statistical tools. This led CTM and BDM to be widely used today in the psychometric testing space by multiple groups~\cite{newpaper1,newpaper2,newpaper3}. 

Among the results in~\cite{ploscomp}, people showed that at 25 years of age they had the highest ability to identify and produce the highest random complexity when required and decreased when older. More experiments should be conducted to test human's recursive prediction capabilities, but what the article showed~\cite{ploscomp} was that people built and had better perception models for producing and identifying random versus non-random content when they were 25 than at any other age and therefore were more sensitive to identifying less predictive data, all of which are compatible with current knowledge of human cognitive trends and capabilities.

While we do not necessarily expect humans to perform well on predictive tests such as those introduced here, humans are, in principle, mechanistically capable of solving them. As such, these tasks are arguably within the definition of AGI as a system possessing the full range of human capabilities, without time constraints. Moreover, in the context of ASI, the natural assumption is that of optimal prediction, for which Algorithmic Information Theory (AIT)--used and exploited here for testing purposes--is the accepted mathematical framework.

\subsection{Next-digit Prediction Task with Binary and Non-binary Sequences}

The objective of this experiment is to test a fundamental property of LLMs, that is, the prediction of the next token, within the context of Algorithmic Information Theory (AIT), the accepted mathematical theory governing optimal prediction, as introduced in Section~\ref{sectionAITIntel}.

We tasked Large Language Models (LLMs) specializing in time series prediction with predicting the final digit of both non-binary sequences and binary sequences, the latter of which were categorised as either random or ``climber'' sequences. The results of the experiment involving binary sequences are presented in Figure~\ref{predLastBinSeq}.

We call ``climbers'' those sequences that have some recursive properties that make them stand out and have been properly ranked as of lower complexity given their recursive properties according to CTM/BDM (see Section~\ref{climbers}). As shown in Figure~\ref{predLastBinSeq}, in the case of what we call ``climbers'', Lag-Llama achieved the best performance, with 70\% precision, while TimeGPT---1 and Chronos barely reached 50\% precision, while CTM/BDM was used as a gold standard.

However, for random sequences, which are considered highly `complex' in this context, all models performed similarly, showing limited predictive power, as expected. These results suggest that, given
the binary nature of the sequences, the models had a 50\% chance of predictive success, effectively reducing the task to guessing, yet some performance difference was noticed for non random cases, indicating some predictive ability. However, their performance aligns with broader research that indicates that LLM models do not effectively capture sequential dependencies or complex patterns inherent in time series data. As highlighted by Tan et al.~\cite{tan2024language}, despite their computational intensity, LLMs often fail to outperform simpler models, particularly when there is high complexity or randomness in the data.

A comparable analysis was conducted using LLMs specialised in time-series data, using non-binary sequences of increasing complexity. In this test, a specific percentage of the final numbers in each sequence was required to be predicted. Three distinct metrics were utilised: general similarity, sort similarity, and the Levenshtein distance (refer to Section~\ref{nextDigitPredictionSubsection} for its definition). Figure~\ref{TimeseriesLLMPlots} presents the results, where sort similarity and general similarity exhibit closely aligned trends. This indicates that the predictive accuracy of LLM models, even when fine-tuned for numerical series, diminishes as the complexity of the sequences increases. The resemblance between sort similarity and general similarity implies that while predictions may include some of the expected numbers, their correct order remains equally critical and may not always be achieved.

\begin{center}
\begin{figure}[H]
\makebox[\textwidth][c]{\includegraphics[width=0.6
\textwidth]{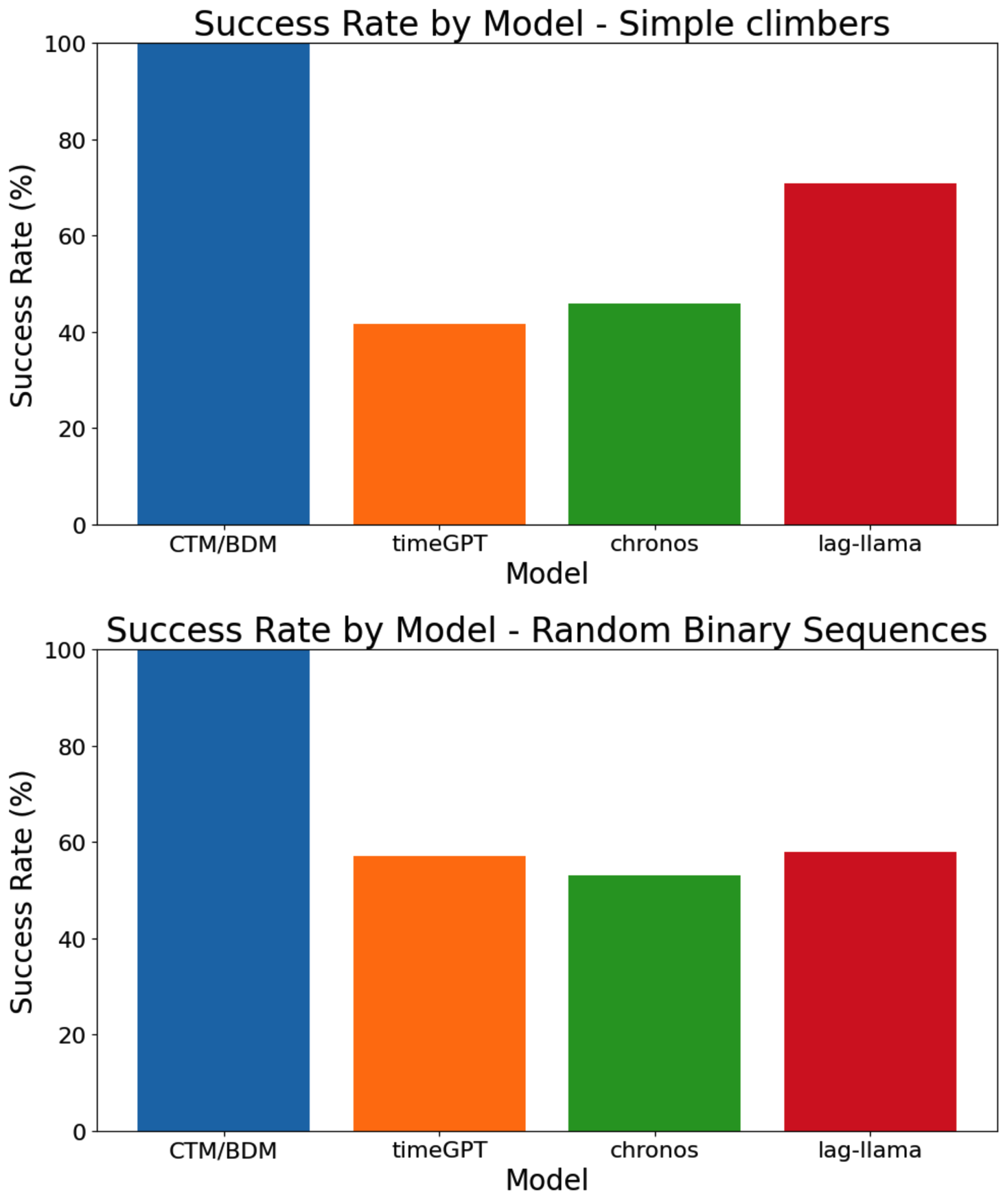}}%
\caption{Percentage of accuracy on binary ``climbers'' and random binary sequences by LLM models specialising in time series prediction compared with BDM. That climbers (up) were better predicted is expected from models that are able to intrinsically characterise and better predict simpler sequences. Sequence prediction is a fundamental problem in science, from genetics to protein folding in biology to digital twin technology in medicine and healthcare.}
\label{predLastBinSeq}
\end{figure}
\end{center}

This observation is corroborated by the findings from the Levenshtein distance metric, which quantifies the minimum number of single-character edits (insertions, deletions, or substitutions) required to transform one sequence into another. As the complexity of the sequences rises, so does the Levenshtein distance, further confirming that predictive accuracy deteriorates with increasing complexity. 

Figure~\ref{BDMComparison} shows an increase in complexity as expected, given the design of each group of generated sequences. The plot suggests that BDM can capture (and can generate) better complexity and randomness, since its values increase more consistently as complexity increases, unlike other measures. Shannon entropy-based measures (and cognates) can account for statistical randomness only. Compression algorithms, for example, decrease as complexity increases, becoming more difficult to find regularities and increasing compression length as a function of complexity growth.

\begin{figure}[h]
\begin{center}
\makebox[\textwidth][c]{\includegraphics[width=.75
\textwidth]{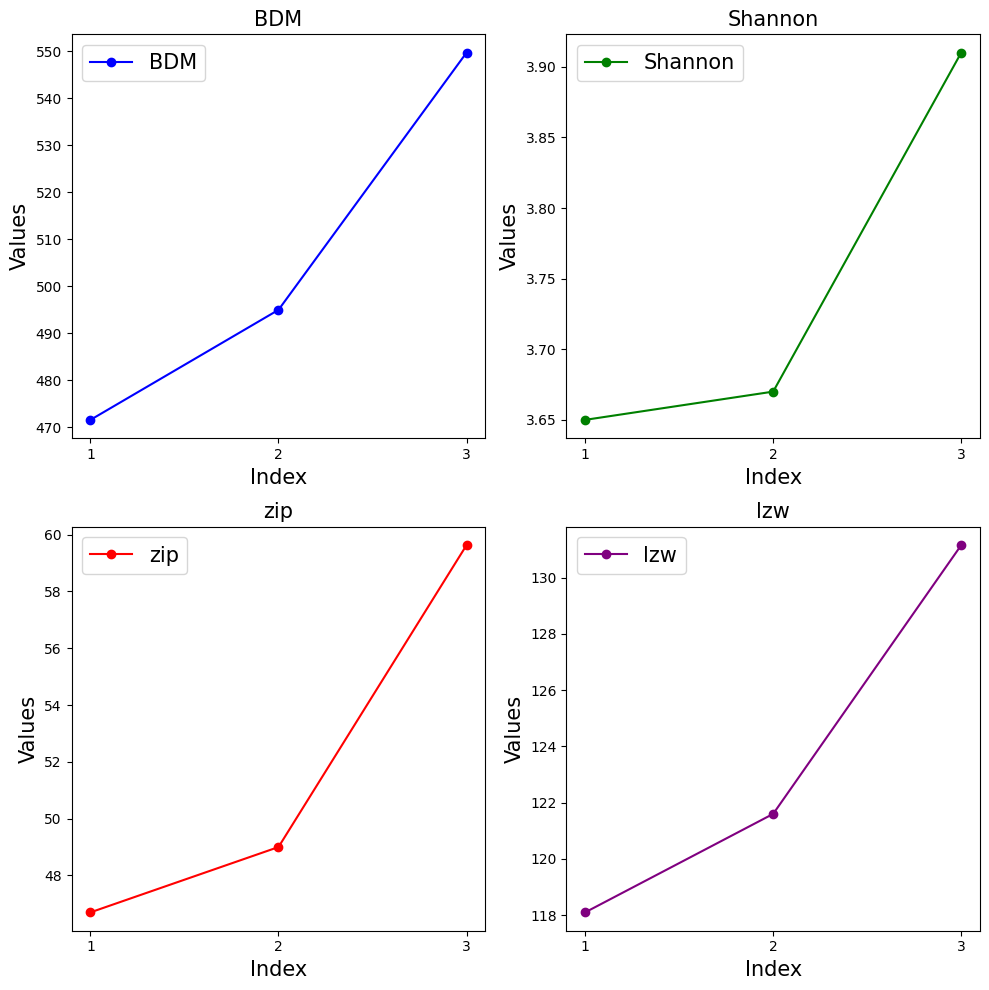}}%
\end{center}
\caption{Quantitative Agreement of Monotonic Sequence Increase of Complexity: Comparison of BDM, Shannon Entropy, average length of Zip and LZW over the time series generated to test LLMs. Sequences chosen for each complexity class follow a pattern of increasing complexity in all cases, according to both statistical and algorithmic measures, and are used to build the testing sets, divided into three complexity groups, against which LLMs will be assessed. }
\label{BDMComparison}
\end{figure}

\subsection{Free-form generation task with non-binary sequences}

A subsequent analysis focused on the free-form test, where LLMs were given complete freedom to generate any model or formula capable of producing target sequences of increasing complexity.

Figure~\ref{MetricsForFormulaGeneration} shows the plots of complexity-related metrics for the models and formulas generated by LLMs used in this research. The metrics evaluated include the length of the LZW-compressed model, the length of the ZIP-compressed model, the BDM of both the uncompressed model and its LZW and ZIP-compressed forms, and the Shannon entropy of the model.

The plots reveal a clear positive correlation between model complexity and the metric values as the complexity of the target numerical sequence increases. Specifically, as the complexity of the sequence grows, the length of both LZW and ZIP-compressed representations increases, suggesting that the LLM-generated models become larger and less compressible. This indicates that the models provided by the LLMs become unable to compress and then to understand the logic behind sequences, giving
as a result the sequence itself.

The BDM values (for the raw, LZW, and ZIP models) also exhibit an incremental trend, further supporting the observation that the LLMs generate less structured models when faced with more intricate sequences. Additionally, the Shannon entropy values rise with complexity, highlighting the increase in unpredictability or information content within the models as they attempt to approximate more complex patterns.

These findings suggest that the LLMs struggle to produce compact or efficient models as the complexity of the target sequence increases. The uncompressed models generated by the LLMs become longer and less structured, as indicated by the rise in all metrics. This reflects a limitation in the LLMs' ability to discover or generate concise, elegant models for more complex sequences. Instead of producing simpler, more generalisable formulas, the LLMs resort to more convoluted representations, indicating a lack of sophistication in their capacity to identify or generate models that optimally balance complexity and brevity.

\subsubsection{Emergent abilities}
Another experiment aimed to evaluate characteristics recently attributed to LLMs, particularly their so-called emergent abilities, which include innovation, discovery, and improvement. 
(See also Sections~\ref{sectionCompCompre} and~\ref{sectionAIDforSuperARC}).
These attributes have been claimed to enable LLMs to perform at levels comparable to the human top 1\% in fluency and originality, as suggested by Zhao et al. in their assessment of creativity in artificial intelligence systems~\cite{zhao2024assessingCreativity}.

The experiment tested these claims by challenging LLMs to generate multiple, diverse approaches to reproducing non-binary sequences of varying complexity. The underlying rationale was that originality often stems from the ability to perceive problems in new, unexpected ways. Thus, the test focused on measuring the variety and creativity of outputs, as well as the models' capacity to discover innovative or unconventional solutions.

Two distinct tasks were designed for this evaluation. In the first, models were asked to create any type of formula or mathematical model capable of replicating the target sequences. In the second, models were tasked with writing Python scripts to achieve the same goal. By incorporating these variations, the experiment sought to assess the models' adaptability, computational reasoning, and creative potential across different problem-solving paradigms.

The results are shown in Figure~\ref{MultiFormulaeResults} and Figure~\ref{MultiScriptResults} where the following classification of cases was used:

\begin{enumerate}
\item \emph{Known Sequences:} using standard algorithms such as Fibonacci or primes.

\item \emph{Pure Math:} using mathematical operations without predefined sequence knowledge.

\item \emph{Not Found:} inability to produce outputs.

\item \emph{Print Scripts:} (only for script generation) trivial solutions directly printing the target sequence.
\end{enumerate}

When it came to the production of different models or formula tests, while Gemini, Claude-3.5-Sonnet, and ChatGPT-1o performed relatively well, they ultimately shared the same core limitations as other models. In contrast, Meta and Mistral consistently underperformed, exposing disparities in baseline capabilities among LLMs.

\subsection{Code generation task with non-binary sequences}

For this experiment, one of the main metrics we measured was accuracy, which refers to the proportion of programs in different programming languages generated by ChatGPT that, after compilation and/or execution, produce the target sequence of digits.  Figure~\ref{print&CorrectAnswers} (top) shows that correct programs are more common at the lowest levels of complexity, with some minor exceptions. Figure~\ref{NoCompressPercentage} (top), on the other hand, shows the distribution of print cases by language and complexity level. They support the earlier observation that correctness in many instances is linked to a lack of compression.

Figure~\ref{print&CorrectAnswers} in the Sup Inf. (bottom) shows the distribution of correct instances by sequence and by programming language generated by ChatGPT. The different programming languages are shown in coloured rows. On the right-hand side, the
percentage of correct instances. At the top, the number of programming
languages that overlap or solve the same problems correctly and, at the bottom, the extent of the overlap. 
For example, 5 languages solve the same 20 of 120 problems.

According to the results (\ref{print&CorrectAnswers} top), the vast majority of correct cases are \texttt{print} failing to compress the sequences. This indicates that in most instances where the system correctly identifies a sequence, it does so by simply outputting the sequence as is, without any attempt at compression.

A second test performed to evaluate compression was based on the no-compression percentage. According to this metric, a compressed---and therefore, comprehended---sequence could be expressed as a general (and ideally short) program. \texttt{Print} cases are considered here to have 100\% non-compression, since they involve displaying the original sequence as is, which in our test is synonymous with not understanding the sequence.

\begin{center}
\begin{figure}
\begin{subfigure}{\textwidth}
  \centering
  \includegraphics[width=1\linewidth]{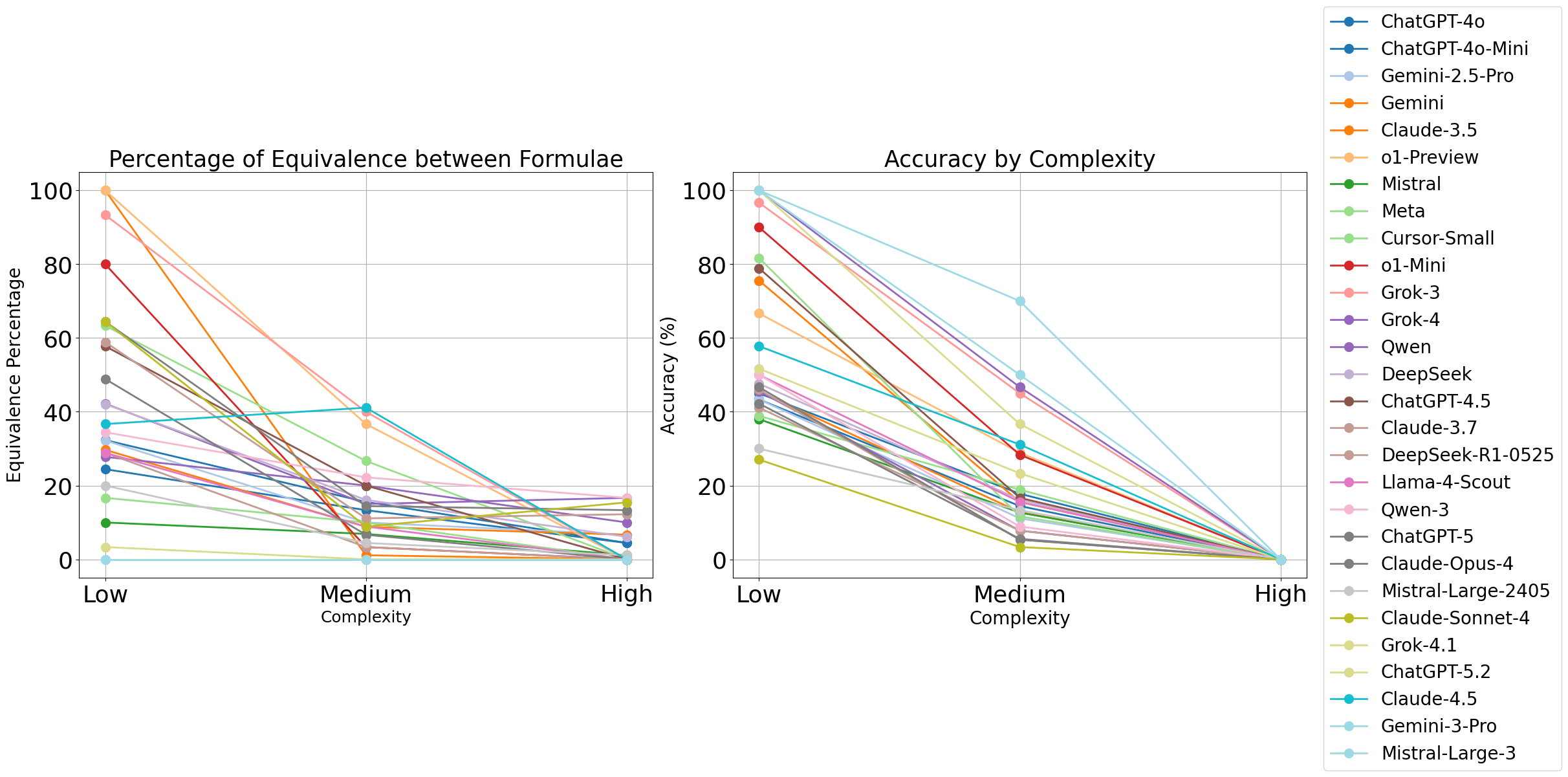} 
\end{subfigure}%

\begin{subfigure}{\textwidth}
  \centering
  \includegraphics[width=1\linewidth]{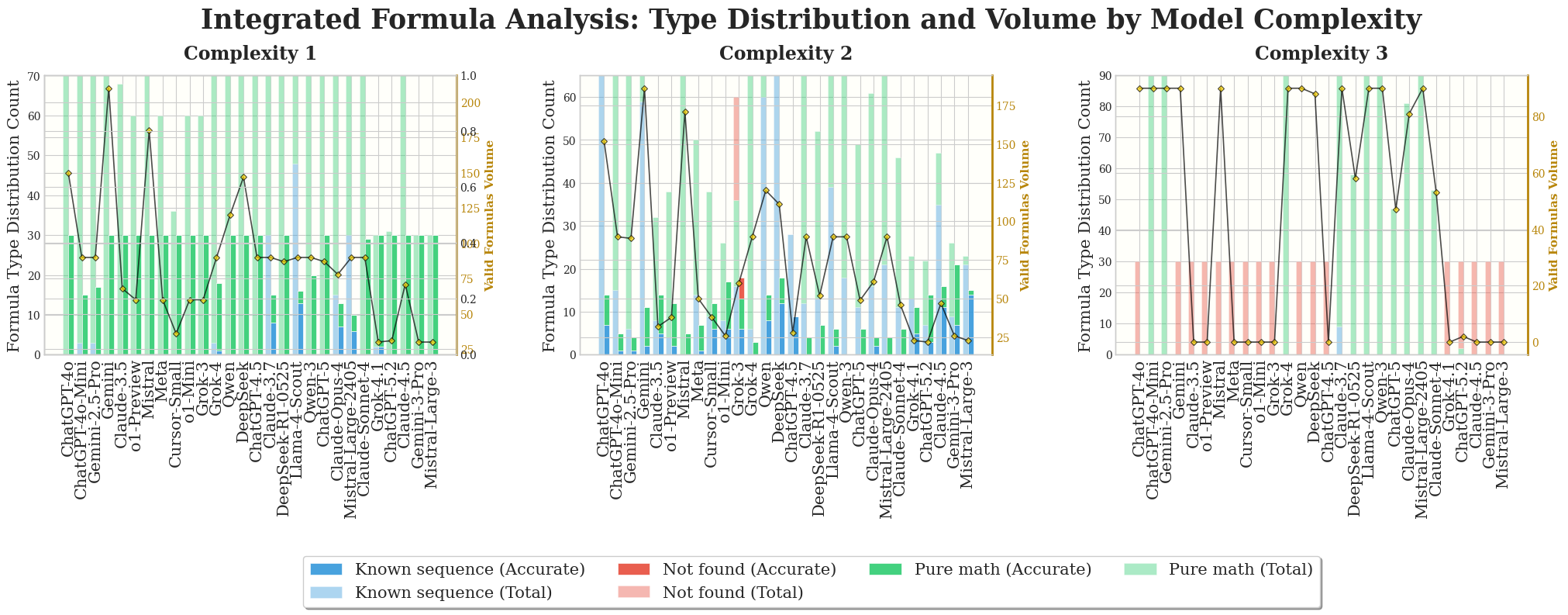}
\end{subfigure}
\caption{Comprehensive analysis of formulae generation for numerical sequences of increasing complexity. Top left: Percentage of equivalence between generated formulae, measuring output similarity and solution diversity. Top right: Accuracy rates showing correct replication of target numeric sequences across complexity levels. Bottom: Integrated view combining formula generation volume (gold line, secondary axis) with type distribution among both total (lighter bars) and accurate (darker bars) responses, categorised as known sequences (blue), pure mathematical expressions (green), and not found (red). The results demonstrate a direct correlation between sequence complexity and diminished model performance, with particularly stark degradation in equivalence rates suggesting limited solution diversity. The integrated bottom panel reveals that whilst models may generate valid formulae at lower complexities, the proportion of accurate responses declines precipitously, and reliance on known sequences dominates over novel mathematical reasoning. These limitations are especially pronounced in contexts permitting complete freedom to discover diverse yet correct solutions, underscoring an absence of genuine creativity and mathematical understanding, attributes often mistakenly attributed to these models~\cite{zhao2024assessingCreativity}. Notably, newer versions of ChatGPT-o1, Grok, and Gemini performed worse than their preview iterations (see Supplementary Information).}
\label{MultiFormulaeResults}
\end{figure}
\end{center}

\begin{center}
\begin{figure}
\begin{subfigure}{\textwidth}
  \centering
  \includegraphics[width=1\linewidth]{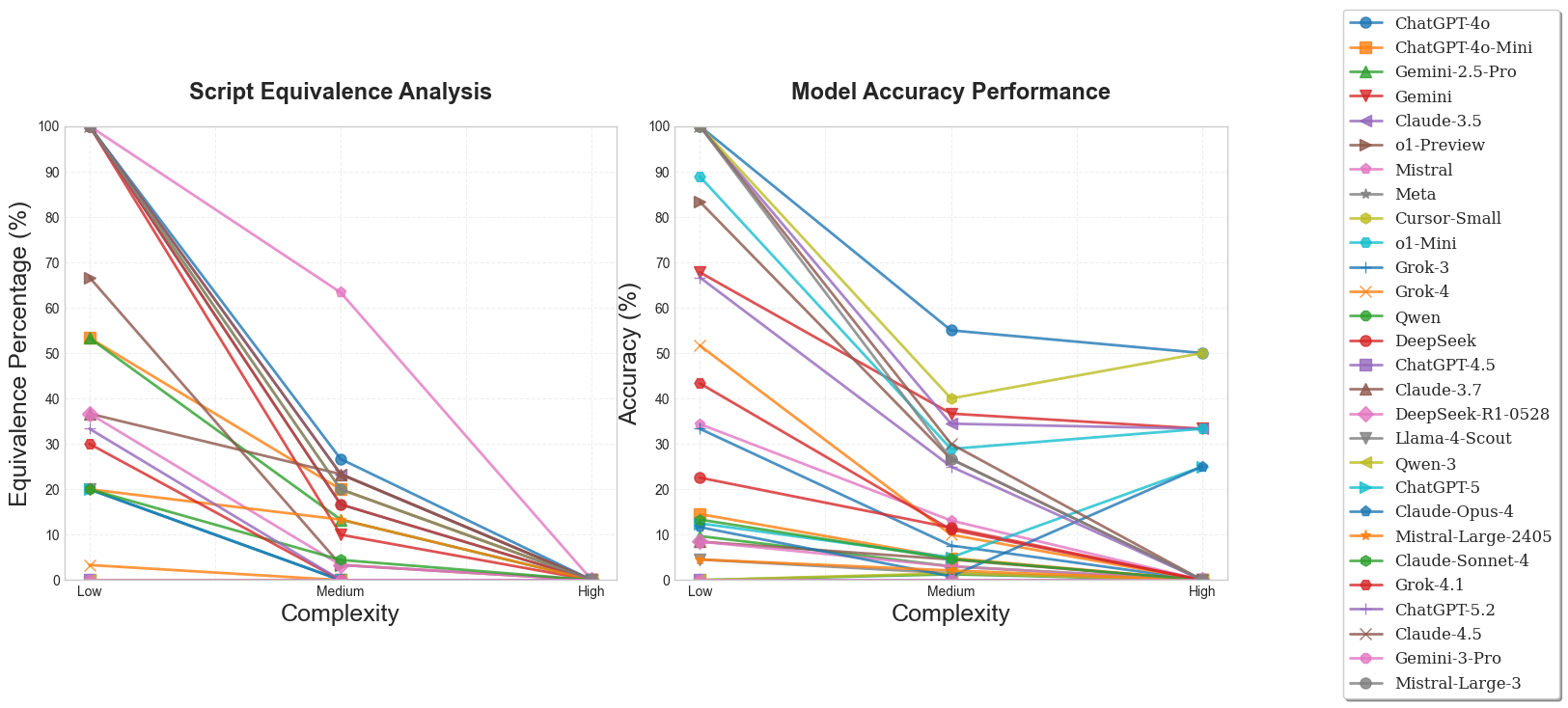} 
\end{subfigure}%

\begin{subfigure}{\textwidth}
  \centering
  \includegraphics[width=1\linewidth]{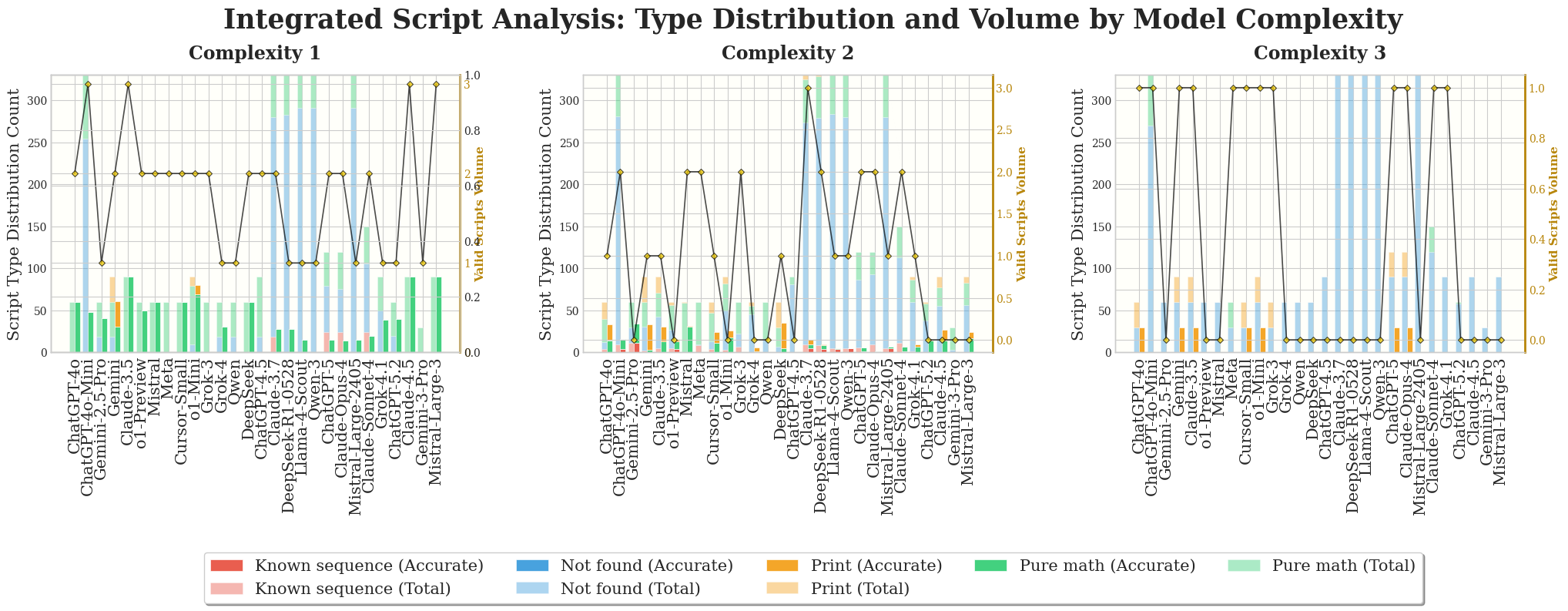}
\end{subfigure}
%\begin{figure}[htp]
%\makebox[\textwidth][c]{\includegraphics[width=0.87
%\textwidth]{images/totalMultiScript.png}}%
\caption{Comprehensive analysis of language model performance in Python script generation across complexity levels (Low, Medium, High). \textbf{Up:} Equivalence percentage (left) and accuracy (right) versus complexity. \textbf{Bottom:} For each model, semi-transparent left bars show total script type distribution (Known sequence=red, Not found=blue, Pure math=green, Print=orange); solid right bars show accurate predictions only; gold diamonds (right y-axis) indicate valid script volume. Disparity between left/right bar heights quantifies the accuracy gap. Results expose fundamental LLM limitations: whilst models generate coherent solutions, accuracy deteriorates markedly with complexity. Predominance of `Not found' (blue) at higher complexities indicates systematic failure to recognise solution strategies. Upper trajectories show equivalence remains stable whilst accuracy plummets—models generate internally consistent but incorrect approaches. Without analogous training exemplars, LLMs cannot reliably deduce solutions despite extensive Python training. Notably, newer iterations (ChatGPT-5, Grok, Gemini) underperformed preview versions (see Supplementary Information), challenging assumptions of monotonic improvement.}
\label{MultiScriptResults}
\end{figure}
\end{center}

% \begin{comment}
% \begin{figure}[htp!]
% \begin{center}
% \makebox[\textwidth][c]{\includegraphics[width=0.8
% \textwidth]{images/Correct_answers}}%
% \end{center}
% \caption{Counting correct instances generated by ChatGPT 4.}
% \label{CorrectAnswers}
% \end{figure}
% \end{comment}

Figure~\ref{NoCompressPercentage} (bottom) shows how no-compression generally increases with complexity, except for Mathematica, where the no-compression percentage is lower at complexity level 2 than at level 1. This happened because Mathematica has the capacity to computationally replicate several well-studied and known sequences of numbers. This capacity leads to shorter code at complexity level 2. However, at complexity level 3, the trend aligns with other languages, showing direct proportionality between complexity and no-compression.

Another analysis addresses the influence of the temperature parameter on the production of code to generate specific numeric sequences.
In Figure~\ref{NoCompressionTemperatureConfidence}, the average percentage of no compression by language,
and across the different values of temperature used during the experiment is shown.
This plot shows the shaded area representing the confidence tolerance over the
average of no compression along the different values of complexity.

The trends in the percentage of no-compression across all temperature values are nearly identical, as are the shapes of the confidence intervals. The temperature value used to generate the code does not affect the result, indicating that the temperature does not have an impact on this experiment. It is worth mentioning the ArnoldC case, where in fact there were not many correct cases, making it difficult to calculate a confidence interval.

\subsection{SuperARC-seq}

Based on the previous experiments, it is possible to characterise one test directly related to the SuperARC framework: the SuperARC-seq. The objective of this test is to quantify intelligence and related cognitive capacities, specifically, reasoning and comprehension, drawing inspiration from the work in~\cite{soler2014calculating} and the theoretical and empirical studies here introduced. 
As mentioned, this test is grounded in one of fundamental cognitive tasks: 
recognising patterns and evaluating the complexity of finite sequences, which inherently requires a level of comprehension in order to provide a meaningful explanation. 
In our experiment, we generated short integer sequences (100 binary and 90 integer-valued in general, as seen in subsections \ref{binseqs} and \ref{intseqs}, respectively, in Sup. Inf.) and tasked several advanced LLMs with deriving a formula capable of reproducing each of the target sequences.

We classified the correct answers provided by the LLMs into three types:
\begin{enumerate}
    \item \emph{Prints}: The model simply reproduced the target sequence without any attempt to encode or express it logically. This response type reflects a failure to abstract or deduce any underlying pattern, simply outputting the sequence as is.
    \item \emph{Ordinal}: The model provided a mapping based on the indices where ``1''s occur in the sequence. This response reflects an attempt by the model to analyse and map some logical structure to the sequence, making it more valuable than simply reproducing it verbatim. For integer sequences in general, a simple ASCII mapping was performed to convert from integers to binary encodings.
    \item \emph{Non-Both}: These responses avoided both simple reproduction and ordinal mapping, reflecting a more sophisticated approach to understanding and encoding the pattern. Such responses are the most valuable as they imply a deeper analysis and potentially creative logic to represent the sequence.
\end{enumerate}

Thus, from these three types of correct results (i.e., the reconstructed sequence matches exactly the original one), we have four different classes of results:
Correct \& Non-Both; Correct \& Ordinal; Correct \& Prints; and Incorrect.%\todo{(R1) Line 234.}

For any given tested model, the percentages of results belonging to each group can be combined as a vector $ \rho $ of rates in the range $ \left[ 0 , 1 \right] $, where $\rho=[\%_{c,np,no},\%_{c,o},\%_{c,p},\%_{inc}]$ such that $\sum_{i=1}^{4} \rho_i = 1$.
Notice that the percentages are represented in the range $ \left[ 0 , 1 \right] $ in order to resemble probabilities.
%\todo{(R1) Line 237.}
We know, beforehand, that the best-performing model would be one with $\rho_{best}=[1,0,0,0]$. Thus, a first possible test would be to check the overall percentage of correct answers.
\begin{equation} \label{tst1t}
    \varphi_a = \sum_{i=1}^{3} \rho_i,
\end{equation}

\noindent which would range from 0 to 1 for models that are not able to reproduce any sequence to models which perfectly reconstruct the sequences, respectively. However, this only accounts for the ability of LLMs to reproduce the initial sequence (planning) but not for their compression capabilities. To account for the latter, let us assume that the best possible algorithm for each element of the data set is $\mathcal{B}_{k,j}$, such that $\mathcal{B}_{k,j} () = D_{k,encoded}[j]$, and here the algorithm does not have a particular input, similar to the definition of algorithmic complexity. Thus, from the basic properties in AIT:

\begin{equation}\label{kb}
    K(D_{k,encoded}[j]) = K(\mathcal{B}_{k,j} ()) \leq K(\mathcal{B}_{k,j}) + \mathbf{O}( 1 )
\end{equation}

%\noindent due to the information non-increase theorem and to the fact that no inputs were used in the function. 
The ratio 
\begin{equation} 
\frac{ K\left( D_{k,encoded}[j] \right) }{ K(\mathcal{B}_{k,j}) + \mathbf{O}( 1 ) } 
\end{equation}
consistently falls within the range [0,1] for medium to long sequences when no embedding algorithms are employed. This behaviour arises because approximations of algorithmic complexity are less reliable for short sequences, primarily due to the overhead inherent in theoretical computations. 
In order to surpass this limitation, 
%since the difference between the actual algorithmic complexity value and its approximation is bounded by a linear constant in general, 
instead of assessing the absolute algorithmic complexity (or any of its approximations), we shall consider a \emph{normalised} version of it (denoted by $ nBDM( \cdot ) $). 

To approximate algorithmic complexity, we will use the BDM/CTM approach, as described in detail in previous sections, and its normalised version, as pointed out in previous works~\cite{bdm} for any object of arbitrary size, it is possible to construct analogous objects that attain the minimum and maximum possible values of algorithmic complexity according to the BDM:
\begin{itemize}
    \item \emph{minimum} complexity object: This case is straightforward and corresponds to an object composed entirely of a single repeating symbol---for instance, a binary string consisting solely of zeros;
    \item \emph{maximum} complexity object. The maximum BDM value is achieved by an object whose decomposition (according to a specified algorithm) results in slices that exhibit the highest values of the Coding Theorem Method (CTM), with each distinct slice occurring only once until all possible configurations of the given shape have been exhausted.
\end{itemize}

The primary advantage of considering a normalised measure lies in its ability to enable comparisons between objects of varying sizes, effectively mitigating the influence of size on the measure itself. This property is particularly in the case of the present study, where we compare complexities of sequences and formulas generating them.

This way, the following ratio presents itself as an interesting weighting factor for the probabilities in Equation~\eqref{tst1t}:

\begin{equation}\label{rat}
    nBDM(D_{k,encoded}[j]) / nBDM(\mathcal{B}_{k,j})
\end{equation}

The ratio in Equation~\eqref{rat} measures how the algorithmic complexity of the formula and sequence compare to the other possible outputs of the LLM. If the relative algorithmic complexity (measured by the normalised BDM value) of the formula is greater than it was for the sequence itself, this suggests the LLM did not success in compressing the input sequence (it made the formula have a greater relative algorithmic complexity). 
On the other hand, if the opposite occurs, then the LLM could compress the sequence comparatively to other possible outputs of the LLM.

The ratio in Equation~\eqref{rat} ranges from 0 to a positive value $M>1$, which happens when the best possible compression is achieved (the inverse mapping of CTM). Since $M$ is not known beforehand, we can use a nonlinear mapping that saturates the value of the ratio to a maximum value of 1 (similar to an activation function). The hyperbolic tangent function can be used in this case, since $\tanh(0) = 0$ and $\lim_{x \to \infty} \tanh(x) = 1$. Thus, a candidate weighting factor for the probabilities in \ref{tst1t} is:

\begin{equation}\label{ratt}
    \delta_{k,j} = \tanh\left(\frac{nBDM(D_{k,encoded}[j])}{nBDM(\mathcal{B}_{k,j})} \right)
\end{equation}

\noindent with the best possible value of $\delta_{k,j}$ approaching 1 in a perfect compression scenario. Since we have several algorithms classified under each of the four types (according to their structure), instead of using the individual ratios for each type $k$, we shall use the harmonic mean per type, defined as:

\begin{equation}
    \delta_k = \frac{n_k}{\displaystyle \sum_{j=1}^{n_k} \delta_{k,j}^{-1}} \text{ for } R_{k,j} \text{ of type }k,
\end{equation}

\noindent where $n_k$ represents the number of algorithms that are of type $k$. If we include $m$ sequences in the test, for example, $n_k = m\rho_k$. Thus, an updated version of the test is:

\begin{equation}
    \varphi_b = \sum_{i=1}^{3} \delta_i\rho_i.
\end{equation}

Deliberately, we want to privilege models that do not simply copy or provide ordinal mappings of the input sequences. Thus we can attribute higher weights to types that are correct and do not copy nor print the results. We also want to give more weight to programs that provide ordinal mappings when compared to print cases. Then, considering a power-law weighting strategy, the final test metric is:

\begin{equation}\label{testf}
    \varphi = \delta_1\rho_1 +\frac{\delta_2\rho_2}{10}+\frac{\delta_3\rho_3}{100}.
\end{equation}

It can be seen that $\varphi \in [0,1]$ encompasses different behaviours. For example, $\varphi \in [0,0.01]$ if only print-type models are outputted. Also, $\varphi \in [0,0.1]$ if only ordinal-like formulas are created. Finally, $\varphi \in [0,1]$ in cases where the LLMs create formulas that are always correct, do not copy nor create ordinal mappings. The ranges will be populated with varying compression levels corresponding to the algorithms obtained. Overall, if the score is 0, all the formulas were wrong. If it is 0.5, it can represent the case where half the outputs were correct and half wrong, with the formulas produced with highest compression levels. So, in a regular half and half case, since compression will not be optimal, the test score is less than 0.5.
The test performance results for each model are calculated using Equation~\eqref{testf} for $\mathcal{T}$ in Algorithm~\ref{algo1}. 

There are some possible variations for the test metric in Equation~\eqref{testf}. For example, some sort of Bayesian approach could be used to consider that the elements of $\rho$ are not constants, but random variables which could account for the number of different correct/incorrect answers for the same input sequence. In this way, the multiplicity of possible generators is taken into account, better capturing the concept of algorithmic probability, and the output of the test would be a random variable instead. However, LLMs hardly produced even one correct answer, therefore we kept the formula simple.

As described, Equation~\eqref{testf} tests for two features, compression via non-print computer programs and non-ordinal mathematical formulas to the input sequence, and prediction, by running all programs and all formulas to match each sequence digit, and penalising them when they did not represent an actual compressed model that generated a possible new digit of the sequence when run in reverse, i.e. when `decompressed'. The test formula assigns greater importance to correct cases that are not solutions of the type `print($s$)' where $s$ is the sequence for which the AI system is asked for a model, given that a print model does not allow generalisation by prediction through simulation, as running a print command will only print up to the last digit. The same is true for what we call `ordinal', which is simply indicating the index of the non-zero non-one element in the binary (or binary embedded) sequence, meaning that, together with the `print' case, the system failed in its attempts at abstracting features of the object. Finally, the formula punishes ordinal and print answers in a weighted fashion. The best performer can only reach a $\varphi$ of 1 while the lowest value is 0.

\subsubsection{Applying SuperARC-seq}

The results of the LLM classification after applying this test according to the formula are shown in Table~\ref{tableRanking} and summarised in Figure~\ref{rankingSuperARC} for \emph{binary sequences}. As shown in Table~\ref{tableRanking} and Figure~\ref{rankingSuperARC}, CTM/BDM would achieve perfect scores in all categories, consistently avoiding trivial responses and providing accurate formulas. 
By design, this model clearly excels in abstract feature recognition, outperforming all other models at prediction, which we claim is key to planning. CTM/BDM actually produces a set of possible generative models (computer programs) that, when run in reverse in what would be the uncompressing process, produce new elements to test against the observation, thus updating and producing new possible outcomes. These models are also hypotheses that do suggest whether a sequence is random or not, rather than looking for such a sequence in the training set or a combination thereof and failing for those not found in the distribution.

\begin{table}[ht]
\centering
\adjustbox{max width=\textwidth}{
\sisetup{round-mode=places, round-precision=3}
\begin{tabular}{c|SSSSSSSS}
Model         & {$\rho_1$}   & {$\rho_2$}   & {$\rho_3$}  & {$\rho_4$}   & {$\delta_1$}       & {$\delta_2$}       & {$\delta_3$}       & {$\varphi$}       \\ \hline
AIXI/BDM/CTM & \textbf{1.000} & 0.00 & 0.0 & \textbf{0.000} & \textbf{1.000} & 0.000 & 0.000 & \textbf{1.000} \\ \hline
ChatGPT-4.5   & 0.00 & \textbf{1.000} & 0.0 & \textbf{0.000} & 0.000 & 0.419 & 0.000 & 0.042 \\ \hline
o1-Mini       & 0.00 & 0.64 & 0.0 & 0.36 & 0.000 & \textbf{0.537} & 0.000 & 0.034 \\ \hline
Claude-3.7    & 0.00 & 0.81 & 0.0 & 0.19 & 0.000 & 0.407 & 0.000 & 0.033 \\ \hline
Claude-3.5    & 0.06 & 0.14 & 0.0 & 0.80 & 0.449 & 0.428 & 0.000 & 0.033 \\ \hline
o1-Preview    & 0.00 & 0.29 & 0.0 & 0.71 & 0.000 & 0.423 & 0.000 & 0.012 \\ \hline
Gemini        & 0.00 & 0.00 & \textbf{1.000} & \textbf{0.000} & 0.000 & 0.000 & \textbf{0.762} & 0.008 \\ \hline
Cursor-Small  & 0.00 & 0.00 & \textbf{1.000} & \textbf{0.000} & 0.000 & 0.000 & \textbf{0.762} & 0.008 \\ \hline
ChatGPT-4o-Mini & 0.00 & 0.00 & \textbf{1.000} & \textbf{0.000} & 0.000 & 0.000 & \textbf{0.762} & 0.008 \\ \hline
Mistral       & 0.00 & 0.00 & \textbf{1.000} & \textbf{0.000} & 0.000 & 0.000 & 0.710 & 0.007 \\ \hline
Qwen          & 0.00 & 0.00 & \textbf{1.000} & \textbf{0.000} & 0.000 & 0.000 & 0.710 & 0.007 \\ \hline
DeepSeek      & 0.00 & 0.00 & \textbf{1.000} & \textbf{0.000} & 0.000 & 0.000 & 0.710 & 0.007 \\ \hline
Llama-4-Scout & 0.01 & 0.00 & 0.0 & 0.99 & 0.450 & 0.000 & 0.000 & 0.004 \\ \hline
Grok-3        & 0.00 & 0.02 & 0.0 & 0.98 & 0.000 & 0.318 & 0.000 & 0.001 \\ \hline
Mistral-Large-3 & 0.00 & 0.00 & 0.0 & 1.00 & 0.000 & 0.000 & 0.000 & 0.000 \\ \hline
Meta          & 0.00 & 0.00 & 0.0 & 1.00 & 0.000 & 0.000 & 0.000 & 0.000 \\ \hline
Gemini-3-Pro  & 0.00 & 0.00 & 0.0 & 1.00 & 0.000 & 0.000 & 0.000 & 0.000 \\ \hline
Claude-4.5    & 0.00 & 0.00 & 0.0 & 1.00 & 0.000 & 0.000 & 0.000 & 0.000 \\ \hline
ChatGPT-5.2   & 0.00 & 0.00 & 0.0 & 1.00 & 0.000 & 0.000 & 0.000 & 0.000 \\ \hline
Grok-4.1      & 0.00 & 0.00 & 0.0 & 1.00 & 0.000 & 0.000 & 0.000 & 0.000 \\ \hline
ChatGPT-4o    & 0.00 & 0.00 & 0.0 & 1.00 & 0.000 & 0.000 & 0.000 & 0.000 \\ \hline
Grok-4        & 0.00 & 0.00 & 0.0 & 1.00 & 0.000 & 0.000 & 0.000 & 0.000 \\ \hline
Claude-Sonnet-4 & 0.00 & 0.00 & 0.0 & 1.00 & 0.000 & 0.000 & 0.000 & 0.000 \\ \hline
Gemini-2.5-Pro & 0.00 & 0.00 & 0.0 & 1.00 & 0.000 & 0.000 & 0.000 & 0.000 \\ \hline
Mistral-Large-2405 & 0.00 & 0.00 & 0.0 & 1.00 & 0.000 & 0.000 & 0.000 & 0.000 \\ \hline
Claude-Opus-4 & 0.00 & 0.00 & 0.0 & 1.00 & 0.000 & 0.000 & 0.000 & 0.000 \\ \hline
DeepSeek-R1-0528 & 0.00 & 0.00 & 0.0 & 1.00 & 0.000 & 0.000 & 0.000 & 0.000 \\ \hline
Qwen-3        & 0.00 & 0.00 & 0.0 & 1.00 & 0.000 & 0.000 & 0.000 & 0.000 \\ \hline
ChatGPT-5     & 0.00 & 0.00 & 0.0 & 1.00 & 0.000 & 0.000 & 0.000 & 0.000 \\ \hline
\end{tabular}
}
\caption{Numerical benchmark ranking of popular frontier models publicly available against ASI represented by fundamental or neurosymbolic models like AIXI~\cite{hutter} and CTM/BDM~\cite{bdm,nmi}. Best per-column values are in bold (for all columns greater values are better except for $\rho_4$, where smaller is better).}
\label{tableRanking}
\end{table}

\begin{center}
\begin{figure}[h]
\makebox[\textwidth][c]{\includegraphics[width=1
\textwidth]{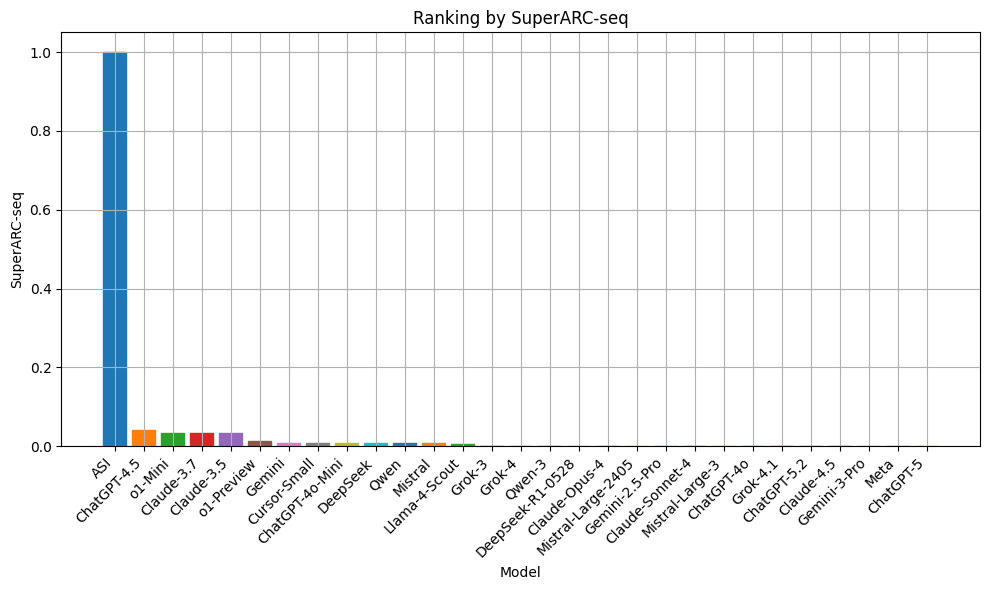}}%
\caption{Benchmarking plot from Table~\ref{tableRanking} showing how most frontier models are close to each other in their performance under this test and far from AGI or ASI goals according to this test. ASI would be able to distinguish simpler from complex sequences and generate predictive models for each accordingly, as AIXI~\cite{hutter2005universal} or CTM/BDM would do~\cite{nmi,bdm} as instantiations of Universal AI (UAI) that we take as an example of ASI as optimal abstraction and prediction. Today, LLMs only produce or retrieve models for sequences that were seen and found in their original training sets, given that increasing the sequences' lengths impacts the LLM performance in identifying the sequence, hence indicating sequences are not recognised from first principles but from simplistic pattern matching.}
\label{rankingSuperARC}
\end{figure}
\end{center}

% \begin{table}[ht]
%     \centering
%     \begin{tabular}{c|c}
%         \textbf{model} & \textbf{ranking} \\ \hline
%         CTM/BDM & 100.0 \\ \hline
%         Claude\_3.5 & 2.83 \\ \hline
%         GPT\_4o\_mini & 1.0 \\ \hline
%         Cursor\_small & 1.0 \\ \hline
%         Gemini & 1.0 \\ \hline
%         Mistral & 1.0 \\ \hline
%         o1\_mini & 0.064 \\ \hline
%         o1\_preview & 0.029 \\ \hline
%         GPT\_4o & 0.0 \\ \hline
%         Meta & 0.0 \\ \hline
%     \end{tabular}
%     \caption{Numerical ranking of various popular and publicly available frontier LLM models against pure neuro-symbolic approach CTM/BDM on this test. The results suggest that while LLMs are capable of answering, they are still mostly objects that retrieve or regurgitate previous cases, unable to pass the test or get any closer to any sort of smart abstraction of a complex input from the external world not likely to have been seen before, and therefore unable to perform any planning, with Claude 3.5 at only 2.83\% out of 100\% but nevertheless well beyond the other LLMs considered. Surprisingly, 4o---mini did slightly better than 4o and ChatGPT-o1 worst that its predecesor ChatGPT---4o (See Sup. Inf).}
%     \label{tableRanking}
% \end{table}

These findings indicate that LLMs perform well when there are discernible patterns in the data, but struggle with randomness, failing to capture complexity in an algorithmic sense. In contrast, AIT can accurately predict (rather than guess) the sequence, regardless of the string's complexity. These results demonstrate that the algorithmic-complexity approach effectively approximates the minimal description length of information, identifying the shortest algorithm capable of generating a given sequence.

Despite being the top-ranked LLM model, chatgpt\_4.5 only provided ordinal mappings (soft copies) of the inputs, which achieved correct results at the cost of no abstraction and comprehension at all (slightly better than a pure a print-only test score). The GPT-4o, Grok-3, Meta, Claude 3.5 and o1---preview LLM versions produced several incorrect formulas while the other LLM models considered mostly produced print-like responses, indicating a lack of pattern recognition beyond basic sequence reproduction.
Notably, in the evaluation of consecutive versions, the most recent models---with the exception of Grok---demonstrated a degradation in performance.

Unlike standard LLMs that predict the next tokens in text, CTM/BDM finds the generative processes of the sequence by
a combination of symbolic and statistical pattern matching algorithms, which allows it to derive concise models that can then run in reverse to match each digit and produce new ones, hence allowing prediction and planning by picking the most likely among a set of possible models based on the algorithmic probability of the model (how short and how often the same model was found to produce the same sequence).

It is important to note that the SuperARC-seq application hereby considered only took into account binary sequences. Whenever integer sequences were considered, a clear biasing of the results was observed as LLMs started to take advantage of their training corpus to actually display memorisation capabilities rather than abstraction and synthetization ones. Figures \ref{fig:figpercent} and \ref{fig:figtst} present the percentages of each type of output and the test scores when different types of sequences were considered.

\begin{figure}
\begin{subfigure}{\textwidth}
  \centering
  \includegraphics[width=0.85\linewidth]{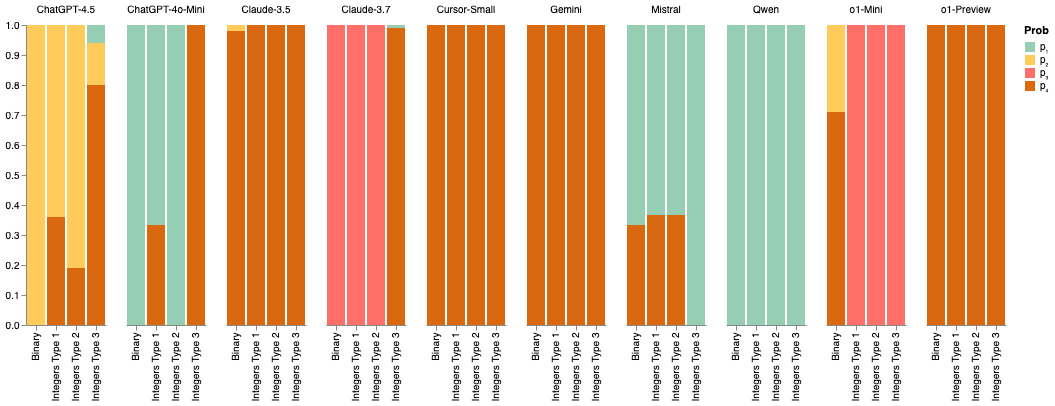} 
\end{subfigure}%

\begin{subfigure}{\textwidth}
  \centering
  \includegraphics[width=0.85\linewidth]{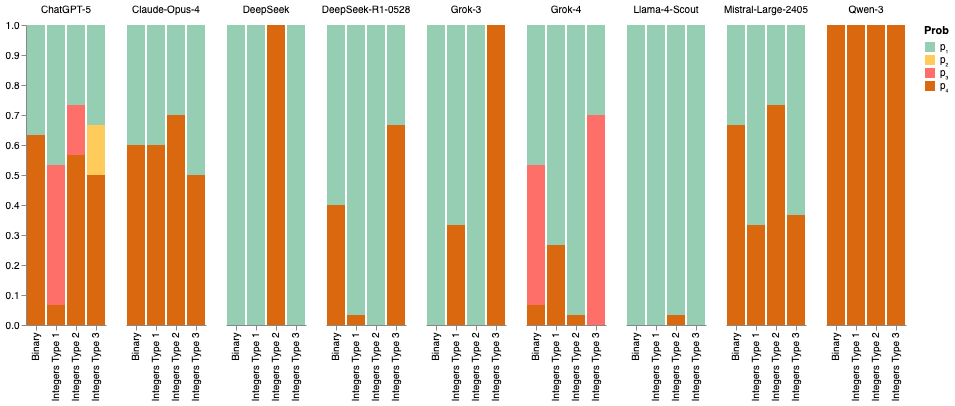}
\end{subfigure}

\begin{subfigure}{\textwidth}
  \centering
  \includegraphics[width=0.85\linewidth]{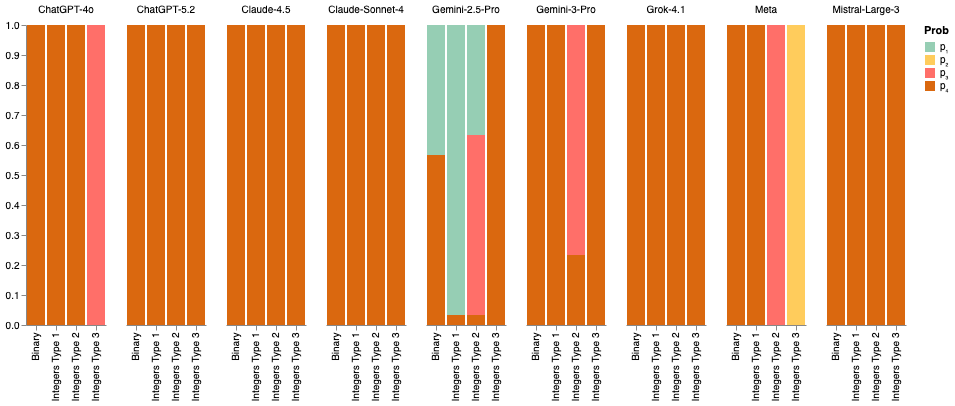}
\end{subfigure}
\caption{Percentages by output types: $p_1$ is the percentage of Correct \& Non-Prints \& Non-Ordinal outputs; $p_2$ is the percentage of Correct \& Ordinal outputs; $p_3$ is the percentage of Correct \& Prints outputs and $p_4$ is the percentage of Incorrect outputs. It is clear that as soon as integer sequences are considered, LLMs start to get better quality output formulas (i.e.,  greater $p_1$ and $p_2$). This suggests that the models were trained on integer sequences rather than binary ones, implying that incorporating integer sequences into the test calculations could introduce bias.}
\label{fig:figpercent}
\end{figure}

\begin{figure}
\begin{subfigure}{\textwidth}
  \centering
  \includegraphics[width=1\linewidth]{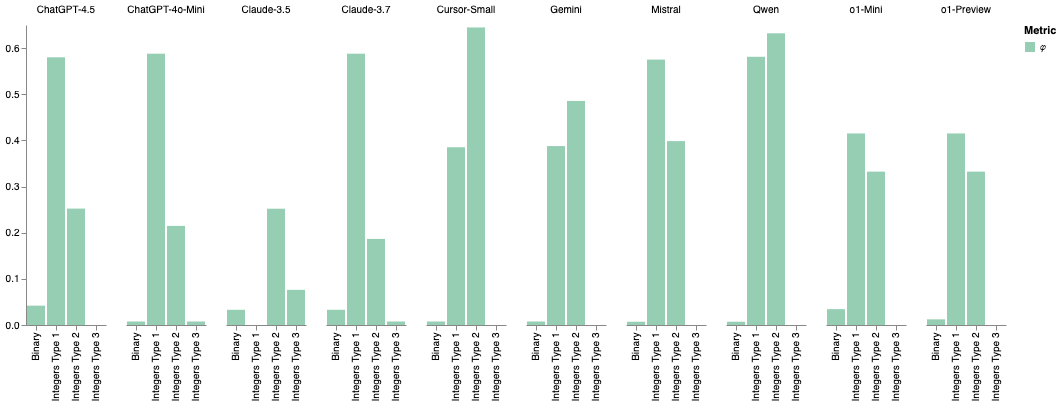} 
\end{subfigure}%

\begin{subfigure}{\textwidth}
  \centering
  \includegraphics[width=1\linewidth]{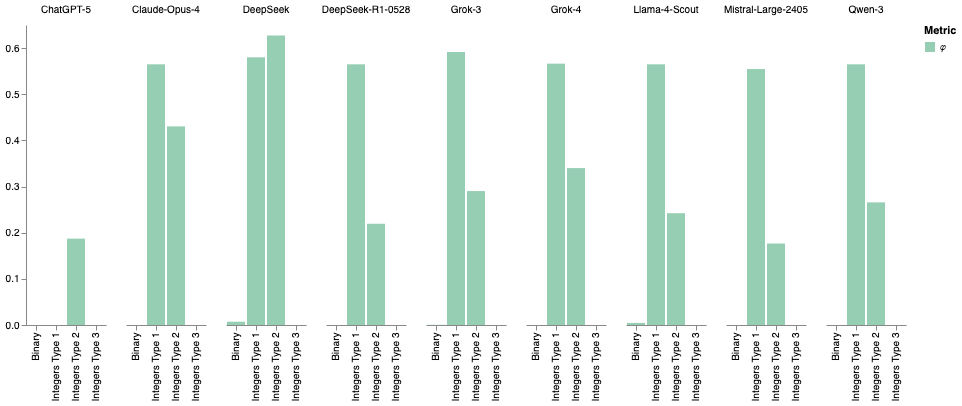}
\end{subfigure}

\begin{subfigure}{\textwidth}
  \centering
  \includegraphics[width=1\linewidth]{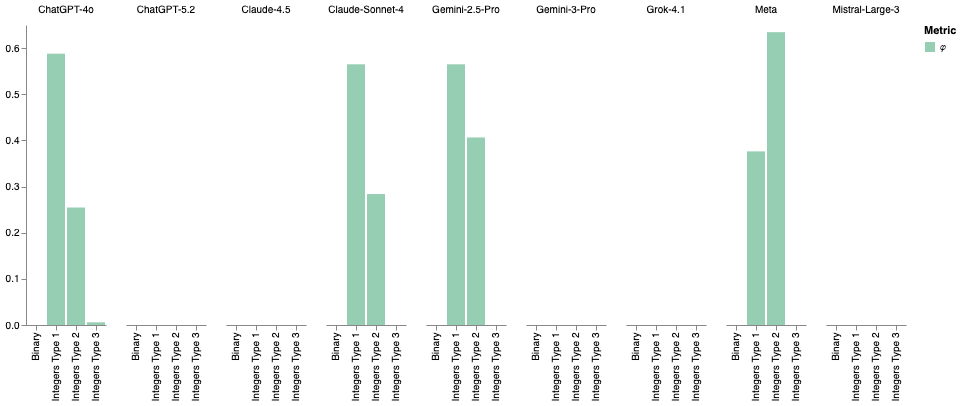}
\end{subfigure}
\caption{Test scores when different types of sequences are considered. Consistent with the results shown in Figure \ref{fig:figpercent}, the inclusion of integer sequences leads to significantly higher test scores for the LLMs. This outcome arises from the models' ability to exploit their internalised training data by directly associating observed sequences with pre-learned formulas, suggesting a form of hash-like memorisation. These findings highlight the importance of restricting the evaluation to binary sequences in order to obtain an unbiased measure of each model's performance, as such sequences are less likely to have been included in the models' training corpora.}
\label{fig:figtst}
\end{figure}

The test scores for different types of sequence reveal that the inclusion of integer sequences leads to significantly higher performance of LLMs, as shown in Figures \ref{fig:figpercent} and \ref{fig:figtst}, where higher percentages of Correct \& Non-Prints \& Non-Ordinal and Correct \& Ordinal outputs are seen, as well as higher test scores. This is likely due to the models leveraging memorised associations between familiar integer sequences and pre-learnt formulas - an effect similar to hash-based retrieval. These findings show the importance of limiting evaluations to binary sequences, which are less likely to have been part of the training data, thereby providing a more accurate and unbiased assessment of model performance.

The robustness of the test score when only binary sequences are considered can be seen in Figure \ref{robustnesstst}, which shows the result of a bootstrap procedure. The bootstrap simulation procedure was conducted as follows: for each specified sample size $s$ ($s$ equal to 25, 50, 75 and 100), 100 bootstrap samples of size $s$ were drawn with replacement from the complete dataset, which consisted of 100 binary sequences (presented in subsection \ref{binseqs} in Sup. Inf.). For each bootstrap sample, the corresponding test scores were computed. The resulting plot presents the confidence intervals for the test scores obtained across all bootstrap iterations. The observed stability in test scores, coupled with the progressively narrowing confidence intervals around the mean as sample size increases, suggests a high degree of robustness in the evaluation metric. This indicates that the test score is largely insensitive to the particular subset of sequences used, thereby validating the reliability of the assessment across different sample sizes.

\begin{center}
\begin{figure}[H]
\makebox[\textwidth][c]{\includegraphics[width=1
\textwidth]{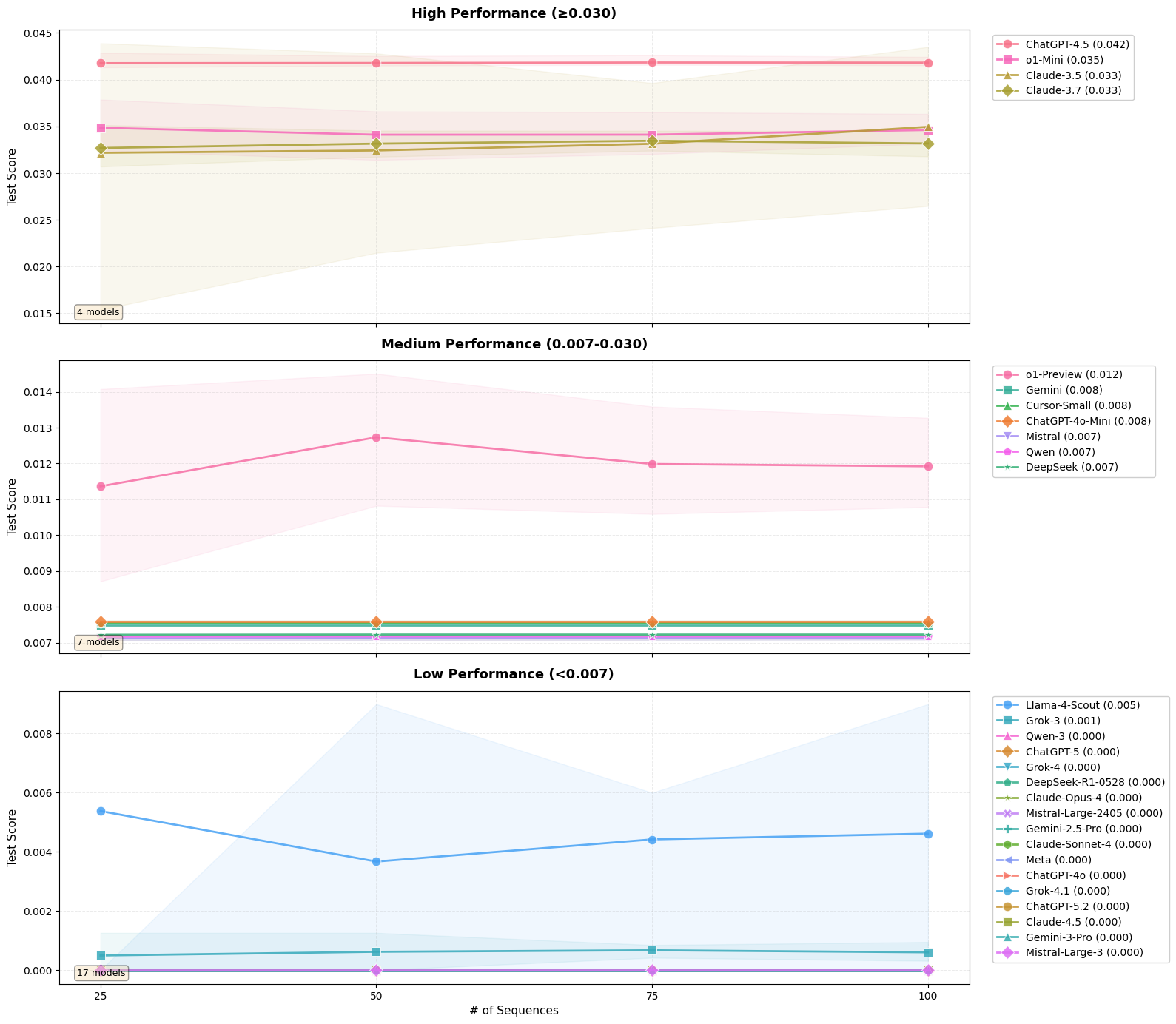}}%
\caption{Bootstrap procedure to assess the robustness of the test score when binary sequences were used. The stability of the test scores, with average values per model provided in parentheses, in combination with the narrowing confidence intervals around the mean as sample size increases, indicates strong robustness of the evaluation metric.}
\label{robustnesstst}
\end{figure}
\end{center}

%{\color{red} \todo{Answer to 3R3.}
These results are aligned with recent work exploring, for example, LLMs' logical reasoning failures~\cite{Dziri} as well as GPT-4's limitations in deductive reasoning~\cite{arkoudas2023gpt4cantreason}. 
Other researchers have also reported degradation of mathematical capabilities~\cite{Frieder} and planning limitations~\cite{Kambhampati}. 
These works collectively document that LLMs struggle with systematic reasoning across multiple domains.
%}

\section{Discussion}\label{sectionDiscussion}

In previous work, we have shown that aspects of human~\cite{ploscomp,zenilbelg} and animal~\cite{zenilanimal} cognition could be characterised, and aspects of their behaviour reproduced in terms of algorithmic probability tools and algorithmic complexity metrics that we have also suggested for artificial and computational systems, including robotics~\cite{zenilrobots}. 
Here, we tested this approach and proposed a new quantitative metric based on an fundamental ASI-level method~\cite{zenilbook1} grounded in Algorithmic Information Dynamics (AID)~\cite{zenil2020algorithmic,zenilbook2} and the principles of Algorithmic Information Theory (AIT) related to recursive compression (as opposed to statistical) and prediction in application to Large Language Models (LLMs), which are believed or have been proposed to be capable of approaching Artificial General Intelligence (AGI) and Superintelligence (ASI).

Algorithmic Information Dynamics (AID)~\cite{zenilbook2} combines aspects of statistical causal inference (those best grounded on causality), such as perturbation analysis, and Algorithmic Information Theory (AIT). AID suggests that for an AI to be properly evaluated, the test has to be dynamic and must evaluate the system's ability to adapt and react to the new conditions in what constitutes a type of more sophisticated Turing test beyond the simplicity of human language.

Applied to SuperARC, the test involved increasing the problem input complexity and analysing the output complexity with our most powerful tools that, as a consequence, are (semi-)computable but truly capable of evaluating random deficiency, that is, how far an answer is from randomness or simplicity, and for optimal prediction power of the extracted/inferred executable computable model(s).

As governed by AIT, recursive compression and optimal prediction go hand in hand~\cite{zenil2020compression}, but previous tests focused on particular subset features, even those designed to test human reasoning and human abstraction, such as ARC-AGI~\cite{Chollet2019MeasureIntelligence}. A system could theoretically excel at ARC-AGI or SuperARC while completely lacking social intelligence, embodied reasoning, common sense, or goal-directed behaviour; and therefore, these tests are not meant to test everything we think an intelligent system that will interact with humans in humans' context should possess. 

Here, we have introduced and demonstrated that recursive compression can quantify the property of model abstraction and data prediction based on a mathematical proof provided in Section~\ref{sectionPredictionandCompression} of the equivalence between model compression and prediction applied to sequences based on Martingales, without resorting to proof-theoretic statistical tests. By incorporating and exploiting the formal equivalence between prediction and recursive compression into an intelligence test framework, we align the assessment of comprehension with fundamental computational principles. 
An agent's ability to abstract information through feature selection and model compression reflects its capacity to identify and utilise patterns within data. 
Similarly, its planning and prediction skills demonstrate its ability to anticipate and enact future events based on these patterns.

Another problem in LLM testing is benchmark contamination: the targeted optimisation of a model using leaked test answers. The open-ended and interactional nature of our testing framework is intended to counteract this problem of benchmark contamination and cheating. Our investigation of frontier models, framed within the AID paradigm, yields several key insights about the models' comprehension capabilities. Most of the models often demonstrate model version performance regression, poor accuracy in replicating and predicting even simple and recursively generated sequences beyond clearly memorisation results from the training distribution (such as sequence labelling). The vast majority of the correct answers turned out to be simple print statements of the numerical sequences themselves, rather than any code or model indicating any sign of understanding or pattern recognition.

These conclusions are reinforced by the model's explicit dependency on specific programming languages for correctness or on well-studied and documented series of numbers. In other words, if there are not enough implementations available in a specific programming language for the model to learn from, or even specific methods of mathematical analysis over specific numerical sequences, LLMs failed to produce the correct answer. 
Considering the most popular and widely used languages, LLMs do not demonstrate understanding but instead rely on selecting from an abundance of previously seen cases.

In this context, we have shown that using those concepts in reverse (that is, as generative rather than for testing purposes) can provide the model with synthesising and recursive predictive power that is otherwise lacking. In previous work, we have shown how optimal prediction can be approximated using tools such as CTM and BDM~\cite{zenilbook1}, and also how BDM can be used in the opposite fashion: not only as a testing tool for intelligence, but as a model generator~\cite{nmi,royal,iscience} via a greedy algorithm as an approach to optimal inference~\cite{bdm,complexity,computability}. 
While some of these methods can be seen as a brute force approach to a giant lookup table of micro-programs to explain the data, BDM is not. 
BDM combines the algorithmic probability approximations produced by CTM but then stitches each most likely program for each piece back together according to valid laws of information theory, in what constitutes a pure form of hybrid statistical and symbolic explanation---hence neurosymbolic. 
BDM, therefore, uses the two best inference theories currently available to science, one being the most used and overused in statistical machine learning (such as Shannon entropy-based measures, along with its limitations~\cite{zenil2017low}), and one that has often been neglected on the basis of uncomputability~\cite{abrahao,zenilreview}, which in fact guarantees open-endedness and optimal solutions.

%While LLMs are impressive linguistic tools, LLMs were never designed to reason, infer, or perform rationally beyond statistical alignment. 

LLM developers are moving slowly towards this direction, often without realisation or acknowledgement. Technologies like RAGs, KAGs/CAGs, agentic workflows, guardrails, long-context LLMs and more, are symbolic computation attempts to improve LLM capabilities. Yet many of these approaches are not being introduced into the system's core functions (e.g. to quantify loss or explore the solution space), yet attempt to combine traditional statistical pattern matching with symbolic approaches. This test and work suggest that this integration and a deeper integration are key for aspects of model abstraction and universal planning.

% Recursive compression and optimal prediction go hand in hand~\cite{zenil2020compression}, but previous tests focused on particular subset features, even those designed to test human reasoning and human abstraction such as ARC~\cite{arc}.

%This novel test of intelligence is clearly closely related to selection and evolution through the connection between adaptation and prediction. An organism adapts best when it simulates many possible outcomes to choose from and then, from among the different outcomes for a given input, predicts an outcome better aligned with its expectations in the greatest number of possible situations.

We reported that even frontier LLMs currently perform close to pure-copy solutions when increasing the complexity of certain mathematical problems, with even advanced models struggling to perform model abstraction/extraction and produce executable predictive results.
%These results would also imply a poor performance of LLMs in traditional tests of education as introduced by e.g. Bloom~\cite{bloom1956taxonomy} in its education hierarchy for humans. %testing for new knowledge and synthesis generation test. 
The results confirm that current LLMs, while competent in pattern replication, lack critical elements associated with what are believed to be key elements of AGI or ASI. 

The LLMs involved in this test showed a high dependency on predefined patterns. As complexity increased, models relied increasingly on trivial strategies, such as direct sequence printing or simplistic brute-force mathematical expressions. This highlights the LLMs' inability to abstract or conceptualise novel solutions that require some degree of mathematical ingenuity.

The poor performance is revealed by the lack of synthesising capabilities and the repetitive nature of the outputs for greater complexity inputs. This tendency to revert to safe and redundant approaches underscores the models' limited synthesising capabilities and exposes their high memory exploration dependencies in simpler modes.

For example, we reported that while from ChatGPT-4.5 to ChatGPT-5 improved human benchmark scores have been observed and reported, SuperARC performance degraded. The steep decline in accuracy and functional outputs as complexity increased reveals that these models are potentially increasingly heavily reliant on data size, unable to generalise.

Our results demonstrate that models can regress in fundamental algorithmic reasoning (SuperARC scores) while simultaneously improving in human-centric benchmarks.

The models' outputs suggest strength in replication but a lack of adaptivity to new situations as a disability to put together solutions in a non-trivial manner to solve a new problem that would constitute a form of model synthesis when increasing problem complexity. The predominance of trivial or incorrect solutions demonstrates an inability to think `outside the box' (as in if it had not been seen in the training distribution). This suggests that while LLMs can mimic comprehension through retrieval, pattern matching, and Chain-of-Thought techniques, their capabilities remain bounded when tested against algorithmically complex sequences. These observations point directly to a key distinction between current systems and AGI/ASI: the latter would require the ability to autonomously generate new strategies, abstract concepts, and exhibit flexible problem-solving beyond training data.

Our results based on the first principles of the mathematical theory of randomness and optimal inference in the context of claims about `reasoning' capabilities and AGI/ASI, suggest that none of the chatbots evaluated dominates any other in absolute terms when it comes to demonstrate whether the solutions have a semantic value, in other words, that LLMs mean the words they produce as opposed to emulating a coherent conversation.

%We believe that our test has fundamental significance because it demonstrates that LLMs primarily rely on direct pattern matching, making it impossible for them to predict in even basic and well-defined scenarios in a meaningful way. This limitation is closely related to the phenomenon of hallucinations in LLMs, which reinforces the criticism that LLMs lack an internal model of the world to allow them to simulate possible future scenarios and pick the most likely for planning purposes. 
%This makes claims about `reasoning', reaching AGI, or heading toward Superintelligence, unfounded. 
%One lesson from LLMs that our present results support is that we should dissociate language from intelligence, something Turing himself suggested with his imitation game~\cite{turing}.

%\todo{(R1)(R2) On compression as necessary condition for comprehension.}
%We proved that compression is proportional to prediction and vice versa. 
%That is, if a system can better predict, it can better compress; and if it can better compress, then it can better predict.
We have argued throughout this paper, and it is distilled by the nature of our test for intelligence, that (semicomputable) open-ended tests not assumed to be isolated ``in a vat'' are needed in order to quantify scientific and mathematical intelligence in the form of abstraction and prediction so that they are fundamentally coupled into a necessary characteristic of comprehension~\cite{bibid}, whether performed by a human~\cite{ploscomp,zenilbelg}, animal~\cite{zenilanimal}, or artificial~\cite{zenilrobots} agent. 
The same also holds for the tested systems: that only by incorporating open-ended, sufficiently powerful semicomputable predicting agents (e.g., via the methods explored in the present paper), one may achieve such a characteristic of intelligent systems, and thereby enabling the possibility of ASI by way (or not) of AGI. 
%We have also argued that optimising for the features that our test captures will lead to Superintelligence.

%% Curent word count excluding the abstract up to here is 8730 words.

\section{Funding}

Felipe S. Abrah\~{a}o acknowledges support from the Sao Paulo Research Foundation (FAPESP), grants $2021$/$14501$-$8$ and $2023$/$05593$-$1$.

\section{Code and Data Availability}

The code and data generated for this work are available at \url{https://github.com/AlgoDynLab/SuperintelligenceTest} where a benchmark table will be updated regularly for frontier models as they release new LLMs and other AI systems.

\newpage

\section{Methods}\label{sectionMethods}

\subsection{Assessing the capabilities of frontier LLM models}\label{sectionFrontierLLMs}

Since the inception of LLMs, these systems have been associated with human intellectual capabilities related to language that range from mastering composition to retrieving contextual data and even generating novel `ideas'~\cite{hubert2024current}. 
However, beyond seemingly arbitrary intelligence tests, questions related to intelligence remain, because intelligence is traditionally not well defined, with the intelligence tests performed remaining rather arbitrary or human-centric and lacking a clear linear progression of difficulty levels. 
Here, we approach both as a single problem and within a quantifiable framework, providing a formal approach to a form of intelligence based on algorithmic information. 

While, in principle, LLMs have shown to be theoretically capable of Turing-complete computation ~\cite{mirzadeh2024gsm,schuurmans2024autoregressive}, this is achieved when they are augmented with external memory and appropriate decoding mechanisms \cite{schuurmans2024autoregressive}. In practice, the models we evaluate operate with standard autoregressive decoding and finite context windows, which do not constitute Turing-complete systems.

Some have claimed that LLMs, and specifically ChatGPT, have the potential to revolutionise technological interaction through accurate understanding across conversational interfaces~\cite{aljanabi2023chatgpt}.
These attributions and capabilities of LLMs have been tested in a variety of ways, from semantic comprehension evaluations in traditional Chinese medicine (TCM), through structured multiple choice and true/false questions~\cite{yizhen2024exploring}, ASCII art~\cite{bayani2023testing}, to answering open questions and using LLMs as judges of the precision and correctness of the answers provided by other models~\cite{wei2024rethinking}. 
Exhaustive and detailed tests have been performed that focus on tasks that require a grasp of a broad context, such
as quantitative investing and medical diagnoses~\cite{zhong2024evaluation}, to mention only two.

Researchers have called into question these supposed understanding capacities, claiming that a lack of novelty and an abundance of hallucinations is formal and/or informal proof of a lack of comprehension ability~\cite{si2024can,gary2}. When evaluating the intelligence and comprehension capacities of LLMs, some limitations of existing works should be highlighted:

\begin{enumerate}
\item All of them contain an element of subjectivity. Measurements of understanding rely on a human or LLM judge, where a type of definition of innovation, usability, correctness is used which could be human-centric or dependent on the context.

\item All evaluations use (mostly) text to provide a context for the questions formulated; hence there are no questions that purely test understanding beyond textual correlations.

\item The test used may take for granted that, since LLMs are trained with intelligent sources of information, this confers some intelligence on the models themselves and thus their comprehension/understanding capacities.

\item LLMs and other AI systems are not self-driven and as such cannot be reasoning agents on their own; they only act upon being triggered and prompted by humans, otherwise they do not possess any internal states (e.g. activity when not prompted).
\end{enumerate}

Other researchers, following a more abstract and formal approach, incline to the view that a test of intelligence in LLMs, which could imply comprehension, understanding, and prediction, might rely on exposing and training LLMs on highly complex datasets,
and testing how well the LLMs could apply learnt knowledge to unrelated but complex tasks (like predicting the next chess move) and reasoning tasks. 
They claim that information at the `edge of chaos', a state between non-chaotic and chaotic behaviour in dynamical systems, is more likely to help LLMs manifest intelligence~\cite{feng2019optimal}. 
Suspicions that current AI is mimicking intelligence rather than displaying it have been reported and substantiated before~\cite{gary1,gary2,bishop}. 
Thus, proposing a test that can address those concerns is very relevant.

\subsection{Compression as comprehension about (and as part of) the world}\label{sectionCompCompre}

%%%

% Algorithmic Information Theory (AIT) and algorithmic randomness provide formal foundations connecting compression with prediction, comprehension, and intelligence through the mathematical concept of Kolmogorov complexity. This connection is rooted in the observation that systems capable of effectively compressing data must understand and predict patterns within that data. 

%% Moved from section of the Foundational aspects
% Chaitin showed that a random string cannot be significantly compressed~\cite{chaitin}, implying that intelligence (as seen in systems that can compress data) involves recognising non-random patterns in data.  

%%%

As presented in Section~\ref{sectionRandomPred}, the formal equivalence between prediction and compression using martingales provides a theoretical foundation for understanding intelligence in terms of the generalisation of computational capabilities (due to universality and invariance) to predict future outcomes, while summarising this whole landscape of capabilities---due to the minimality of algorithmic complexity and maximality of algorithmic probability in the algorithmic coding theorem (ACT)~\cite{Calude2002,Downey2010}. %\todo{(R1)(R2) On compression as a measure of comprehension.}
As we discuss in this section, these are the aspects that the testing framework we propose aims at (see Section~\ref{sectionSuperARCframework}).
In the context of designing a test for intelligence, such an equivalence suggests that an agent's ability to abstract (e.g., through feature selection and model compression) and to plan (e.g., through prediction, counterfactual analysis, and simulations) are fundamentally interconnected aspects of intelligence involved in scientific or mathematical creativity.
More specifically:
%We argue that all or most of such features are related to two methods:
\begin{itemize}
    \item recursive/computable compression and decompression: seen as the summarisation or abstraction of main features (or feature selection) that can be simulated in reverse (decompression), and in contrast to simple statistical pattern-matching or statistical compression;
    
    \item process (algorithmic or symbolic) regression and prediction: formally established by AIT as equivalent to compression by way of optimal/universal simulation~\cite{levin2,levin3,levin4} through the concept of algorithmic randomness and martingales (betting strategies)~\cite{vonmises,schnorr1,schnorr2} (see Section~\ref{sectionPredictionandCompression}); or universal (Solomonoff) induction~\cite{Solomonoff1986,hutter2005universal,solo} (see also pseudocode~\ref{algo1}).
\end{itemize}

An agent that can devise or find a model that can compress a set of phenomena that when uncompressed generates this set faithfully (and beyond statistical compression) is necessarily able to comprehend it at some level~\cite{bibid}.
That is, a set of phenomena that is compressible by an agent into some first principles, or into a succinct model that when uncompressed, reconstructs, describes, and can also simulate future states of the originally described set of phenomena, needs to be comprehensible by that agent~\cite{bibid}; 
otherwise, we would have an agent that does \emph{not} comprehend the phenomena at the same time that can devise formal theories that can explain them into 
(fewer) first principles and (better) predict future events, which seems to go against a common-sense understanding of `comprehension':
on the one hand, this necessarily indicates at least some type of comprehension, e.g., a scientific or mathematical one;
on the other hand, an agent that can \emph{only} mimic, copy-paste, describe, or depict the phenomena at the same time that comprehending them contradicts any conception of `an agent able to understand something at a deeper level beyond mere appearance'.
From both scenarios, we find that abstraction and prediction arise as \emph{necessary} conditions for comprehension, particularly those intrinsic to the process of devising \emph{novel} scientific or mathematical theories.
%\todo{(R1)(R2) On compression as necessary (but not necessarily sufficient) condition for comprehension.}
As introduced in Section~\ref{sectionIntro}, instead of covering all the sufficient conditions for all types of human intelligence (including those for which compression may seem unrelated), here we constrain our study to this type of comprehension---crucial to scientific knowledge and mathematical creativity.
In this regard, compression necessarily plays a defining, encompassing role that is mathematically grounded and empirically feasible, as we explain in the paragraphs below.

It is important to clarify possible misinterpretations of the meaning of the word ``compression'' used in our framework. 
In machine learning and cognitive science, feature selection involves identifying the most relevant variables or attributes that contribute to predictive modelling. 
This summarisation process reduces the dimensionality~\cite{MILS2025InfoScie}, focusing on the most informative aspects of the data. It is, of course, a compression approach, but just a part of the one we intend to refer to. 
In our framework, beyond finding crucial features, attributes, or aspects, (algorithmic) compression into a model refers to other ``mechanistic'' relationships that are less descriptive in nature, while also guaranteeing that a model does not compromise performance.
%\todo{(R1)(R2) On performance, efficiency and compression.}
It involves reducing the complexity of the model, but often leading to a more effective generalisation and greater efficiency~\cite{calude}.
Model abstraction through effective recursive/algorithmic compression allows simulation of various scenarios when the model captures its main features, that is, its most important patterns for prediction are captured as a necessary condition for outcome prediction. 
Then, model selection happens when each outcome is compared against each time-step observation, hence updating the belief model, instantiating, and enabling `planning'.

`Compression as comprehension' is thus also tied to pragmatic characteristics such as its utility and feasibility.
For example, a compressed model in the form of a mathematical theory or a set of equations, such as a set of laws of physics, is only sound if it allows one to predict the future state of a physical system in ``a shorter time than the time taken by the actual unfolding of the phenomenon'' ~\cite{bibid}. 
Comprehension only takes place if one can understand real-world phenomena to computationally ``outrun'' reality at some sufficiently higher level---see also computational irreducibility in~\cite{Wolfram2002}.

%\todo{(R1)(R2) On compression as a measure of scientific and mathematical creativity.}
Such a process of understanding or comprehension into formal-theoretical or computational capabilities is demonstrated in the context of the scientific discovery itself.
Real-world phenomena that have the appearance of being random or unexpected become a topic of interest for research, analysis, and future development (if successful) of a more comprehensive model or theory that is then able to compress the apparent noise-like phenomena by allowing one to explain these and to predict other unfoldings from it.
In this way, science moves in an iterative pace of converting something that is currently considered ``irreducibly complex'' or unexpected into something that becomes comprehensible by theoretical means that allows computational predictions~\cite{bibid}.
For example, consider Newtonian mechanics and General Relativity in physics. 
The former has represented a highly successful compression of observational data on e.g. celestial motion.
However, to account for anomalies like the precession of Mercury's orbit, ``requires a stream of regular adjustments'' or corrective patches, which basically increase the complexity of the explanations of those anomalies to a level similar to that of describing the anomalies themselves.
General Relativity then provided a superior, more compact, and elegant set of field equations that not only subsumes the previous phenomena that Newtonian mechanics successfully accounted for, but also explains these and other ``anomalies'' in a more compressed form than before.
In fact, the continued success of the scientific method in more compact and elegant mathematical theories corroborates the conclusion that comprehension necessarily involves compression of natural phenomena~\cite{bibid}.

%\todo{(R1)(R2) On compression and other intelligence abilities.}
We know from AIT that once one can sufficiently approach universally optimal compression (i.e. approximate the actual values of algorithmic complexity), any necessary increase in complexity (of the phenomena already explained along with those to be explained) is proved to require a novel theory that is (entirely or at least a proper part of it) irreducible to the old one, like when one needs to find a new axiom.
From the algorithmic coding theorem (ACT)~\cite{Calude2002,Downey2010}, the minimisation of such an irreducibly larger quantity of complexity also corresponds to the best inference method (such as the case when one employs abductive reasoning) for the new theory (or e.g., the new axiom) one can devise from a yet unexplainable phenomenon.
As discussed in Sections~\ref{sectionAIDforSuperARC} and~\ref{sectionAGIandASI}, we argue that these two features are central to tackling the problem of AGI and ASI.
%of an Artificial Superintelligence (ASI) that can better approximate the optimal solution any agent (human or otherwise) can interactively achieve.

%\todo{(R1)(R2) On compression and human intelligence.}
The invariant and universal properties of algorithmic information imply that compression is more than a complexity index that might be correlated to abstraction and prediction.
In fact, optimal compression is only achieved by the best formal-theoretic methods proven across the whole landscape of algorithms and methods one may attempt to apply, regardless of the type of agent applying them.
Optimal compression is one such task that subsumes any other task---formalisable into an algorithm or a mathematical method that can be computationally implemented---an agent may perform in order to create a new theory that predicts new phenomena. 
Unlike directly equating compression/prediction to intelligence~\cite{Schmidhuber2009DrivenCompression} or straightforwardly applying the ACT like in other universal induction-based methods, we propose that compression is a necessary and fundamental condition for comprehension if it is achieved through (and as a product of) the interaction between the AI agents being evaluated, the evaluator agents (including the methods, frameworks and metrics we formalise), and the external real world whose process may affect and be affected by the other two entities.
For example, notice that this implies that the ACT itself becomes a constituting knowledge that the evaluator agents may devise by formalising a new mathematical theory, obtaining such a formal-theoretic knowledge after the experience (i.e., after the interactions take place).
As explained in Section~\ref{sectionAIDforSuperARC}, this is a distinctive characteristic of Algorithmic Information Dynamics (AID)~\cite{zenil2020algorithmic,zenilbook1,zenilbook2,Abrahao2021bEmergenceAIDPTRSA} upon which our proposed framework is based.

%\subsubsection{Compression and Prediction for Planning}

% Effective prediction requires understanding the underlying structures and regularities within the data. A successful martingale indicates that the sequence is not truly random and can to some extent be predicted. The ability to compress a sequence is equivalent to the ability to predict it using computable methods. If a sequence can be compressed, it can also be predicted, and vice versa.

% By incorporating the equivalence between prediction and compression into an intelligence test framework, we align the assessment of intelligence with fundamental computational principles. An agent's ability to abstract information through feature selection and model compression reflects its capacity to identify and utilise patterns within data. Similarly, its planning and prediction skills demonstrate its ability to anticipate future events based on these patterns.

As presented before in~\cite{bibid,computableuniverse}, the remarkable features of AIT~\cite{miraculous,Minsky2014ACT} discussed in the present section
seem to underpin the apparently unreasonable effectiveness of algorithmic complexity~\cite{zenil2012empirical} and computation~\cite{Wolfram2002} in explaining the natural world, including cognition, and in advancing science as the practice of finding or synthesising models that can explain and predict natural phenomena and the world.
%% Moved from compression section
Thus, by putting forward a formal and more objective approach to measuring general intelligence,
%intelligence and based on previous work on computational irreducibility~\cite{Wolfram2002} and unpredictability~\cite{zenil2012empirical}, 
we propose in Section~\ref{sectionSuperARCframework} a test for ASI and AGI based on AID and AIT, namely \emph{SuperARC}, that specifically tests recently strongly associated features with intelligence in the context of discussions of AGI~\cite{lecun,lecun2,ajay,ajay2,goertzel2007agi,lake2017building,bengio2019meta,Chollet2019MeasureIntelligence,marcus2020next}. 
While human intelligence includes many other abilities to perform a myriad of tasks, here we chose to focus on abstraction, explanation, and prediction related to building new formal theories, and to scientific creativity. 
Those are the ones for which we have computational methods and a solid theoretical foundation which has proved the universal and agnostic properties that a testing framework for general intelligence should aim at,
particularly should one want to tackle the distinction between narrow AI and AGI (see also Section~\ref{sectionAGIandASI}).

\subsection{SuperARC testing framework}\label{sectionSuperARCframework}

%\subsection{Foundations and Principles of Complexity Related to Intelligence}

Based on the theoretical background presented in Sup. Inf., 
%we propose a general testing framework, referred to as \emph{SuperARC}. %in the most general formal-theoretic sense:
%The \emph{assumption} underpinning SuperARC is that an intelligent agent must excel at both compression (abstraction) into a model (or theory) and prediction (planning) of future events (or yet unknown, unexplainable phenomena).
%Such models with greater compression (i.e., that produce shorter programs or simpler explanations for data) exhibit higher intelligence, and therefore that compression is a necessary (possibly, also the main) condition for comprehension.
%In other words, 
we ground our framework on the following aspects:
%More specifically, SuperARC is based on the assumption that:
\begin{itemize}
    \item \emph{intelligence} necessarily involves the ability to create (i.e., through abduction from the experience) or enact (i.e., through prediction from future interactions) a computational generative model that effectively explains any given data while loosing the least amount of information as possible;
    
    \item and \emph{greater} intelligence corresponds to performing prediction and abduction as close to the optimal solution as possible while maximising the compression of the generative model.
\end{itemize}

%\subsubsection{An updated definition of Intelligence}

As a metric for (general or super-)intelligence, designing tests that measure these abilities can lead to a more nuanced and computationally grounded understanding of intelligence that is applicable to biological (e.g. animal), human cognition, and computational intelligence.
This can establish a universal approach to measuring the capabilities of intelligent systems, serving as both a theoretical and a practical upper bound for the highest possible levels of compression such as model abstraction and prediction, which are believed to be fundamental features of intelligence.

As discussed in Section~\ref{sectionAITIntel}, using algorithmic complexity as a measure of model compactness (i.e., compression, conciseness or summarisation) and optimal prediction provides an agnostic quantitative metric, as its value corresponds to the shortest possible program capable of correctly reproducing (via decompression) a given dataset, and its optimal prediction value is governed by algorithmic probability. 
First, unlike standard tests that assess intelligence based on predefined `correct' answers---inevitably influenced by subjective notions of correctness---we shift the focus to identifying the shortest possible explanation for a given dataset, explanation which is proved to be sufficient for predicting not only the given dataset but also future outcomes.
%\todo{(R1)(R2) On SuperARC and benchmark contamination.}
In this context, correctness is understood purely as the ability to reproduce exactly the same original data (i.e. losslessly). %, while intelligence is measured by achieving this with the most concise program or formula as a function of optimal prediction (via decompression).
Secondly, an agnostic method aims to achieve measurable quantities as independent of human biases as possible, including those high-order biases in the scientific practice such as when one chooses to employ one formal theory instead of another in order to model certain phenomena (see Section~\ref{sectionCompCompre}).
%\todo{(R1)(R2) On compression subsuming comprehension tasks.}

%% moved from comprehension section

Beyond a measure of a single-purpose compression task (as discussed in Sections~\ref{sectionCompCompre},~\ref{sectionRandomPred}, and~\ref{sectionAIDforSuperARC}), the SuperARC test is a proposal to capture the potential future trajectories leading to hybrid neurosymbolic systems more capable of the abstraction and planning, deemed central to what has been conceived of AGI and ASI~\cite{schmidhuber2007godel,lecun,lecun2}, one that may take into account statistical pattern matching, but favours symbolic regression and program synthesis as a test of intelligence based on optimal inference rather than statistical `reasoning'. 
The test proposed expands current efforts to characterise AGI such as the Abstraction and Reasoning Corpus (ARC) challenge~\cite{Chollet2019MeasureIntelligence} which have been suspected to be `hackable' from test result leaks because the test data set is fixed (even if part of it is concealed but prone to be leaked).
Unlike recent results in the ARC-AGI test, our results find a similar lower performance than that reported in a recent mathematical benchmark test~\cite{math}, with the advantage that our proposed test does not require the selection of human mathematical problems and the test problems can be dynamically generated with test elements introduced cheaply and efficiently. %\todo{(R2) On SuperARC, efficiency, and feature selection bias.}
Although this new test may require the selection of objects and elements such as sequences, unlike the original ARC challenge tests, this selection can be based mainly on quantitative measures of complexity and less on human selection.
%%%

These features are crucial in avoiding biases introduced by the datasets, such as benchmark contamination, when evaluating the performance of an AI algorithm.
Given that one would be trying to measure the ability of the learning algorithm to predict phenomena whose type or class was not the one of the data it was trained for in the first place, this is especially the case in \emph{zero-shot learning} scenarios where any small leakage of data with information about the (upcoming and irreducibly new) test to be performed makes a big difference in the score.
Due to the mathematical properties in AIT discussed in Sections~\ref{sectionAITIntel} and~\ref{sectionAIDforSuperARC}, SuperARC avoids human-centric and other cognitive biases because lower (algorithmic) complexity (higher compression, or equivalently, higher algorithmic probability) of a model is proven to indicate better overall prediction capabilities, regardless of the nature of the new phenomena or the type of data on which one is trying to measure the generalisation capabilities of the AI algorithm.
For example, even if one can update the benchmark test in practice with a new type of task to be performed, this possibility itself assumes that a new type of task might be known to us, rendering the test inherently prone to contamination.
Following from the properties of algorithmic probability, SuperARC quantify prediction in new contexts and potential different scenarios without the need of a new type of task, distinct data, or posterior apprehension of previously unknown phenomenon.

\color{black}

\subsection{The Role of SuperARC on Distinguishing Algorithmic from Statistical Prediction}

While prediction is fundamental to both human-centric and algorithmic benchmarks, the nature of what is being predicted differs fundamentally. Human-centric benchmarks evaluate whether models can predict outputs that humans would generate given specific inputs, which is a task solvable through statistical pattern matching over human-generated corpora. 
This is because as models are increasingly trained on datasets that cover more of human knowledge, performance on these benchmarks asymptotically approaches data memorisation rather than genuine understanding.

SuperARC, by contrast, evaluates whether models can induce the algorithmic structure that underlie the sequences, i.e., that can find the minimal program that generates the observed data. 
This capability is irreducible to pattern matching because:

\begin{itemize}
    \item Infinite hypothesis space: Unlike human-centric tasks with finite answer sets, algorithmic induction in principle searches over an infinite space of possible programs, and thus possible formal theories;
    \item Distribution shift immunity: Novel algorithmic patterns (new combinations of primitives) are fundamentally out-of-distribution, requiring genuine abductive reasoning, such as when one devises a new axiom;
    \item Compression-prediction duality: Theorem \ref{thmIncompUnpredict} establishes that predictions success is equivalent to compression over the algorithmic space, which subsumes the statistical space although the equivalence does not require statistical evaluation nor success;
\end{itemize}

As shown by our results, we argue that such a theoretical difference uncovers a practical consequence: models can simultaneously improve on human benchmarks while regressing on algorithmic reasoning. 
This divergence should be impossible if both algorithmic prediction and statistical prediction were supposed to measure the same underlying predictive capability, and therefore SuperARC provides empirical evidence that captures a distinct and arguably more fundamental aspect of intelligence.

\subsection{CTM and BDM: A neurosymbolic approach to Superintelligence benchmarking}\label{sectionBDMforASI}

The SuperARC framework accommodates any type of data as input-output pairs, requiring only a complexity-based metric to be predefined. 
To achieve this, in addition to approximate methods to algorithmic complexity, such as LZW and ZIP which are more closely related to \emph{Shannon Entropy}~\cite{zenilreview}, we use the \emph{Block Decomposition Method} (BDM) as our gold-standard approach to algorithmic compression that goes beyond statistical compression or statistical pattern-matching~\cite{bdm}.
%The latter is based upon the Coding Theorem Method (CTM)---a direct consequence of Algorithmic Probability~\cite{ctm1}. %---and therefore also able to quantify algorithmic probability.
Using the principles of both classical and algorithmic information theories, BDM combines the calculation of the global Shannon Entropy rate of the object with local estimations to algorithmic complexity of smaller blocks into which the object is decomposed for which values are found in a pre-computed database of direct approximations of algorithmic probability. 
By combining both statistical and algorithmic inference methods, one way to think of BDM is by depicting it as a Deep Learning Transformer which aims to build a predictor that maximises the probability of being correct in explaining the data by looking for long-range and short-range correlations. 
The difference, in this case, is that long-range correlations are covered by Shannon Entropy (not fundamentally different from Transformers) but short-term correlations are estimated using the principles of algorithmic probability through the ACT~\cite{ctm3,ctm1,ctm2,zenilbook2} (see Section~\ref{sectionAITIntel}). 
See the discussion on the limitations of BDM below. %\todo{3R3 concerns about BDM to be inserted here.}
%This therefore combines the two best methods for statistical and algorithmic inference.

In this manner, BDM is based on combining the best capabilities of Shannon entropy-type metrics to find patterns (e.g. block entropy rate~\cite{Ozelim2024AssemblyTheoryReduced}) with universal (Solomonoff) induction-based approaches (such as the minimum description length~\cite{mdl}) through algorithmic complexity, and thus deals with uncertainty in an optimal Bayesian fashion based on the principles of algorithmic probability~\cite{solo}. %\todo{(R1)(R2) On BDM as a metric for comprehension beyond exhaustive research.}
BDM improves upon Shannon entropy and LZW compression, which are limited to detecting only statistical regularities, this is, pure pattern-matching approaches. 
%\todo[inline]{Unpack and move to highlight the limitations of LZ and AIXI universal intelligence. (L) You may check that I took another path in order to reply both reviewers 2 and 3 (I focused on indicating that intelligence has other pillars and kept LZ and AIXI as they were). What do you think?}
%\todo[inline]{\fsa{Ok. I will take a look.}}
In fact, BDM subsumes these methods, and therefore one can only do better in capturing structure than statistical compression algorithms, as BDM detects both regular statistical patterns and recursive ones with causal generative signatures~\cite{zenilreview,bdm,Ozelim2024AssemblyTheoryReduced}.
See also Section~\ref{sectionAIDforSuperARC}. By recursive we mean exactly those that are not statistical in nature (e.g. the digits of the mathematical constant $\pi$ does not display any statistical patterns but it is recursive).

The BDM relies on the following assumptions:

\begin{enumerate}
    \item in the case of small enough objects, their \emph{algorithmic complexity} can be approximated using an exhaustive search (sometimes guided, e.g. with AID);
    \item for larger objects, breaking them into smaller parts allows for the approximation of the overall complexity by summing the complexity of individual blocks, with a correction factor to account for interactions between the blocks;
    \item for every other length, values of Shannon Entropy rates are calculated and combined with the previous values by using the same principles of information theory.
\end{enumerate}

Formally, let \( x \) be a string divided into blocks \( x_i \), with \( x = x_1 \oplus x_2 \oplus \dots \oplus x_n \), where $\oplus$ denotes a concatenation operator. The \emph{BDM complexity} of a string \( x \), denoted by \( \text{BDM}(x) \), is given by:

\begin{equation}\label{eqBDMdef}
\text{BDM}(x) = \sum_{i=1}^{ n } \text{CTM}(x_i) + \log m_i
\end{equation}
where:
\begin{itemize}
    \item \( \text{CTM}(x_i) \) is the algorithmic complexity approximation for block \( x_i \), derived from the Coding Theorem Method (CTM).
    \item \( \log m_i \) is a correction factor accounting for the multiplicity $ m_i $ of how many times the block $ x_i $ appears.
\end{itemize}
For a generalised version of BDM holding for any encodable object, see~\cite{Ozelim2024AssemblyTheoryReduced}.

The \emph{Coding Theorem Method (CTM)} is a method based on the ACT~\cite{kolmobook,zenilbook1} as in Section~\ref{sectionAITbasics}, which connects probability to complexity, randomness, and prediction~\cite{ctm3,ctm1,ctm2,zenilbook2}; and the ACT underlies the universal induction-based methods applied to Artificial Intelligence.
CTM works by searching for all the formal-theoretic explanations (models or programs) for an object that are shorter than the object itself~\cite{bdm} in order to calculate the ratio of those explanations of a particular object with the all the explanations found for any object. 
From the value obtained for each of these ratios, one can approximate the algorithmic probability of an object, and thereby its algorithmic complexity via the ACT so that a list of these pre-computed probability values is built, which in turn can be used to approximate the universal distribution~\cite{ctm1}.
%CTM produces and stores the set of G\''{o}del numbers that correspond to all the programs that compute an object, such as an integer sequence, up to the given digit or any other recursively describable~\cite{ctm1,ctm2}. 
%Each program can then be uncompressed from its unique (G\''{o}del) number and run to produce the next digit for predictive purposes with the programs themselves the abstract future-planning models. 
%Thus, CTM maps sets of micro programs (e.g., small Turing machines) to small assembly objects for which it can empirically estimate the algorithmic probability $ P( 
%\cdot ) $ of an object, such as a time series. %based on the following relationship~\cite{ctm4}.

%\todo{(R2) On BDM and its applicability.}
On the one hand, CTM provides an approximation to \emph{algorithmic probability} \( P(s) \) by connecting the empirical frequency of occurrence of an object produced by a random computer program with its \emph{algorithmic complexity} \( K(s) \) %, using the relation \( K(s) \sim -\log P(s) \) 
and also keeps track of the set of programs that generated the original object.
(See Section~\ref{sectionAITIntel}).
On the other hand, BDM offers a method to map the micro-programs produced by CTM to their corresponding pieces from the larger object to explain by decomposing the original object into smaller blocks for which micro-programs have been found by CTM with a correction factor for block interactions (e.g. repetitions).
While CTM operates by brute force and thus is only effective for small programs/models, BDM leverages the pre-computed distributions that can be queried in linear time and stitches together longer explanations from small computer programs according to the rules of information theory to guide the search of the best sequence of programs explaining larger objects, thereby constituting a method that approximates the optimal causal explanation of the objects.
BDM also allows massive parallelisation because objects with low complexity (i.e. higher causal impact at the global level) are the most frequent according to algorithmic probability and therefore are exponentially more frequent, counteracting their intractability~\cite{bdm}. 

%\todo{(R2) On BDM and its applicability.}
With limited computational resources, because of the limitations from the CTM, BDM behaves exactly like algorithmic complexity at the local scale and exactly like entropy at global scales~\cite{zenilreview,bdm,Ozelim2024AssemblyTheoryReduced}. 
%BDM includes an entropy-like component to account for pattern repetition at larger scales beyond the local scale for which all the algorithmic generative models are known (i.e., precomputed). 
In principle, with unbounded computational resources, all the algorithmic generative models at any scale would be known/computed. %\todo{Answer to 3R3 concerns here too.}
As a consequence, BDM would return the optimal value given by algorithmic complexity, and therefore would achieve optimal compression in the general case.
In addition, for the particular cases in which the conditions for optimal statistical compression (like those discussed in Section~\ref{sectionRandomPred}) are met, BDM is also proved to perform as optimally as entropy does because: at the local scales, algorithmic complexity already encompasses and subsumes entropy; and at the global scale, BDM converges to entropy via the CTM.
Therefore:
\begin{itemize}
\item in practice, BDM always performs equally or better than entropy;
\item BDM through CTM converges to algorithmic complexity in the limit, outperforming any statistical compression method;
\item both in principle and in practice, BDM remains sensitive to underlying structures even for objects with maximal Shannon entropy. 
\end{itemize}
%\todo{(R2) Advantages of BDM.}
Following from the fundamental properties of AIT (discussed in Section~\ref{sectionAITIntel}), such as universality, invariance, maximality, and optimal prediction, BDM offers a `\emph{principled}' alternative in comparison to statistical measures.
It is currently the only viable and computable approximation to algorithmic complexity grounded in AIT beyond statistical compression algorithms.

Thus, BDM is a hybrid neurosymbolic~\cite{ShethNeurosymbolicAIWhy2023} method that combines statistical machine learning and symbolic regression, and prediction that can be applied to inverse problems in causality~\cite{nmi,iscience}. %AI and Superintelligence (sometimes confounded with AGI) for program and explanation synthesis. 
BDM can be thought of as a quintessential type of neural network transformers (as in self attention) where it estimates the local (short-range) causality through algorithmic complexity while computing long-range correlations through Shannon Entropy guaranteed convergence (worse case)~\cite{Ozelim2024AssemblyTheoryReduced}.
Such a benchmarking method has already been reported in applications to data summarisation~\cite{MILS2025InfoScie} and in various fields ranging from cell and molecular biology to genetics~\cite{iscience,nar} to biosignatures~\cite{zenilld,abrahao,Uthamacumaran2024SalientnpjSBA}.

%{\color{red}
BDM is an approximation to algorithmic complexity and probability with its known limitations that demand explicit discussion. 
BDM's complexity estimates depend on choices of block size for decomposition, with different block sizes potentially yielding different rankings of sequence complexity. 
As mentioned above, this limitation occurs because of limited computational resources, since in the asymptotic theoretical limit, the algorithmic complexity could be approximated across any coarse-graining scale.
For the short range or smaller blocks, BDM uses a precomputed table of short programs (or equivalently, Turing machines with a few limited number of states) for which the Busy Beaver values are known and therefore one can solve the halting problem.
In the long range or for the largest block sizes possible, BDM upscales those actual algorithmic complexity values via Shannon (block) entropy.
Therefore, due to limited computational resources, the resulting BDM value may inherit the same limitations of entropy in the long-range scenario.
In addition, the empirical application of the method's reliance on decomposing sequences into overlapping or non-overlapping blocks means that patterns spanning boundaries between blocks may not be captured optimally, potentially underestimating complexity for sequences with long-range dependencies. 
%Furthermore, the CTM reference tables that BDM relies upon mix statistical frequency information from exhaustive enumeration with algorithmic structure, introducing elements of statistical complexity alongside purely algorithmic measures. 
These are fundamental characteristics of practical computable approximations to algorithmic complexity or algorithmic probability, which remain uncomputable in the general case.
    
Despite these limitations, our use of BDM is justified for reasons that directly address concerns about result interpretation: 
first, the models we evaluate fail predominantly on short sequences where BDM's approximation is most accurate and where the gap between estimated and actual algorithmic complexity values is minimal;
secondly, our conclusions do not depend on fine-grained complexity distinctions but on coarse patterns (models fail across broad complexity ranges rather than at specific threshold values where approximation errors might matter);
thirdly, our neurosymbolic baseline employs actual program synthesis through systematic enumeration rather than BDM estimation, yet reaches qualitatively similar conclusions.
While future work comparing alternative complexity metrics may uncover or highlight other aspects or discrepancies, we argue that the robustness of LLM failures across these multiple lines of evidence suggests our core findings about inadequate algorithmic reasoning in current models remain sound regardless of specific approximation method choices.

\subsubsection{Applicability of CTM and BDM to abstraction and planning in machine learning}\label{sectionWhyCTMBDM}

BDM with CTM can be applied both as a reference and as a direct generative model.
This is because it provides a fundamental complexity-based value estimation that can guide and evaluate other predictive and learning approaches, but also as a stand-alone predictive system. 

%\begin{itemize}
%    \item 
    CTM helps identify the set of candidate underlying generative mechanisms and provide a set of models from which it can actively predict future values by running it further into the future providing a set of projections. CTM forecasting requires an iterative refinement process in which multiple possible generative programs are tested and updated. CTM can help select the most likely program candidates by favouring those with lower complexity in accordance with the principles of algorithmic probability. 
    
%    \item 
    In a predictive task, multiple candidate programs generated by CTM are evaluated against new observations, discarding those that are not consistent with the new data while retaining the set of shortest valid programs that do. Planning requires CTM as the algorithmic mechanism to iteratively refine predictions from projections. CTM serves as a criterion for model selection---helping identify which approach best maintains parsimony and explanatory power---rather than functioning as a decision-making agent of its own.
    
    BDM then stitches multiple programs that can explain longer pieces of data and larger objects by using the rules of classical information theory, serving as a reference point to compare different models based on how well they align with the inherent complexity of the data. By breaking down an object into smaller pieces and estimating their individual algorithmic complexity using CTM, BDM provides a tighter recursive upper bound to traditional pattern matching. BDM leverages, therefore, both algorithmic and classical information theory as a proxy for deeper connections to causality, allowing it to indicate how predictable a time series or integer sequence is. Both CTM and BDM combined can benchmark different models on the basis of how efficiently they approximate the set of shortest best explanatory and generating mechanisms.   

%\end{itemize}

The way BDM approaches uncertainty is to update the belief at time $t$ of an object $s$ (e.g., an integer sequence), and choose a (small) program $p^\prime$ to explain for the next digit $i\in s_{i-1}$ deviating from the previous hypothesis $p$;
or in case we do not have (or we cannot obtain) such a program for this observation, we combine smaller programs $p^{\prime\prime}$ to explain observation of digit $i\in s_i$ at index $t+1$. 
In this manner, the ability of BDM to capture both local and global patterns in a time series or integer sequence makes it a powerful tool for approximating complexity and enabling prediction, aligning with the principles of algorithmic probability and Levin's universal distribution.

BDM shows some fundamental similarities but in pure form to ``Attention is All You Need'' algorithms and LLM's by assigning different weights to different parts of an object focusing both on short-range and long-range correlations where the short-range is recursively correlated hence based on causally generated models for that patch of data unlike LLMs and other ML approaches that rely only on Shannon-entropy-based correlations or basic pattern-matching that BDM only uses for its long-range correlations. BDM is therefore a proper generalisation of the short- and long-range capabilities that gave LLMs their particular advantage in language~\cite{bdm}. Together with CTM as a universal generator~\cite{ctm3}, the CTM/BDM combination represents a model of models of languages, where languages are all computer languages, and a superset of LLMs themselves.

%\todo{(R1) Lines 274-278} 

As mentioned above, a limitation of CTM is that running CTM to approximate model compression and achieve optimal prediction is computationally very expensive. If there were infinite resources, CTM would perform perfect recursive compression and provide the most optimal answer to any computable question given an observation. However, even with access to infinite resources, there are no theoretical or practical guarantees of LLM convergence to any optimal answer. In practice, LLMs are currently more expensive in applications where approaches like CTM could deliver better results (such as for this benchmark, empirically proven to better characterise questions and predict answers encoded in the form of binary sequences) without spending billions of USD in training giant neural systems like LLMs. However, our point is that one does not need to pick one over the other as they can be combined to provide the best approximation to both an optimal but efficient path to an answer under time and resource restrictions. In this regard, CTM/BDM is a resource-bounded approximation to optimal inference that combines pure forms of each side (neuro-based on classical statistics, and symbolic-based on optimal theory). 
The CTM/BDM combo represents the purest form of neurosymbolic computation with no extra steps.

In the framework we propose, CTM and BDM are used as a benchmark to evaluate model performance and as a representative of a Universal AI~\cite{hutter2005universal} method capable of ASI~\cite{Solomonoff1986}.
They can be applied to test both:

\begin{itemize}
    \item \emph{compression as model abstraction}: The BDM can approximate the algorithmic complexity of a time series by decomposing it into smaller subsequences (blocks), computing the complexity of each block using CTM, and summing up the block results. This serves as a measure of the recursivity of the time series but also serves as a method to find generating mechanisms (a set of algorithms that produce each past and possible future element/token of an object, in particular, a time series).
    \item \emph{prediction as planning}: Using the BDM complexity as a proxy for the time series' regularity, one can infer the predictability of future values. Lower BDM complexity implies a simpler underlying structure, which can help in forecasting future elements of the series---which is similar to how algorithmic probability and universal distribution can be used for predictive modelling. (See Sections~\ref{sectionPredictionandCompression} and~\ref{sectionLevinsearch}). This is related to planning, because once several program pathways are identified, one can verify each against the next token and update the program set (by discarding those programs that did not fit the next token) while keeping the shortest program criterion.
\end{itemize}

\subsection{A method for measuring comprehension via algorithmic probability}

As explained in Sections~\ref{sectionAITIntel} and~\ref{sectionBDMforASI}, BDM is a divide-and-conquer method which extends the power of a CTM that approximates local estimations of algorithmic complexity via theory of algorithmic probability, a foundational result established in AIT.
The method consists of finding the sequence of computer programs that can generate the original piece of data---each program represents a hypothesis or model for the time series and a sequence of datasets that can be interpreted as time series, binary and non-binary---, providing a closer connection to complexity (or irreducible information content) than previous attempts based on statistical regularities such as popular lossless compression schemes~\cite{bdm}. 

Based on AIT, we measure comprehension of LLMs (see Section~\ref{sectionCompCompre}) with a test designed to assess the model's ability to generate code or mathematical models/formulae that compress sequences of increasing complexity. 
Non-binary sequences are categorised into three levels of complexity (Low, Medium, and High) representing datasets that exhibit simple, intricate, and random patterns, respectively. 
Binary sequences, on the other hand, are classified as either random or what we call `climber strings, low-complexity strings as defined in the following section. Thus, a pragmatic compression-as-comprehension test is designed and applied to various LLM models and versions, encompassing test elements of diverse complexity classes which can be understood and compared individually and collectively.

\begin{algorithm}[ht!]
\caption{Pseudo-code for SuperARC framework}\label{algo1}
\begin{algorithmic}[1]
\Require \\
\begin{itemize}
    \item $D_{low}$, $D_{medium}$, $D_{high}$ (datasets of any type with low, medium and high complexities with sizes given as $ \left| \cdot \right| $. These are needed to ensure complexity diversity 
    %Hector: within the dataset,
    but the choice of three groups is arbitrary and can be changed by the user.);
    \item $enc$ (encoding chosen to put the datasets in a common format);
    \item $\mathcal{M}$ (complexity metric used to qualify the datasets and quantify the complexities of the models created by LLMs);
    \item $\mathcal{T}$ (test formula to evaluate a candidate model).
\end{itemize}
\State $c_{\mathcal{M}} \Leftarrow $ an array containing binary values.
\State $Aux_{\mathcal{M}} \Leftarrow $ an array containing auxiliary values.
\State $All_{\mathcal{M}} \Leftarrow $ an array containing complexity values.
\For{$k \in \{low,medium,high\}$}
\State $D_{k,encoded} \Leftarrow $ encoding of $D_k$ using $enc$ (the UTF-8 or ASCII binary representation of strings or a binary representation of integers, for example).
\For{$j \in \{1,2,...,|D_{k,encoded}|\}$}
        \State $R_{k,j} \Leftarrow $ the response obtained from prompting a LLM model to write a program to reproduce the $j$-th element of $D_{k,encoded}$.
        \State $c_{k,j} \Leftarrow $ a binary variable indicating if the output obtained after running $R_{k,j}$ is correct (equal to the input dataset) or not.
        \State $\mathcal{M}(R_{k,j}) \Leftarrow$ the complexity of $R_{k,j}$ according to $\mathcal{M}$.
        \State $a_{k,j} \Leftarrow $ a vector with real-valued variables representing the result of applying auxiliary functions to $R_{k,j}$.
        \State Append $c_{k,j}$ to $c_{\mathcal{M}}$.
        \State Append $\mathcal{M}(R_{k,j})$ to $All_{\mathcal{M}}$.
        \State Append $a_{k,j}$ to $Aux_{\mathcal{M}}$.
\EndFor
\EndFor
\State $\mathcal{T}(c_{\mathcal{M}},All_{\mathcal{M}},Aux_{\mathcal{M}}) \Leftarrow$ the test score for the candidate model.
% \State $\mathcal{T}_{positive} \Leftarrow \alpha (\mathcal{T} - \min(\mathcal{T})) + \epsilon$ \Comment{Affine transformation to ensure all values are positive and differences remain proportional}
\end{algorithmic}
\end{algorithm}
In other words, the SuperARC framework assesses how the LLM model is able to generate an algorithm $\mathcal{A}$ such that, when applied to the input data set $\tau$, it is able to compress this input by learning its features and producing a compressed representation $\partial$. Then, by inverting such an algorithm and obtaining the algorithm $\mathcal{A}^{-1}$, the inputs $\tau$ are obtained losslessly with minimal complexity of the combined algorithms according to a complexity metric $\mathcal{M}$.
From AIT,  we have that universal induction indicates the best way to predict future elements of a sequence as favouring the simplest (i.e., the least complex) hypothesis or explanation---which aligns with the concept of Occam's razor, as discussed in Section~\ref{sectionAITIntel}. 
By minimising the complexity of the description of the data ($ \mathcal{M}\left( \mathcal{A}^{-1} \circ \mathcal{A} \right) $), the theory effectively formalises prediction ($ \mathcal{A}^{-1} \circ \mathcal{A} \colon \{\tau \to \partial \to \tau\} $).
In this manner, the framework can be described as the pseudo-code in Algorithm~\ref{algo1} for which
the LLM is presented with the following task:

\begin{eqnarray*}
\minimize_{\mathcal{A},\mathcal{A}^{-1}}&& \mathcal{M}\left( \mathcal{A}^{-1} \circ \mathcal{A} \right) \\
\subjto && \mathcal{A}^{-1} \circ \mathcal{A} \colon \{\tau \to \partial \to \tau\}
\end{eqnarray*}

It is important to clarify that the encoding $enc$ in Algorithm~\ref{algo1} does restrict the analysis. 
For example, different data types could be encoded as vectors obtained in the latent space of a given deep neural network. As long as the encoder algorithm is known and common to all the input data, the framework can be applied because of the theorems in AIT. 
In particular, the information non-increase theorem~\cite{Calude2002} indicates that, for any computable function $f$, the inequality $ K(f(x))\leq K(x) + K(f) + \mathbf{O}(1) $ holds. 
Therefore, once $f$ is fixed for all data sets considered, $K(f)$ becomes an additive constant that does not affect the analysis when $K(x)$ is %constrained from above and 
used to investigate the value of $K(f(x))$. 
In other words, the encoding is not important as long as it is known and kept fixed during the analysis.

It should also be noticed that CTM/BDM is not purely a brute-force approach~\cite{bdm}. %requires no previous data and its current implementation required orders of magnitude less computational power. 
Although CTM alone would be a brute-force approach that seeks the shortest computer programs explaining the data, BDM is not (see Section~\ref{sectionBDMforASI}). CTM/BDM operates by exploiting the best of both worlds~\cite{bdm}, operating at the fine balance between what traditional Machine Learning and Deep Learning approaches implement, while also combining it with optimal Bayesian causal inference~\cite{zenilreview} or algorithmic deconvolution~\cite{nmi}.
As further discussed and explained in Section~\ref{sectionAIDforSuperARC}, we have called this approach Algorithmic Information Dynamics (AID)~\cite{zenil2020algorithmic,zenilbook1,zenilbook2}.
%\todo[inline]{Move and unpack. (L) Move where? Any particular point?}
%\todo[inline]{\fsa{No worries. Thats on me.}}

In order to present a quantitative implementation of a test following the SuperARC framework, an exploratory analysis is needed. This will be described in the next Section~\ref{sectionDesign}.

\subsection{Design of experiments}\label{sectionDesign}

To evaluate how LLM models can be assessed within the SuperARC framework, we consider datasets composed of non-binary and binary sequences. 
It is worth highlighting that this choice is not mandatory, and all data should be encoded consistently (see also Section 7.13).
Different encodings may lead to different BDM values and thus other benchmarks may favour one type/structure of data or the other.
Nevertheless, as BDM is an approximation to algorithmic complexity, AIT guarantees that algorithmic probability converges to the optimal solution in the asymptotic limit (if enough computational resources are provided).

Although prompting has been shown to considerably impact the performance of LLMs in a code generation task~\cite{wang2024,Li2025}, we use the simplest possible prompt to avoid providing additional information to the LLM which could bias its output (even if towards better codes). Also, for the same reasons, we performed zero-shot learning tasks.

The non-binary sequences of integers used in the questions were divided into 3 levels of complexity, as indicated in the previous subsection. Intuitively, the complexity levels could be explained as follows:

\begin{enumerate}
\item Low Complexity: Sequences of digits or integers whose pattern is easily recognisable by a person and highly compressible. They have low CTM/BDM values.

\item Medium Complexity: Sequences of digits integers generated recursively with longer formulas than those in the simpler set. They have intermediate CTM/BDM values.

\item High Complexity: Random-looking sequences of digits or integers. They have high CTM/BDM values.

\end{enumerate}

The following experiments were carried out: 

\begin{itemize}
\item \emph{Next-digit prediction task with binary and non-binary sequences}: We prompted LLMs specialising in time series forecasting to predict the digits of non-binary sequences of increasing complexity of two type. The first type are random binary sequences according to increasing CTM/BDM, and the second type are called `\emph{climbers}'. 
\begin{itemize}
    \item \emph{Climbers} are strings that when sorted by algorithmic probability in descending order (highest to lowest probability), or algorithmic complexity in ascending order (lowest to highest randomness), these binary sequences are longer than strings in their same complexity group defined as strings with the same or very close complexity values as measured by BDM but of significantly longer length than them. This means that for these strings, their complexity is definitively not driven by string length only but by (simple) their internal structure, aligning with an intuitive understanding of simplicity vs. randomness in sequence structure~\cite{soler2014calculating}. In other words, these are strings that clearly correspond to lower randomness values because they show lower complexity estimations compared to shorter strings in the vicinity. For example, the sequence 0101010101... up to certain finite size $n$ is clearly less algorithmic random and therefore more algorithmic probable than any other more random looking string, short or long of the same size $n$, and therefore such a patterned sequence must appear earlier in a complexity hierarchy if BDM works correctly. So, knowing these are highly structured strings with high algorithmic probability, we tested whether LLMs would identify them by producing short models and better predictions for them compared to others.
\end{itemize}

\item \emph{Free-form generation task with binary and non-binary sequences}: We challenged advanced language models, including GPT-4o, GPT-o1, Claude 3.5 Sonnet, GPT-4o-mini, Grok, o1-mini, Qwen, and DeepSeek, to generate models, algorithms, formulas, or Python scripts capable of reproducing specific target sequences.

\item \emph{Code generation task with non-binary sequences}:
%\footnote{ChatGPT4-o was released at the time this paper was being written. All original experiments were systematically compared with the outputs generated with this new model, with only minor variations found in the code generated for the ArnoldC language. We concluded that results from versions 4-Turbo and GPT4-o could be considered equivalent.} 
An answer was requested to generate source code that would produce sequences of numbers using prompts of the following type: 
\begin{quote}
    ``With no additional explanations or comments or notes, write the code in \{\} programming language
to produce the sequence {[}sequence{]}.
\end{quote}
A full list of all sequences can be found in the Sup. Inf.. Each prompt was submitted with varying values for the temperature parameter: {[}1, 0.7, 0.5, 0.2, 0.001{]}, allowing for a comparison of its effect on the quality of the outputs.

Each prompt was formulated in such a way that it was expected that the LLM would return the code generating the defined sequences in the following programming languages: ArnoldC, C++, Python, Mathematica, Matlab, R, JavaScript. After the codes were generated, they were executed, and their performance was compared.

\end{itemize}

\subsubsection{Code and free-form generation tasks}
Code generation in different programming languages was performed exclusively using non-binary sequences of increasing complexity and only run by ChatGPT. In contrast, free-form generation was conducted using both non-binary and binary sequences and prompted to a list of the most prominent LLMs. Depending on the case, the following processing steps were applied according to the Algorithm \ref{algo1}: 

For the $j$-th element of $D_{k,encoded}$, $k \in \{low, medium, high \}$, the output code (able to reproduce these elements) provided by the LLM model was $R_{k,j}$. Then, for these, after being logically evaluated to ensure that they produced the expected results, the following functions were applied.

\begin{itemize}
\item Auxiliary functions:
\begin{itemize}
    \item The script and model/formula lengths generated by LLMs were measured by the number of characters.
\item Since program or model/formula length was taken as an indicator, and sequences were defined as either single- or multi-digit numbers, a process called normalisation was applied to the original code generated.
This normalisation took out repetitions of the entire sequence from the code if this was included. For example, if a script that aims to reproduce the sequence `1, 2, 3, 4' were to be `Print(1, 2, 3, 4)', after being normalised, it would be transformed
into `Print()'. In this way, we obtained lengths of normalised and non-normalised answers. 
\item Compression: The zlib algorithm was applied to the normalised and non-normalised answers generated; also to the target sequences of digits alone in such a way
that we obtained ASCII representation of the compressed and non-compressed variations of all scripts and their lengths. 
\item For the code in different programming languates, a compression percentage measurement was designed: this is an indirect
measurement of compression based on the number of elements of a sequence
and their order of appearance in the answer to a question. For example,
if the target sequence is ``1, 2, 3, 4, 5''  the code Print({[}1, 2 , 3, 4, 5{]}) is considered to be 100\% uncompressed,
not only because it contains all elements of the original sequence
but it also keeps its original order. On the other hand, the code For i=1
to 5 Print(i)  is considered to have a higher degree of compression,
since it only contains 2 of the original elements, but the logic to generate
it ``lives'' in the code. Additionally, the code repeat print(n+1)
is considered more compressed. 
\item A set of filters was designed to study our results and they were applied accordingly if non-binary or binary sequences were the target: 
\begin{itemize}
\item \emph{Print code} (applicable to binary and non-binary sequences): this type of program could be of two types: a) the target sequence defined as a variable or a set of variables followed by a print(sequence), for example a=`1,2,3', print(a), b) a simple print(Sequence) without definition of variables, for example print(`1,2,3').
\item \emph{Correct code} (applicable to binary and non-binary sequences): if the given answer by any LLM models generated the target sequence. 
\item \emph{Print-correct} (applicable only to non-binary sequences): the combination of the two above. 
\item \emph{Incorrect-print} (applicable only to non-binary sequences): the negation of the previous one.
\item \emph{Ordinal} (applicable only to binary sequences): The model or formula exclusively references the positional arrangement of digits to reproduce the target sequence.
\end{itemize}
\item The application of filters was done over all our measurements, allowing classification by averages of compressed, not compressed, normalised,
and not normalised answers, filtered by prints, or correct and all its combinations.
\end{itemize}
\item Correctness variable: Computer programs and models/formulae were evaluated or executed in their respective compilers/interpreters to verify if they generated the target number sequences correctly. 
\end{itemize}

\subsubsection{Next-digit prediction task}
\label{nextDigitPredictionSubsection}

For the next-digit prediction task we used binary and non-binary sequences. We compared results obtained with different LLMs specialising in time series forecasting to predict values in the sequences used in our experiments. The models used included Chronos, TimeGPT-1, and Lag-Llama. Our criteria for selecting these models can be summarised as follows: 
\begin{enumerate}
\item researchers reported very high-quality predictions in zero-shot tasks, i.e., in time series never seen before;
\item they were compared to traditional machine learning models, showing superior results;
\item they are reported to capture dynamics in real-world datasets rather than relying on simple statistical patterns;
\item authors advocate for the superiority of LLM architectures in time-series forecasting;
\end{enumerate}

\begin{center}
\begin{figure}[H]
\makebox[\textwidth][c]{\includegraphics[width=.87
\textwidth]{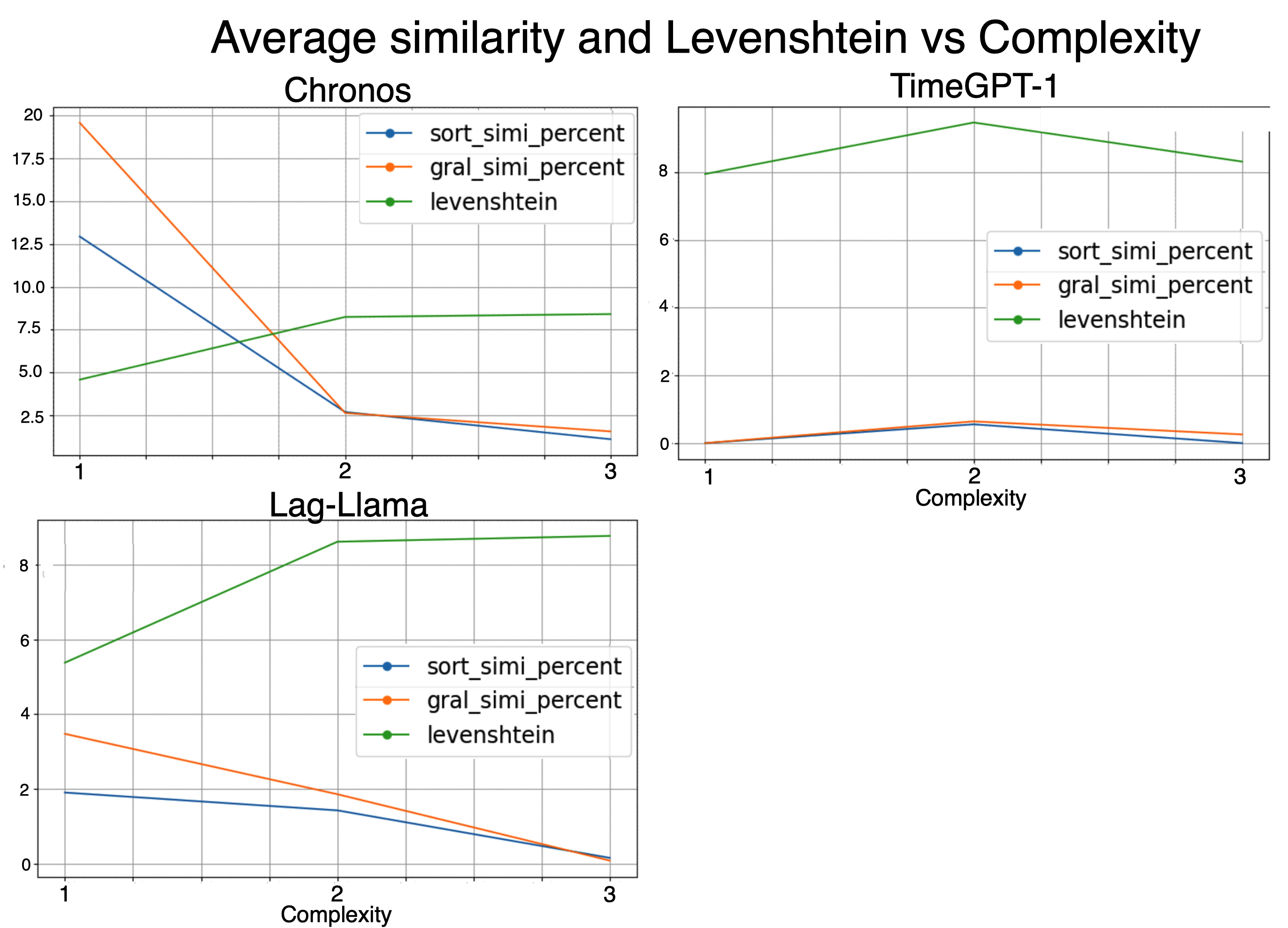}}%
\caption{Similarity over predictions with Chronos, TimeGPT-1 and lag-llama. Methods and descriptions in the Supp. Inf.}
\label{TimeseriesLLMPlots}
\end{figure}
\end{center}

We split our sequences into several segments, using the models described to predict the remaining portions, which correspond to 10\%, 25\%, 50\%, and 75\% of the sequence. This approach divided the sequence into a `root' and a `target'.
%\todo[inline]{\fsa{Can also the root be composed of only low complexity subsequences and the target be composed of only medium-to-high complexity substrings?} [Luan]: It could, but we would need to redo all the tests with the updated prompts. Could we explore this in a new paper?} \todo[inline]{\fsa{OK. I will insert some comments on this direction. Anyway I think we should mention this suggestion of future work in the text to address R3.}}
For instance, given the sequence {[}1, 2, 3,
4, 5, 6, 7, 8, 9, 10{]} and a prediction of 25\%, the `root' (the context provided to the prediction model) would be {[}1,
2, 3, 4, 5, 6, 7, 8{]}, with the `target' {[}9, 10{]} expected to be predicted. An asymptotic distribution of test results $\varphi_1,\ldots,\varphi_n$ for growing $n$ where $|s|=n$ should provide some insight into the generalisation of the capabilities of the LLMs to scale their reported abilities, if any.

We employed three methods to measure the accuracy of the predicted target: 
\begin{enumerate}
\item \emph{Sort similarity}: This measures how many elements in the target sequence were predicted correctly, with their order being considered.
\item \emph{General similarity}: This measures the correctness of predicted elements, without considering their order. 
\item \emph{Levenshtein}: This measures the Levenshtein distance between the expected and predicted sequences after converting them to strings. 
\end{enumerate}
%%%

%% Bibliography

 \bibliographystyle{IEEEtran}
\addcontentsline{toc}{section}{\refname}\bibliography{bibliography}

\newpage

\section{Supplementary Information}

%% Moved from the old AIT and AI section
%\section{Theoretical background}

\subsection{Algorithmic Information Theory (AIT) and Intelligence}\label{sectionAITIntel}

\subsubsection{Compression and machine learning}

%% Moved from old section regarding the foundational aspects.
Understanding compression as a necessary and fundamental characteristic of general intelligence refers to the ability to come up with a model capable of summarizing and eliminating redundancies by enabling one to explain more with less~\cite{zenil2020compression} or to gain ``the ability of explanatory compression'' ~\cite{orallo1998} in order to achieve (or approximate) the necessary and sufficient causal conditions for describing, predicting, explaining, or simulating a certain phenomenon. 

%% moved from old section foundational aspects
In machine learning models, such as large language models (LLMs), training involves learning to predict the next token in a sequence.
This is essentially an exercise in compression---understanding the structure of language or other data and compressing it into a representation that allows accurate predictions. 
LLMs can be thought of as word (token) time series predictors based on short- and long-range correlations that compress data from their very large training sets based on text repositories mostly available online, and captured in a much smaller object such as a giant matrix, whose numerical entries can partially and lossy reconstruct the training dataset. 
%%%
%LLMs are a powerful modelling approach yielding fascinating objects known for their ability to compress data such as text (and other types in multimodal systems) that when decompressed are capable of describing the original uncompressed information. 
Such a compressor can simulate/predict the uncompressed information stored in a multidimensional tensor probability distribution in a manner comparable to the uncompressed data captured in the smallest possible model---the smaller the better, and hence the smaller the model is, the better compressor~\cite{zenil2020compression}. 
The compressor's success can be evaluated in terms of how much information is lost in transit between the original world description and the decompressed data from the LLM model.
%Whether they build a compressed version that can amount to a level of understanding or comprehension is what this work (and test) sets out to help assess and determine, based on the foundational aspects of theoretical computer science and information theory.

In order to predict the future state of an event, a model shorter than the explanandum that captures its main features (object, event) is necessary in order to avoid conflating underlying laws with spurious patterns~\cite{Calude2017,Abrahao2021dSBpaperarxivBEPE}, and the more recursively compressed the model, the more adequate and less overfitted. 
`Recursively' here refers to ``mechanistically'' in the more general sense, that is, to any process whose unfolding (dynamics or evolution over time) can be mathematically modelled and determined---or more formally, a process that is computable---, and not only engaged in pattern matching as in statistical compression, which is only one type, and a limited one, of data/model compression. 
For example, recursively (i.e., algorithmically) compressing an object, such as a list of observations or events, yields the ability to predict some part of the object given other parts, as a byproduct of being able to run the compression process in reverse (decompression), should the events be removed from randomness (i.e., they are not disconnected from each other in such a way that any pattern, causation or correlation found is merely spurious).

An effective decompression process not only reconstructs or reassembles the original explanandum but it should be able to produce a continuation of it based on the continuation of the optimal recursive compressed features in reverse, producing a simulation that acts as a prediction on which a future action can be modelled. 
This principle of `compression', understood broadly as the process of simplifying a generative model as much as possible, constitutes the heuristic underlying any machine learning method with the purpose of avoiding overfitting, e.g., by ensuring the minimisation of the generalisation error while compromising the training error the least as possible.

The process of planning beyond what has made available, i.e., to predict the yet unseen, comprises making a comparison of the possible outcomes of the processes that unfold from the most succinct and expressive model (or theory) that is previously known.
Then, by adjusting it during an iteration over a recursive process of further comparisons between new unfoldings from future possible optimal models yet to be discovered, one obtains an evolving ground `truth' in a continuous learning process. 
As we discuss in Section~\ref{sectionAITIntel}, when such an iterative update process occurs over the entire algorithmic space one attains the most optimal prediction in the Bayesian sense~\cite{levin,miraculous}.

\subsubsection{Fundamental properties}\label{sectionAITbasics}

In the context of algorithmic information theory (AIT), universal computation is considered a central aspect of general intelligence in arbitrary systems.
These systems are in turn considered to be capable of making formal-theoretic predictions (e.g. of solar and lunar eclipses) with high accuracy according to a mathematical theory whose logical and equational derivations/predictions can always be verified by an unambiguous decision procedure~\cite{Abrahao2021bEmergenceAIDPTRSA,Calude2002}.
This underlies all science as it presumes and assumes that world phenomena can be described in a mathematical form in which science can deal with, e.g. using equations, computer simulations, mathematical modelling, etc. 
The ongoing rate of success of the scientific practice in finding compressible models for explaining and predicting natural phenomena evinces fundamental algorithmic and non-random characteristics underlying the universe itself~\cite{computableuniverse}. 
If reality were truly random, science as a predictive, model-building enterprise would be impossible.
A hallmark of a powerful compressed model is its capacity to unify previously disparate or anomalous observations under a single, coherent explanatory framework. 
The example of General Relativity, which provides a more compact description than Newtonian mechanics plus all its necessary ad-hoc corrections, demonstrates this property. 
A superior compression subsumes old models and explains their exceptions.

AIT is an overarching generalisation of classical (Shannon) information theory (see also Section~\ref{sectionRandomPred}) and the accepted mathematical definition that tells apart randomness from non-randomness able to objectively describe and quantify what a compact model is and what abduction (i.e., finding a model or theory that explains a given set of phenomena) and prediction corresponds in a formal mathematical setting.
%%%

Being one of the complexity indexes in AIT, \emph{algorithmic complexity} (also referred to as program-size, Kolmogorov or Solomonoff-Kolmogorov-Chaitin complexity) is a measure of the complexity of an object invariant (up to an object-independent additive constant) to the underlying formal theory, computation model and programming language.  
The value denoted by \( K(\sigma) \) of a finite string \( \sigma \) is the length of the shortest binary program (on an arbitrarily fixed universal Turing machine) that outputs \( \sigma \). 
The more compressible (or less random) a string \( \sigma \) is, the larger the positive value of \( |\sigma| - K(\sigma) \), where \( \left| \sigma \right| \) is the length of \( \sigma \). 
More complex (or random) objects require longer representative instantiations of their underlying generative model, while simpler, more regular objects can be generated by shorter programs~\cite{kolmogorov,chaitin,zenilbook1,zenilbook2}.
If a sequence $ \sigma $ can be represented by a shorter program $p$, the shorter program captures the regularities in $ \sigma $. 
In this sense, the program can be used to generate or predict future segments of the sequence, based on the learnt regularities, thus directly tying compressibility to the ability to predict future patterns.
The algorithmic complexity (along with other complexity indexes in AIT) of an object is proved to be universal (i.e., the value exists for every computing system) and invariant (i.e., the value remains the same according to any computing system, except for an object-independent constant) to the arbitrarily chosen computation model, programming language, probability measure, and formal theory~\cite{Cover2005,kolmobook,Chaitin2004}.
As we discuss in Sections~\ref{sectionCompCompre} and~\ref{sectionSuperARCframework}, both of these properties are crucial for measuring a system's intelligence in a manner that is generalisable beyond biases in human-centric metrics and real-world datasets made available for training.

Algorithmic complexity goes beyond strings, beyond binary and beyond computer programs. 
One uses this language or framework as a technicality given the fundamental nature of universal computation, including strings, binary languages, computer programs, etc. 
For example, as proven by Shannon any, discrete data can be transformed into binary without loss of information~\cite{Cover2005}.
Under the Church-Turing thesis, any effective description and decidable rule can be enacted by a computer program.
These computer programs are also not restricted to deal with strings only, just as computers deal with images, vectors, tensors, sounds, video or anything else that can be encoded.
Not only playing a crucial role in data compression, algorithmic complexity is therefore a concept of fundamental nature in the scientific method~\cite{bibid}. 
As we discuss in Section~\ref{sectionCompCompre}, science itself can fundamentally be seen as compressing natural phenomena, as the process of producing ever more compact representations of the physical world into rules, equations, and scientific models that provide ever greater explanatory and predicting power.

For illustration purposes and without loss of generality, let us consider a sequence of integers. 
The ability to compress such a sequence effectively is often taken as an indicator of understanding a model that is capable of generating the sequence, and one does not need to take the minimum requirement to the limit to find short plausible explanations.
The decimal expansion of an irrational number like $ \pi $ may appear complex and non-repeating, but the entire infinite sequence can be generated by a very simple mathematical formula (and algorithm). 
This demonstrates that a complex appearance does not necessarily entail a complex underlying generating process. 
In this sense, sufficient knowledge for explaining (or `comprehending') $ \pi $ is not a matter of memorizing its digits, but in coming up with---in other words, performing abductive reasoning---the mathematical formula that generates it~\cite{bibid}.
These explanations are computational in nature so that they can be built (and employed) by an universal constructor independent of the arbitrary choice of the computation model that one may deem fundamental. 
Such a universal constructor can be equivalent to a Turing machine, although not necessarily identical or isomorphic to a Turing machine, not even to the mechanistic nature upon which the machine is implemented or embedded into. 

Within the context of AIT, universal induction (based on Solomonoff's Theory of Inductive Inference) proves that prediction and compression are tightly linked in order to obtain optimal abductive reasoning, and therefore inferring the best model or theory.
Solomonoff~\cite{solo} laid the foundation for \emph{algorithmic probability} (another complexity index in AIT), which is a universally optimal probability measure in which an object is generated by a random program fed into a universal constructor (see Sup. Inf.).
The algorithmic coding theorem (ACT)~\cite{Calude2002,Downey2010} in Equation~\eqref{eqACTcomplete} displays one of the central results not only in AIT, but also one that has pervasive implications for any mathematical endeavour in science, particularly in artificial intelligence~\cite{Minsky2014ACT}.
Because algorithmic probability and algorithmic complexity are inversely proportional, ACT states that the more frequent an object is generated, the lower its algorithmic complexity, and vice-versa.

\begin{equation}\label{eqACTcomplete}
K(s)
= 
- \log P(s) 
=
 	- \log\left( \mathbf{m}\left( s \right) \right) \pm \mathbf{O}( 1 )
=
- \log\left( \sum\limits_{ p \in \left\{ w \colon U(w) = s \right\} } 2^{-|p|} \right)
\pm \mathbf{O}( 1 )
\text{ ,}
\end{equation}
where:
\begin{itemize}
    \item \( P(s) \) is the \emph{algorithmic probability} of string \( s \);%, as defined by the universal distribution;
    \item \( K(s) \) is the (prefix) \emph{algorithmic complexity} of string \( s \);
    \item $ \mathbf{m}\left( s \right) $ is a \emph{maximal} semicomputable semimeasure on the object $ s $;
    \item $ \sum\limits_{ p \in \left\{ w \colon U(w) = s \right\} } 2^{-|p|} $ is the universal (a priori) probability of the event $ s $.
\end{itemize}

Notice that a semicomputable semimeasure $ \mathbf{m}\left( \cdot \right) $ is said to be \emph{maximal} if for any other semicomputable semimeasure $ \mu\left( \cdot \right) $---including any computable probability measure one may arbitrarily choose---, where $ \sum\limits_{ x \in \left\{ 0 , 1 \right\}^* } \mu\left( x \right) \leq 1 $, there is a constant $ C > 0 $ (which does not depend on $ x $) such that, for every encoded object $ x $,
 %\begin{equation*}
 $ \mathbf{m}\left( x \right) \geq C \, \mu\left( x \right) $.
 %\end{equation*} 
This means that any arbitrarily chosen (computable or semicomputable) probability (semi)measure can only assign a higher probability to an event than the algorithmic probability of the event could up to an independent multiplicative constant.
Thus, across the landscape of all generative processes of each object that can be generated, the \emph{universal distribution} defined by the (algorithmic) probability from Equation~\eqref{eqACTcomplete} eventually dominates all other prior distributions one might devise~\cite{miraculous}, thereby enabling one to infer the model optimally in the asymptotic limit.

The universal probability of an event can be understood as the probability of randomly generating (by an i.i.d. stochastic process) a prefix-free (or self-delimiting) program that generates the event.
In other words, the probability that a randomnly generated explanation (like a computer program) can generate an object~\cite{miraculous}.
In conjunction with the invariance theorem of algorithmic complexity, ACT is related to the universally optimal encoding of objects, generalising information content measures for measurable spaces beyond what classical information theory is able to achieve~\cite{Downey2010,Burgin2009,Abrahao2021bEmergenceAIDPTRSA}, e.g., setting the theoretical underpinnings of any method based on algorithmic probability such as the coding theorem method (CTM)~\cite{ctm3,ctm1,ctm2,zenilbook2}, universal (Solomonoff) induction~\cite{Solomonoff1986,hutter2005universal,solo}, Levin's universal search~\cite{levin}, and minimum description length~\cite{miraculous,kolmobook,mdl}.
(See also Section~\ref{sectionBDMforASI}).
Universal predictors, such as those based on Levin's universal search (Sup. Inf.) or universal induction, use the ACT to model the most likely future events based on past data, capturing the link between compression and prediction proved to hold in the theory of algorithmic randomness, as we discuss in Section~\ref{sectionRandomPred}.

%
%In practical terms, compression algorithms like {ZIP} or {LZW} attempt to reduce the size of the data by identifying recurring statistical patterns. 
%If an AI system like ChatGPT can generate a concise and generalisable program to reproduce a sequence, it shows that the model has `compressed' the information in that sequence by finding its underlying patterns. 
%Algorithmic compression is more powerful because it can continue generating data while statistical pattern matching does not. 
%Pattern matching can only be descriptive, but computation-based regression, symbolic processing, and program synthesis can be prescriptive.

\subsubsection{Randomness, prediction, and compression}\label{sectionRandomPred}

%If a sequence $x$ can be represented by a shorter program $p$, the shorter program captures the regularities in $x$. In this sense, the program can be used to generate or predict future segments of the sequence, based on the learnt regularities. Thus, the ability to compress is directly tied to the ability to predict future patterns.

In comparison to statistical randomness (defined as a lack of patterns which cannot be identified by a particular statistical test), a key aspect of algorithmic complexity is the deeper relationship with \emph{algorithmic randomness}, whose lack of patterns is not only attested by an arbitrarily chosen statistical test but also by any conceivable formal-theoretic mathematical test effected by a computational decision procedure, thereby yielding incompressibility (and vice-versa).
For example, in the case of a unidimensional machine, a sequence is considered algorithmically random if its shortest generative program has essentially the same length of the sequence itself, that is, no shorter program exists that can generate the sequence.
In the case of a higher-level programming language, such a sequence can at best be described as a program of the type `print($x$)'. 
Formally, an object $x$ is (algorithmically) random if $ K(x) \geq |x|  - \mathbf{O}(1) $, where $|x|$ is the size of the object. 
Notice that $x$ is incompressible because no smaller program can produce it (except for a string-independent constant that may only depend on the arbitrarily chosen machine or programming language), which contrasts with highly structured or predictable data, where $K(x)\ll|x|$. 
A random string cannot be significantly compressed~\cite{chaitin}, implying that intelligence (as seen in systems that can compress data) involves recognising non-random patterns in data.

Statistical randomness (such as when a random event is measured by entropy-based statistical methods) is quantifiable by degrees of uncertainty based on frequency distributions, which is indeed effective and optimal when compressibility arises from repetition or statistical redundancy. 
This is because entropy is known to achieve optimal compression for pure stochastic processes that are ergodic and stationary~\cite{Cover2005}. 
%In this case, the minimum expected codeword length (or, equivalently, the expected algorithmic complexity) per symbol converges to the entropy rate, and vice-versa [1]. 
Under these same conditions, statistical compression methods, such as the algorithms in the LZ family, have been proven to also achieve optimal compression.
However, in case those conditions are not met, such as when the random source is not guaranteed to be stationary or the process is mixed (partially stochastic and partially mechanistic like complex systems found in nature), entropy (or any other statistical method) is proved to diverge from the optimal value given by algorithmic complexity---value which is also proved to be invariant under the arbitrary choice of programming language, computation model, probability distribution, and formal theory. 
As we further discuss in Section~\ref{sectionCompCompre}, this invariance is one of the reasons compression is a task that subsumes other intelligent systems' capabilities; and why other statistics-based approaches such as LLMs are limited in comparison to neurosymbolic approaches that include the algorithmic view and subsumes the statistical one.

%\todo{(R1)(R2) On compression and performance for comprehension.}

Thus, in the general case, statistical compression methods cannot achieve optimal compression even `in principle'. 
Entropy cannot detect algorithmic or generative structure that is not statistically apparent.
When a statistical compression algorithm such as {ZIP} or {LZW} compresses $x$ into other computer files much smaller than $ \left| x \right| $, it is a sufficient proof of \emph{non}-randomness. 
However, if it does not compress $x$ (or if it can only compress $ x $ into another file of size of the same order of $ \left| x \right| $), it is not a proof of randomness because there may be a generative program that the statistical compression is unable to produce/find. 
In other words, algorithmic randomness always implies statistical randomness, but the opposite does not always hold.

In practical terms, compression algorithms like {ZIP} or {LZW} attempt to reduce the size of the data by identifying recurring statistical patterns. 
When an AI system like ChatGPT can generate a concise and generalisable program to reproduce a sequence, it shows that the model has `compressed' the information by finding underlying patterns. 
Nevertheless, algorithmic compression is more powerful because it can continue searching for algorithmic generative processes while statistical pattern matching cannot. 
Pattern matching can only be descriptive of an entire object, but computation-based regression, symbolic processing, and program synthesis can be fundamentally generative; in the sense that other mechanisms or causal processes among (underlying or in common to) the parts of the object are prone to be swept over in an algorithmic compression.

%Schnorr and Levin independently 
In addition to formalising randomness beyond statistical patterns, the theory of algorithmic randomness established a profound connection between prediction and compression~\cite{Downey2010,Chaitin2004,Calude2002,kolmobook}. %,schnorr2,levin3,levin4}. 
%They proved that a sequence is algorithmically random if and only if no computable betting strategy (martingale) can succeed on it. 
It is equivalent to say that a sequence is algorithmically random (i.e., incompressible) if, and only if, no computable betting strategy (martingale) can succeed on it, establishing the equivalence between the inability to compress a sequence and the impossibility of predicting its future bits using any computable betting strategy a formal theory can devise.
This result demonstrated that the ability to compress a sequence is equivalent to being able to predict its future bits using any effective method (mathematical proof of this direct equivalence is provided in the Sup. Inf. Section~\ref{sectionPredictionandCompression}).
%A random string cannot be significantly compressed~\cite{chaitin}, implying that intelligence (as seen in systems that can compress data) involves recognising non-random patterns in data.
%Therefore, it is equivalent to say that a sequence is \emph{algorithmically random} (i.e., incompressible) iff no computable martingale succeeds on it, establishing the equivalence between the inability to compress a sequence and the impossibility of predicting its future bits using any computable betting strategy.
It also highlights the deep interplay between randomness, prediction, and compression, % as established by Schnorr and Levin using martingales~\cite{levin1,schnorr1,levin3,schnorr2}.
setting the underpinnings of our framework introduced in Section~\ref{sectionSuperARCframework}.
%\todo{(R1)(R2) On compression as a basis for comprehension.}
These results in AIT demonstrate that compression is not only a particular task that a (-n artificial or physical) system might be able to perform.
In fact, as we further discuss in Section~\ref{sectionCompCompre}, it is also a `task' that subsumes other comprehension tasks while taking into account the entire algorithmic space; but in an agnostic manner to the arbitrarily chosen computational capabilities, formal-theoretic features, and a priori knowledge.

\subsubsection{Equivalence between compression and prediction via Martingales}\label{sectionPredictionandCompression}

An infinite sequence (or equivalently, a real number) is denoted by \( x = x_1x_2x_3\ldots \), where each \( x_i \in \{0,1\} \).
Let $ x \upharpoonright_{ n } $ the sequence of the first $ n $ bits of the binary representation of $ x $.

A \emph{(super)martingale} function \( d: \{0,1\}^* \rightarrow \mathbb{R}^+ \) represents a betting strategy that satisfies
the fairness conditions:
\begin{itemize}
	\item[] 	\begin{equation}\label{eqMartingale}
	d(\sigma) = \frac{d(\sigma0) + d(\sigma1)}{2} 
	\text{, in the case of a martingale;}
	\end{equation}
	
	\item[] \begin{equation}
	d(\sigma) \geq \frac{d(\sigma0) + d(\sigma1)}{2} \text{, in the case of a \emph{super}martingale.}
	\end{equation}
%		\text{ ,}
\end{itemize}
%where \( \{0,1\}^* \) is the set of all finite binary strings. 
This conveys the idea that the expected capital after the next bet is either equal (for martingales) or is lost (for supermartingales) with respect to the previous capital.

A (super)martingale \( d \) \emph{succeeds} on a sequence \( x \) if:
\[
\limsup_{n \rightarrow \infty} d\left( x \upharpoonright_{ n } \right) = \infty
\]
This implies that the betting strategy can make an unbounded amount of money on \( x \) at the asymptotic limit as the length of the initial segment of $ x $ increases.

A martingale \( d \) is \emph{(left) semicomputable} if there is an algorithm that computably enumerates the left cuts of \( d( \sigma) \) for any given string \( \sigma \).
Thus, if a semicomputable $ d $ succeeds on a sequence $ x $, this (super)martingale can be interpreted as revealing the existence of 
an algorithm that can computably enumerate a betting strategy that always increases its capital gains at the asymptotic limit as the length of the initial segment of $ x $ increases.
This holds even if eventually one loses expected capital in the next bit (as the supermartingale condition allows).
The existence of such an enumerating algorithm guarantees that there is at least one asymptotically effective way of predicting the forthcoming bits in the infinite sequence $ x $ so as to render the betting strategy successful as this process goes on.

Now, remember that an algorithmically random infinite sequence (or real number) $ x $ is incompressible up to a fixed constant so that $ K\left( x \upharpoonright_n \right) \geq n - \mathbf{O}(1)$, and the constant does not depend on $ n $.
Therefore, if $ x $ is \emph{not} algorithmic random, then for any $ k $ and for any $ n' \geq 1 $, there is $ n \geq n' $ such that $ K\left( x \upharpoonright_n \right) < n - k $.
In other words, $ x $ is compressible (by more than a fixed value) infinitely often.

The notion of predictability conveyed by martingales should reflect the fact that in the case of an algorithmically random sequence, there would not exist an enumerating algorithm that guarantees that there is at least one asymptotically effective way of predicting the forthcoming bits in the infinite sequence $ x $ so as to render the betting strategy successful as this process goes on.
In summary, one should not expect to be able to devise a computably enumerable betting strategy that is successful on a perfectly random sequence.
Indeed, the equivalence between (super)martingales and algorithmic randomness holds:
\begin{itemize}

\item If a sequence \( x \) is not algorithmically random (i.e., it is compressible infinitely often), then there exists a semicomputable martingale that succeeds on \( x \).

\item Conversely, if there exists a semicomputable martingale that succeeds on \( x \), then \( x \) is not algorithmically random (i.e., it is compressible infinitely often).

\end{itemize}
Another equivalence between algorithmic randomness and the notion of predictability can be achieved from (stochastic or probabilistic) martingale processes, which are defined upon real-valued random variables.
In this case, one can demonstrate that an infinite sequence is algorithmic random iff no \emph{computable} martingale process succeeds on it~\cite{Downey2010}.

Usually, (super)martingales and randomness are demonstrated to be equivalent via proof- and measure-theoretic statistical (Martin-L\"of) tests.
A sequence is incompressible iff it does \emph{not} pass on any ($ \Sigma_1^0 $) theoretic statistical test~\cite{Downey2010}, thereby called (prefix) algorithmic random ($ 1 $-random or $ \mathbf{O}(1) $-$ K $-random).
It is important to remark that the triple equivalence between predictability (via martingales), statistical tests (via proof and measure theory), and compressibility (via algorithmic complexity) establishes one of the foundational results in the theory of algorithmic randomness and algorithmic information~\cite{Downey2010,Calude2002}.

In order to highlight the connection between predictability and compressibility, in the following, we introduce a novel and alternative proof for the \emph{direct} equivalence between compression and (successful computably enumerable) martingales.

Regarding algorithmic randomness deficiency~\cite{kolmobook}, one can define a weaker notion of supermartingales to account for language and computation model dependencies. 
We say a function $ d $ is a $ C $-\emph{supermartingale} iff for any sequence $ \sigma $, there is a constant $ C \geq 0 $ (that does not depend on $ \sigma $) such that
\begin{equation}\label{eqDefSupermartingaledeficiency}
	\frac{ 1 }{ 2^C }
	\leq
	\frac{ d( \sigma 0 ) + d( \sigma 1 ) }{ 2 \, d( \sigma ) }
	\leq
	\frac{ 1 }{ 2^{ - C } }
	\text{ .}
\end{equation} 
On the one hand, the expected capital from the bet in the next bit is never smaller than a constant ratio of the previous bet.
On the other hand, one may gain some expected capital in the next bet but only up to a multiplicative constant.
Instead of a constant $ C $, one can also define $ \mathfrak{d}( \sigma ) $-supermartingale, where $ \mathfrak{d} \colon \left\{ 0 , 1 \right\}^* \to \mathbb{N} $.
For the present purposes, we focus on the constant that does not depend on the object.

From the basic properties in algorithmic information theory, it is straightforward to prove that the function
\begin{equation}\label{eqFirstCsupermartingale}
	d_{(1,k)}( \sigma ) = \frac{ 2^{ \left| \sigma \right| } }{ 2^{ k + K( \sigma ) } }
\end{equation}
is a $ \mathbf{O}( 1 ) $-supermartingale.
Clearly, if $ x $ is not an algorithmic random infinite sequence, then $ d_{(1,k)}( x \upharpoonright_n ) \geq 1 $ for every $ k $ and $ n $ in which $ K( x \upharpoonright_n ) < n - k $.
From the definitions and the property that the summation of any two $ C $-supermartingales is also a $ C $-supermartingale, one can demonstrate by induction that if $ d_1 , d_2 , \dots , d_i , \dots $ is an infinite family of $ C $-supermartingales and $ \sum\limits_{ i = 1 }^{ \infty } d_i( a ) < \infty $, where $ a $ is any string for the initial capital (usually, the empty string $ \lambda $, $ 0 $, or $ 1 $), then  $ \sum\limits_{ i = 1 }^{ \infty } d_i( \cdot ) $ is a $ C $-supermartingale (see also~\cite{Downey2010}).
From Equation~\eqref{eqFirstCsupermartingale}, we have it that $ \sum\limits_{ i = 1 }^{ \infty } d_{(1,i)}( a ) = \mathbf{O}(1) $.
In addition, for any $ \sigma $, one has it that $ \sum\limits_{ k = \left| \sigma \right| }^{ \infty } d_{(1,k)}( \sigma ) \leq \mathbf{O}\left( 2^{ \left| \sigma \right| } \right) $, and as a consequence $ \sum\limits_{ i = 1 }^{ \infty } d_{(1,k)}( \sigma ) < \infty $ holds.
We also have that $ \sum\limits_{ i = 1 }^{ \infty } d_{(1,i)}( \sigma ) $ is left semicomputable because there is a program that can always approximate the value of $ \sum\limits_{ i = 1 }^{ \infty } d_{(1,i)}( \sigma ) $ from below for any $ \sigma $.
Therefore, if $ x $ is not an algorithmic random infinite sequence, it follows that there is a left semicomputable $ \mathbf{O}(1) $-supermartingale $ d_1( \sigma ) = \sum\limits_{ i = 1 }^{ \infty } d_{(1,i)}( \sigma ) $ such that $ \limsup_{n \rightarrow \infty} d_1\left( x \upharpoonright_{ n } \right) = \infty $.
The converse implication can be proved analogously %\footnote{ To this end, take $ f_k\left( \sigma \right) = \frac{ d( \sigma ) }{ 2^{ C \left( \left| \sigma \right| - k \right) } d( \sigma \upharpoonright_{ k } )}  $ instead.}
to the proof in Theorem~\ref{thmIncompUnpredict}, because every martingale is a $ \mathbf{O}(1) $-supermartingale. 

Nevertheless, as we show in Theorem~\ref{thmIncompUnpredict}, one can also obtain a demonstration of the implications in both directions between compression and the traditional (successful computably enumerable) \emph{martingales} without resorting to proof- and measure-theoretic statistical tests.

%Proof:\\

\begin{theorem}[incompressibility and unpredictability]\label{thmIncompUnpredict}
	Let $ x = x_1 x_2 \ldots x_n \ldots  $ be an infinite sequence (or equivalently, a real number).
	Then,
	$ x $ is algorithmic random \emph{iff} there is no (left) semicomputable martingale that succeeds on $ x $.
\end{theorem}

\begin{proof}[Proof (Compression implies Prediction):]

%Let's assume that the sequence \( x \) is not algorithmically random, that is, compressible infinitely often.
% that is, for every constant $ k $ and for every $ n_0 $ there is $ n \geq n_0 $ such that $ K( x \upharpoonright_n ) \leq \left| n \right| - k $.
%We construct a semicomputable martingale \( d \) that succeeds on \( x \).\\
%As usual, let $ x \upharpoonright_{ n } $ the sequence of the first $ n $ bits of the binary representation of $ x $.
%For the sake of simplifying the forthcoming formulas, let $ \sigma \updownarrow_{ -n } $ denote the bit string $ \sigma $ with the last (rightmost) $ n $ bits flipped.
%Let $ {\{0,1\}^*} ^{ n } $ be the set of all $ m $-bit strings with $ m \leq n $.
For any arbitrary sequences $ w $ and $ z $, let $ w \preceq z $ denote $ w $ being a prefix of the sequence $ z $.
%For arbitrary $ x $ and a fixed $ n $, let $ \lceil x \rceil^{ n } = \min\left\{ x , n \right\} $ denote the value of $ x $ up to $ n $.
Without loss of generality, let $  C > 0 $ be a constant such that 
\begin{equation}\label{eqConstantforinitialcapital}
	K( a ) <  C
	\text{ ,}
\end{equation}
for $ a \in \left\{ \lambda , 0 , 1 \right\} $.
Let
\begin{equation}\label{eqSetofcompressibleextensions}
	W_k\left( \sigma \right) = 
	\left\{  w \in {\{0,1\}^*} \bm{\colon} \begin{array}{l}
%		\left| w \right| \leq 2 \, k , \\
		w \succeq \sigma , \\
%		\sigma \succeq z , \\
		\left( K\left( w \right) < C \right) \lor \left( K\left( w \right) < \left| w \right| - k \right)
	\end{array}  \right\}
\end{equation}
be the set of bit strings that are compressible by at least $ k $ bits, strings which have $ \sigma $ as a prefix.
%Notice that $ W_k\left( \cdot \right) $ is basically the set of all bit strings $ \sigma' $ for which $  K\left( \sigma' \right) < \left| \sigma' \right| - k $ holds, and hence the set only depends on $ k $ in the asymptotic limit.
For arbitrary $ k \in \mathbb{N} $, let 
$ d_{(2,k)} \colon \{0,1\}^* \rightarrow \mathbb{R}^+ $
be a function such that
\begin{equation}
	d_{(2,k)}\left( \sigma \right) = 
	\frac{
	2^{ \left| \sigma \right| } 
	}{
	2^k
	}
	\left( \sum\limits_{ w \in W_k\left( \sigma \right) } \frac{ 1 }{ 2^{ K( w ) } } \right)
	\text{ .}
\end{equation}
First, notice that $ W_k\left( a \right) \neq \emptyset $ for any $ k \geq 1 $ because of our choice of the constant $ C $.
Secondly, from the basic properties of a prefix-free (or self-delimiting) programming language~\cite{kolmobook,Calude2002,Downey2010}, we have that 
\begin{equation}
	0 \leq
%	\left( \frac{
%	2^{ \left| \sigma \right| } 
%	}{
%	2^k
%	} \right)
%	\frac{ 1 }{ 2^{ K( a ) } } 
%	\leq
	d_{(2,k)}\left( \sigma \right) \leq \frac{
	2^{ \left| \sigma \right| } 
	}{
	2^k
	}
\end{equation}
holds for any $ \sigma $ and $ k $.
As a consequence, we will have it that $ \sum\limits_{ k = 1 }^{ \infty } d_{(2,k)}\left( a \right) = \mathbf{O}\left( 1 \right) $ 
%and 
%$ \sum\limits_{ k = \left| \sigma \right| }^{ \infty } d_{(2,k)}\left( \sigma \right) = \mathbf{O}\left( 2^{ \left| \sigma \right| }  \right) $.
and $ \sum\limits_{ k = 1 }^{ \infty } d_{(2,k)}\left( \sigma \right) < \infty $.
From the definition of $ W_k( \cdot ) $ in Equation~\eqref{eqSetofcompressibleextensions}, we have that 
\begin{equation}\label{eqDisjoint1}
	W_k( \sigma 0 ) \cap W_k( \sigma 1 ) = \emptyset
\end{equation}
and
\begin{equation}\label{eqUnion1}
	W_k( \sigma 0 ) \cup W_k( \sigma 1 ) = W_k( \sigma )
\end{equation}
hold for any $ \sigma $, and therefore one can straightforwardly demonstrate that $ d_{(2,k)} $ is a martingale for each fixed $ k $.
%We will also have it that for any $ k $, $ w $ and $ z $ with $ w \succeq z $, 
%\begin{equation}
%	d_{(2,k)}\left( w \right) \geq d_{(2,k)}\left( z \right)
%	\text{ .}
%\end{equation}
We know that if $ d_1 , d_2 , \dots , d_i , \dots $ is an infinite family of arbitrary martingales and $ \sum\limits_{ i = 1 }^{ \infty } d_i( a ) < \infty $, where $ a $ is any string for the initial capital, then  $ \sum\limits_{ i = 1 }^{ \infty } d_i( \cdot ) $ is a martingale~\cite{Downey2010}.
Therefore, we will have that 
\begin{equation}\label{eqDefd2}
	d_2\left( \sigma \right) = \sum\limits_{ i = 1 }^{ \infty } d_{(2,i)}\left( \sigma \right)
\end{equation}
is a martingale.
Since the infinite set $ W_k( \sigma ) $ can be computably enumerated from below for any $ \sigma $, we will have that $ \sum\limits_{ i = 1 }^{ \infty } d_{(2,i)}\left( \sigma \right) $ is left semicomputable.
By construction, for any $ k $ and $ \sigma $ in which $ K\left( \sigma \right) < \left| \sigma \right| - k $ holds, one has it that 
\begin{equation}
	d_{(2,k)}\left( \sigma \right) \geq d_{(1,k)}\left( \sigma \right) \geq 1
	\text{ ,}
\end{equation}
where 
$
	d_{(1,k)}\left( \sigma \right)
$
was defined in the above Equation~\eqref{eqFirstCsupermartingale}.
Additionally, for any $ w $ and $ z $ with $ w \succeq z $ such that $  K\left( z \right) < \left| z \right| - k $ and $  K\left( w \right) < \left| w \right| - k - 1 $ hold, we will have it that $ d_{(2,k+1)}\left( w \right) \geq 1 $ and $ d_{(2,k)}\left( w \right) \geq 1 $.
One can extend this property recursively so that if  $ w_{ m } \succeq w_{ m - 1 } \succeq \dots \succeq w_{ 0 } $ such that $  K\left( w_{ i } \right) < \left| w_{ i } \right| - k - i $ holds for any $ i $ where $ 0 \leq i \leq m $ and $ m > 0 $, then $ d_{(2,k+i)}\left( w_{ m } \right) \geq 1 $ holds for each $ i \leq m $, thereby one obtains that $ d_2\left( w_{ m } \right) \geq m $.
Therefore, if $ x $ is not an algorithmic random infinite binary sequence, then $ \limsup_{n \rightarrow \infty} d_2\left( x \upharpoonright_{ n } \right) = \infty $.

\end{proof}

\begin{proof}[Proof (Prediction implies Compression):]

% Let $ \sigma \upharpoonright_{ n } $ be the sequence of the first $ n $ bits of the string $ \sigma $.
%Let $ \sigma \updownarrow_{ -n } $ denote the bit string $ \sigma $ with the last (rightmost) $ n $ bits flipped.
%Since $ d $ is left semicomputable by hypothesis, let $ f\left( \sigma, n \right) $ be a computable function such that for some $ n \in \mathbb{N} $ one has it that $ f\left( \sigma, n \right) = d\left( \sigma \right) $.
From the martingale condition in Equation~\eqref{eqMartingale}, where
\begin{equation}\label{eqAltMartingale}
	\frac{ d'( \sigma 0 ) + d'( \sigma 1 ) }{ d'( \sigma ) } = 2
\end{equation}
holds for any $ \sigma $ and an arbitrary martingale $ d' $, we will have that
\begin{equation}\label{eqConditionspecialfunctionfk}
	\frac{ d'( \sigma ) }{  d'( \sigma \upharpoonright_{ k } )} 
	=
	\prod\limits_{ i = 1 + k  }^{ \left| \sigma \right| } \frac{ d'( \sigma \upharpoonright_{ i } ) }{ 
%	2^{ \left( \left| \sigma \right| - k \right) } \, 
	d'( \sigma \upharpoonright_{ i - 1 } )} 
	\leq
	2^{ \left| \sigma \right| - k }
%	\text{ .}
\end{equation}
holds for any arbitrary natural number $ k \geq 1 $ with $ k < \left| \sigma \right| $.
Let $ \left< \cdot \right> $ be any computable encoding of a string in a prefix-free language such that for any $ w \in \left\{ 0 , 1 \right\}^* $, one has it that 
\begin{equation}\label{eqBasicencodingprefixfree}
	\left|  \left< w \right> \right| 
	\leq \left| w \right| + \mathbf{O}\left( \log\left( \left| w \right| \right) \right)
\end{equation}
and
\begin{equation}\label{eqBasicPrefixfree}
	\sum\limits_{ \sigma \in \left\{ 0 , 1 \right\}^* } \frac{ 1 }{ 2^{ \left| \left< \sigma \right> \right| } }
	\leq
	1
	\text{ .}
\end{equation}
Let
\begin{equation}\label{eqDefW'k}
W'_k\left( \sigma \right) =
\left\{  w \in {\{0,1\}^*} \bm{\colon} \begin{array}{l}
	w \succeq \sigma , \\
	\log\left( \frac{ d\left( w \right) }{ 2^k } \right) \geq \left| \left< \sigma \upharpoonright_{ k^2 } \right> \right|
\end{array}  \right\} 
\text{ .}
\end{equation}
be a set of the extensions of $ \sigma $ for which their values obtained from $ d $ are sufficiently large.
Notice that since $ d $ is (left) semicomputable by hypothesis, then the set $ W'_k\left( \sigma \right) $ is computably enumerable for any $ \sigma $ given $ k \in \mathbb{N} $.
Additionally, from Equation~\eqref{eqBasicencodingprefixfree}, the condition $ \limsup_{n \rightarrow \infty} d\left( x \upharpoonright_{ n } \right) = \infty $ implies that for every $ k , m_0 \in \mathbb{N} $ with $ m_0 \geq k $, there is at least one $ x \upharpoonright_{ m } \succeq x \upharpoonright_{ m_0 } $ such that
\begin{equation}\label{eqConsequenceoflimsup}
	d\left( x \upharpoonright_{ m } \right)
	\geq
	2^{ k^3 }
	\gg
	2^{  \left| \left< x \upharpoonright_{ m_0 } \upharpoonright_{ k^2 } \right> \right| + k }
%	\text{ ,}
\end{equation}
with $ m > m_0 $,
and thereby one obtains that $ x \upharpoonright_{ m } \in W'_k\left( \sigma \right) $.
Now, we define the function
\begin{equation}
f_k\left( \sigma \right) =
\frac{  \underset{\tiny w \in W'_k\left( \sigma \right) }{\arg\!\min} \; \left( 2^{ \left| w \right|  } \right) }{ 2^k }
\end{equation}
built upon the set $ W'_k $ in Equation~\eqref{eqDefW'k}.
From the computable enumerability of $ W'_k $, we will have that $ f_k\left( \cdot \right) $ is a right semicomputable function (i.e., semicomputable from above), and hence $ \frac{ 1 }{ f_k\left( \cdot \right) } $ is left semicomputable (i.e., semicomputable from below).
Clearly, in case $ \sigma \in W'_k\left( \sigma \right) $, one will have it that
\begin{equation}\label{eqUpperboundforfk}
	f_k\left( \sigma \right)
	=
	2^{ \left| \sigma \right| - k }
	\text{ .}
\end{equation}
Furthermore, from Equations~\eqref{eqConditionspecialfunctionfk} and~\eqref{eqDefW'k}, one also has that
\begin{equation}\label{eqSpecialconditionforfk}
	f_k\left( \sigma \right)
	\geq
	\frac{ d\left( w \right) }{ 2^k }
	\geq
	2^{ \left| \left< \sigma \upharpoonright_{ k^2 } \right> \right| }
\end{equation}
holds for some $ w \in W'_k $ and any fixed $ k $.
Therefore, from Equations~\eqref{eqBasicPrefixfree} and~\eqref{eqSpecialconditionforfk}, one will have it that
\begin{equation}\label{eqConditionsemimeasurefromfk}
	\sum\limits_{ \sigma \in \left\{ 0 , 1 \right\}^* } \frac{ 1 }{ f_k\left( \sigma \right) }
	\leq
	1
	\text{ .}
\end{equation}
Let
\begin{equation}\label{eqDefsemimeasure}
	\mu\left( \sigma \right)
	=
	\frac{ 1 }{ f_k\left( \sigma \right) }
\end{equation}
so that from Equation~\eqref{eqConditionsemimeasurefromfk} we directly obtain that $ \mu\left( \cdot \right) $ is a left semicomputable semimeasure.
Since $ \limsup_{n \rightarrow \infty} d\left( x \upharpoonright_{ n } \right) = \infty  $, then Equations~\eqref{eqConsequenceoflimsup},~\eqref{eqUpperboundforfk}, and~\eqref{eqDefsemimeasure} imply that for each fixed $ k $, there are infinitely many $ m \in \mathbb{N} $ such that
\begin{equation}\label{eqSpecialcaseformumandk}
	\frac{ 1 }{ \mu\left( x \upharpoonright_{ m } \right) }
	=
	2^{ m - k }
	\text{ .}
\end{equation}
From the ACT~\cite{Calude2002,Downey2010,kolmobook} in Equation~\eqref{eqACTcomplete}, we have that
\begin{equation}\label{eqACT}
	\begin{aligned}
		K\left( x \right) = 
	- \log\left( \mathbf{m}\left( x \right) \right) \pm \mathbf{O}( 1 )
	\text{ ,}
	\end{aligned}
\end{equation}
holds,
where
$ \mathbf{m}\left( \cdot \right) $ is a \emph{maximal} semicomputable semimeasure. 
Finally, it follows from Equations~\eqref{eqSpecialcaseformumandk} and~\eqref{eqACT} that there is a constant $ C' $ such that for each fixed $ k $, there are infinitely many $ m \in \mathbb{N} $ such that
\begin{equation}
	K\left( x \upharpoonright_{ m } \right)
	\leq
	\log\left( \frac{ 1 }{ C'  \mu( x \upharpoonright_{ m } )  }  \right)
	\pm \mathbf{O}(1)
	\leq
	m - k + \mathbf{O}(1)
	\text{ .}
\end{equation}

 \end{proof}

\subsubsection{Levin's Distribution and the Algorithmic Probability of Integer Sequences}\label{sectionLevinsearch}

As shown in Equation~\eqref{eqACTcomplete}, the algorithmic probability $ P(s) = 1 / 2^{ K\left( s \right) } $ of a string $ s $ is equivalently\footnote{ Except for a multiplicative independent constant.} given by~\cite{solo,levin}: 
\[P(x) = \sum_{ U(p) = x} 2^{-|p|},\] 
where $U(p) = x $ means that the (prefix) universal Turing machine $ U $, when given program $ p $, produces the string $ x $. $|p| $ is the length of the program $ p $, so $ 2^{-|p|} $ can be interpreted as the probability assigned to that program, with shorter programs being more probable.

Levin's distribution modifies the algorithmic probability by adding a penalty for the time taken by the program to compute the output, for example as
\[m(x) = \sum_{p : U(p) = x} 2^{-|p| - \log T(p)},\] 
where $T(p) $ is the time taken by program $ p $ to generate the string $ x $, where $ \log T(p) $ is the logarithmic penalty for the time complexity of program $ p $.
Notice that $ m $ is a lower bound for the universally optimal semicomputable semimeasure $ \mathbf{m} $ in Section~\ref{sectionPredictionandCompression} that appears in the ACT.

In the context of a time series $ x_1, x_2, \dots, x_t $, the goal is to predict the next value $ x_{t+1} $ based on the previous observations $ x_1, x_2, \dots, x_t $. Modifying it according to the conditional version of the ACT, the probability of the next element $ x_{t+1} $, given the previous values, becomes

\[P(x_{t+1} \mid \left( x_1, x_2, \dots, x_t \right) ) = \sum_{ U\left( \left< x_1, x_2, \dots, x_t , p \right> \right) =  x_{t+1} } 2^{-|p| - \log T(p)}\]

This represents the posterior probability of $x_{t+1}$, where shorter and faster programs (that generate it from the sequence $ x_1, x_2, \dots, x_t $) are favoured.

The compression of a time series $ x_1, x_2, \dots, x_t $ seeks the shortest program that generates the observed sequence. Using Levin's distribution, the compressed length $ K(x_1, x_2, \dots, x_t) $ is approximately 

\[C(x_1, x_2, \dots, x_t) \approx \min_{U(p) = \left( x_1, x_2, \dots, x_t \right) } \left( |p| + \log T(p) \right)\]

This expression seeks the minimum of the program length $ |p| $ plus the time penalty $ \log T(p) $, giving the most compressed form of the time series while also considering the computational time complexity.

\begin{figure}
\begin{center}
\makebox[\textwidth][c]{\includegraphics[width=1
\textwidth]{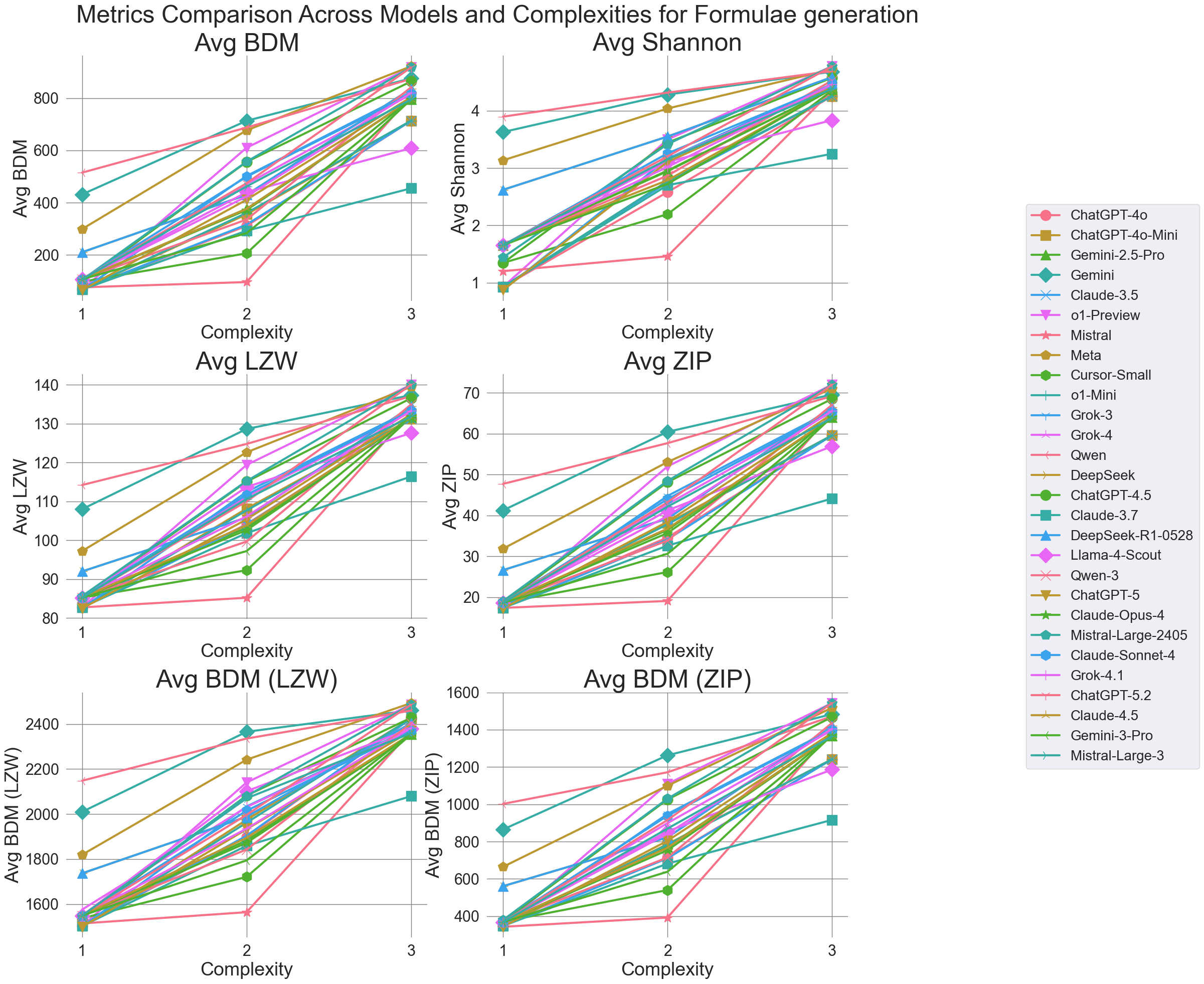}}%
\end{center}
\caption{Complexity measures in the free-form test. LLM answers follow the theoretical expectation. For increasingly complex sequences, we see a decreasing number of compressed answers (or any answers at all) when LLMs are asked to produce a generating mechanism (such as a formula).}
\label{MetricsForFormulaGeneration}
\end{figure}

\begin{center}
\begin{figure}[htp!]
\makebox[\textwidth][c]{\includegraphics[width=.75
\textwidth]{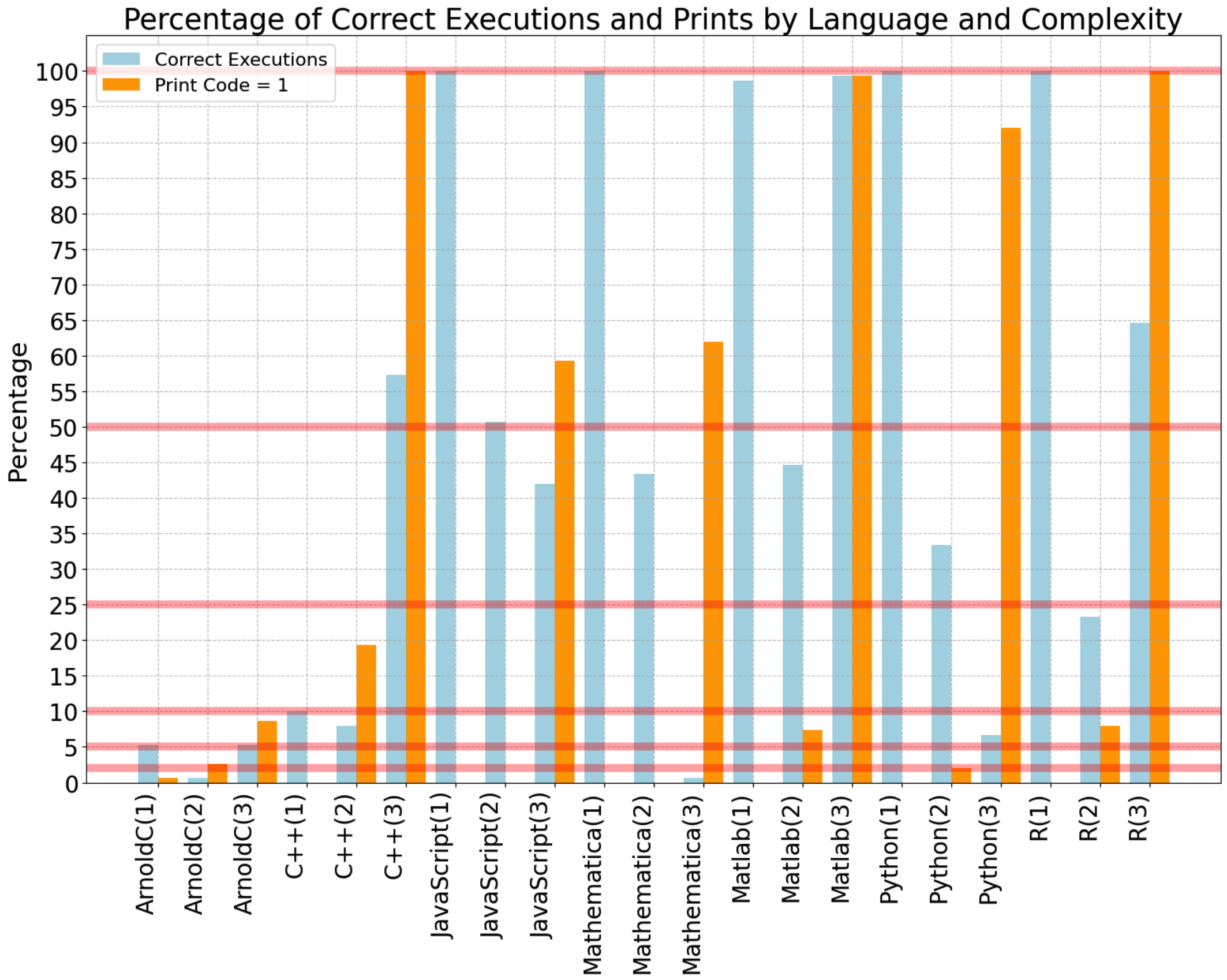}}
\makebox[\textwidth][c]
{\includegraphics[width=.76
\textwidth]{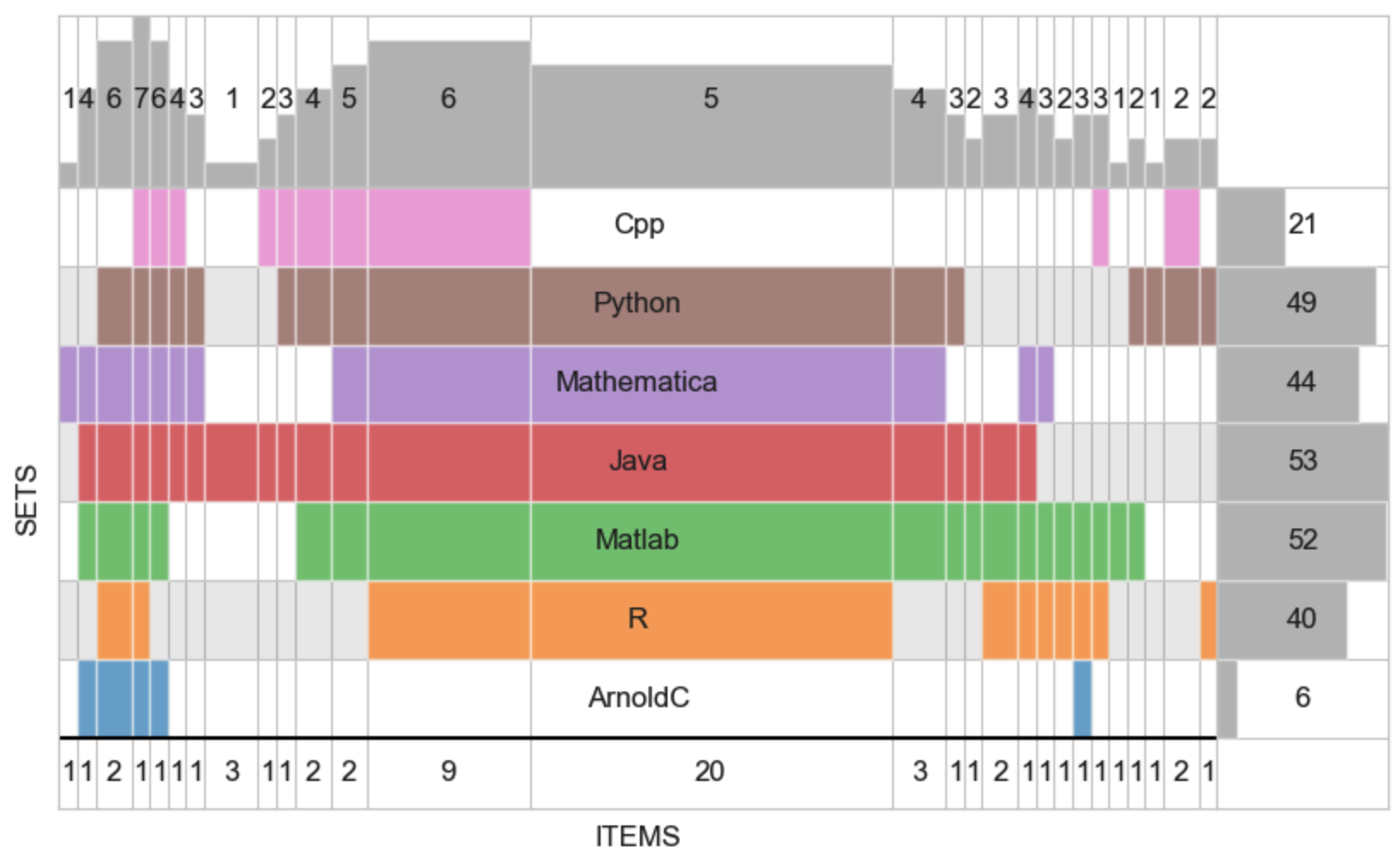}}
\caption{\emph{Top:} Distribution of correct and print cases by language and complexity produced by ChatGPT-4. The results show an inversely proportional number of correct answers to sequences' complexity increase, and a proportionally direct trend for simplistic print codes, both conforming with the expectation that higher complexity would retrieve fewer correct code evaluations and more trivial programs of type `print', with a few exceptions, most likely as a result of examples found in the LLM training set. \emph{Bottom:} Distribution of correct answers for ChatGPT-4. The upper section shows the number of scripts in different programming languages that reproduce the target sequences indicated below. The right section shows the total scripts by language successfully reproducing target sequences.This distribution highlights a subset of well-documented sequences accurately replicated by LLMs, with failures attributed to insufficient examples rather than language choice or understanding.}
\label{print&CorrectAnswers}
\end{figure}
\end{center}

\begin{center}
\begin{figure}[htp!]
\makebox[\textwidth][c]{\includegraphics[width=0.8
\textwidth]{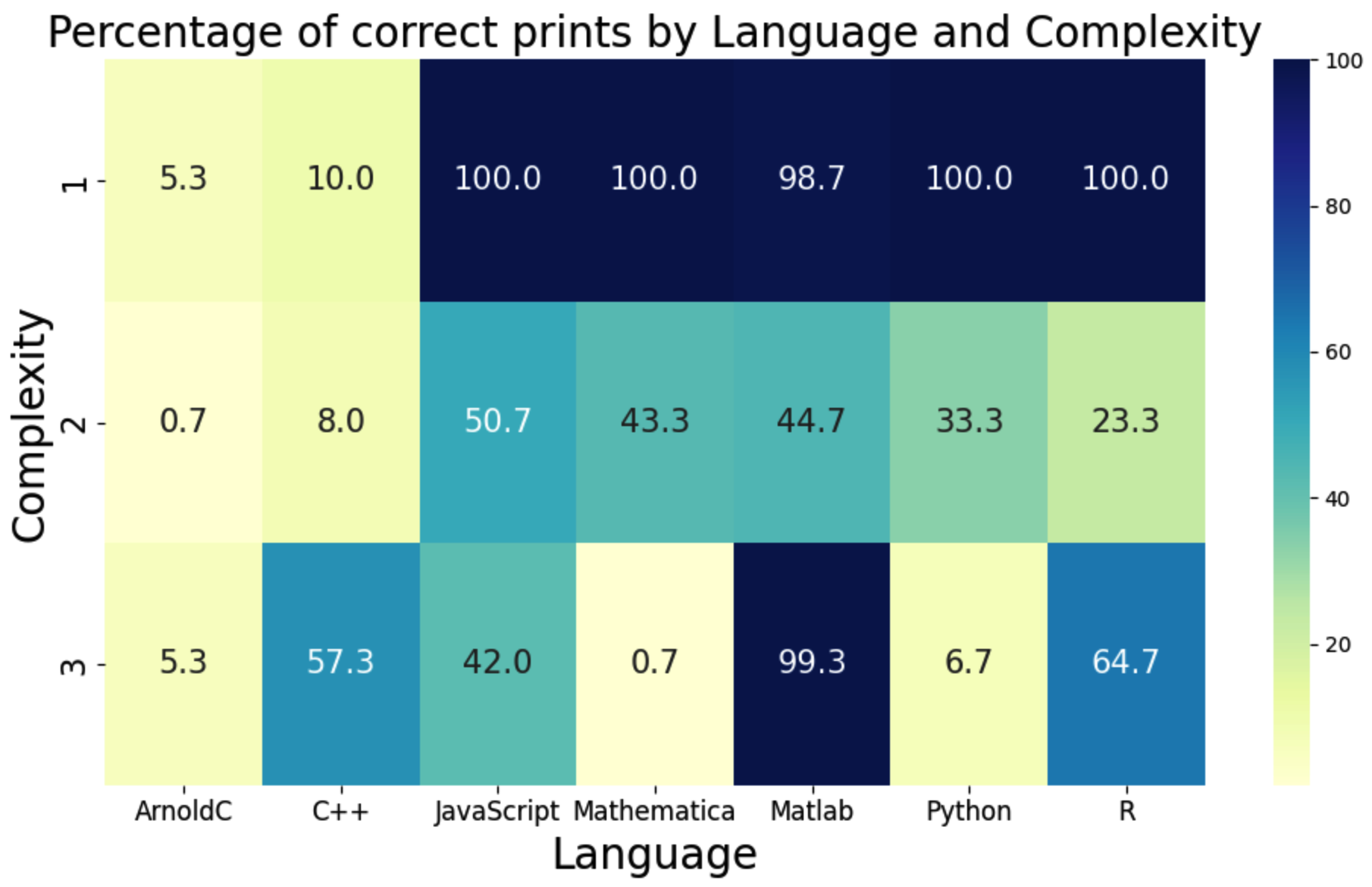}}\\

\makebox[\textwidth][c]{\includegraphics[width=0.83
\textwidth]{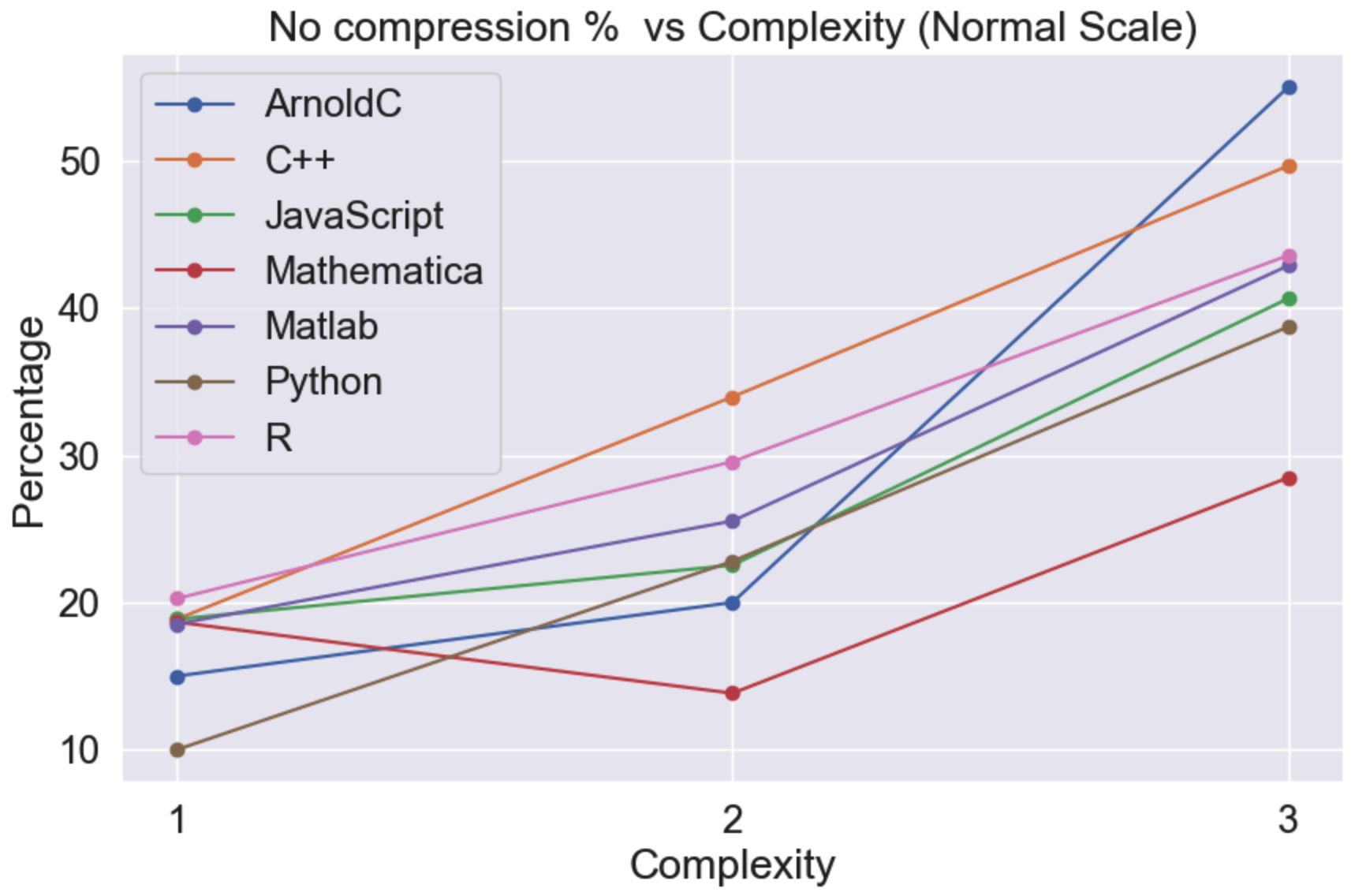}}
\caption{\emph{Top:} Print cases by language and complexity for ChatGPT 4. \emph{Bottom:} No compression percentage in original answers from ChatGPT 4.}
\label{NoCompressPercentage}
\end{figure}
\end{center}

\begin{figure}[htp!]
\begin{center}
\makebox[\textwidth][c]{\includegraphics[width=1
\textwidth]{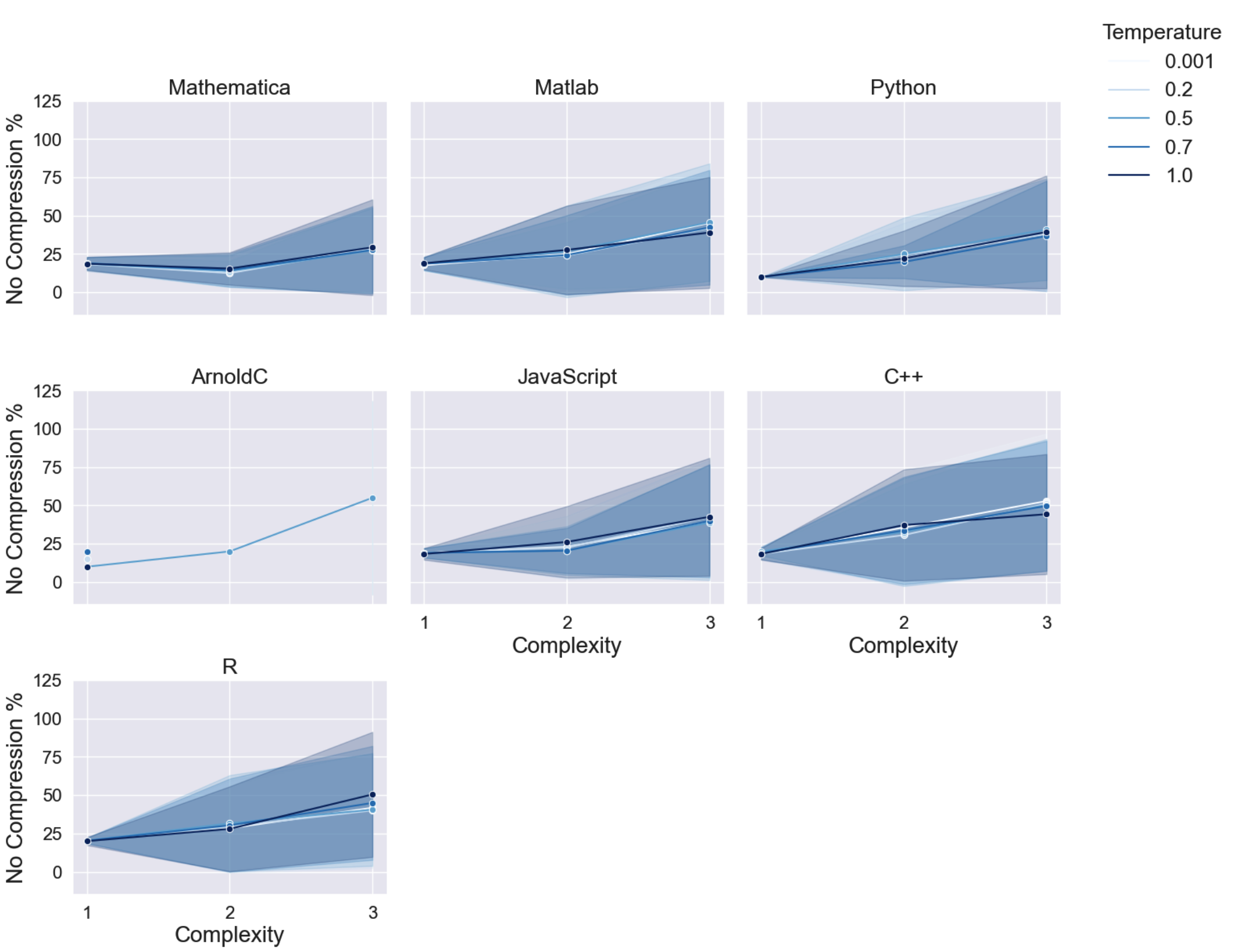}}%
\end{center}
\caption{Complexity vs no compression and variation of temperature parameter showing robustness of results independent of controlled noise, where 1 is the typical LLM balance between `precision' or repeatability and `creativity' as defined by each LLM version.}
\label{NoCompressionTemperatureConfidence}
\end{figure}

\subsection{An Algorithmic Information Dynamics (AID) of AI algorithms, external processes, and evaluator agents}\label{sectionAIDforSuperARC}

A distinguishing and fundamental characteristic of Algorithmic Information Dynamics (AID)~\cite{zenil2020algorithmic,zenilbook1,zenilbook2,Abrahao2021bEmergenceAIDPTRSA} with respect to traditional applications of universal induction and the algorithmic coding theorem is the shift from observational analysis to \emph{interventional} (or \emph{perturbation}) analysis.
AID extends AIT from static analysis to dynamic intervention.
To this end, AID employs (algorithmic) perturbations that are (computable) interventions, changes, modifications, or mutations, such as altering a bit, removing a graph vertex or edge, changing an input variable or a code snippet, or restructuring the entire system~\cite{Abrahao2021bEmergenceAIDPTRSA}.
Instead of merely observing a system, AID performs perturbations and measures the resulting shift in algorithmic complexity or randomness, creating a ``calculus'' for software or physical dynamics of irreducible information content~\cite{iscience}.
By measuring such a resulting shift, rather than only from statistical correlations, the framework infers a degree of \emph{causality} that is explainable by the observer equipped with a formal theory~\cite{Abrahao2021bEmergenceAIDPTRSA}.
The \emph{observer} in this scenario is the hypothetical \emph{agent} applying the algorithmic coding theorem and other tools from AIT to the experienced data in order to find the optimal model.

While methods or frameworks  based on universal (Solomonoff) induction like the minimum description length~\cite{mdl} or in the AIXI~\cite{hutter2005universal}---which mainly employ statistical compression approximation to algorithmic complexity and therefore are closer to entropy than to algorithmic information~\cite{zenilreview}---focus on achieving universally optimal prediction, they generally assume that the data is grounded in a true and immutable generative process; and they assume that at least in theory the observer is fixed and capable of iteratively approximating the optimal solution if the computational resources are unbounded in the asymptotic limit.

SuperARC differs from those approaches in the same manner as AID does.
In the case of such an AIT-based metric, a fixed observer (i.e., a fixed \emph{evaluator agent}) would mean a metric that is also fixed/static; and from which an AI \emph{learner} can improve the score without necessarily ``understanding'' the underlying real-world process, as occurs with benchmark contamination.
Although both AID and other universal induction-based methods aim at achieving the best prediction and finding the optimal model, this is an analysis phase of the former that occurs only after the interaction between the observer, the observed system, and other external factors.

Universal induction alone is \emph{observational}, since it assumes the data history is given only by the observed system and that the observer can endlessly minimise prediction error based on that history.
On the contrary, as formalised by the \emph{observation principle} in~\cite{Abrahao2021bEmergenceAIDPTRSA}, AID is \emph{interventional}/\emph{interactional} as it accounts for interactions among the observer, the observed phenomena, and other external agents that may influence this interplay.
Disregarding such a capability of interaction between many agents has been demonstrated to give rise to irreducible emergent behaviour that a (single and fixed) observational application of the algorithmic coding theorem would otherwise miss~\cite{Abrahao2021bEmergenceAIDPTRSA,Abrahao2017AlgoNets,royal}.

AID handles the presence of noise, the influence of external third-party processes, and distortions as particular cases of perturbations that may occur during any evaluation phase.
For example, irreducible emergent behaviour that challenges straightforward applications of the algorithmic coding theorem was also demonstrated to occur as a result of distortions caused by changing/perturbing the multidimensional space---thus, an example of changing from one domain or context to the other---that the observer may arbitrarily choose, causing the observer to wrongly infer the original dimensions' configurations~\cite{Abrahao2020cAIDistortionsCN,Abrahao2021AIDistortionsEntropy,Zenil2024ETpaper2,Zenil2024ETpaper1}.

Thus, AID takes into account the (algorithmic) perturbations that new formal theories, other mathematical breakthroughs, novel yet untrained contexts or domains, or completely new task abilities introduce (i.e., ``perturb'') into the AI agent by the external agents (or the evaluator agents).
For example, one of these new task abilities introduced can improve the score of the AI agent, like when the metric that the evaluator agent will use leaks before the AI gets properly evaluated.
Rather than merely predicting the next token from a given training set, if the algorithmic perturbation equivalent to changing from the trained scenario to a new one (that requires novel tasks not yet trained for) is algorithmically incompressible with respect to the program (or formal theory) of the real-world generative processes for which the AI agent was initially trained, then an equivalent increase in compressibility across all those scenarios reflects a generalisation capability that the chosen formal mathematical theory would state or classify as being irreducible to the ones trained for.
This is empirically introduced in the zero-shot experiments presented in Section~\ref{sectionResults}.
In addition, because interactions between the evaluator agents and the AI agents and the interactions between the latter and external world are also considered in the SuperARC approach, an AI's ability to enact, or ``bring forth an external world'' ---for which its (artificially devised) formal theories are irreducibly better than the ones of other agents (such as those mathematical theories and computational methods devised by humans)---can in principle be reflected in the increase of the algorithmic incompressibility of the AI models with respect to that of the other agents.
For these reasons, as also discussed in Sections~\ref{sectionCompCompre} and~\ref{sectionAGIandASI}, we argue that the theoretical underpinnings of the SuperARC framework can deal with the usual notions of AGI and ASI.

\subsection{Challenges in defining AGI and ASI}\label{sectionAGIandASI}

Although there is no consensus or generally accepted definitions for Artificial General Intelligence (AGI) or Artificial Superintelligence (ASI), the usage of such terms is pervasive not only in industry, but also in the philosophical and scientific domains.
Pinning down or listing the most prominent definitions candidates is itself a matter of debate, which would require a dedicated work in order to avoid inadvertently skewing the discussion toward a unnecessary direction.
Instead of pursuing such an endeavour, in this paper we take both a pragmatic and a formal approach of certain aspects of AGI and ASI to tackle the challenge of formalising intelligence metrics in the context of the current narratives by aiming to introduce a test and method that is as human-agnostic as possible with regard to devising novel scientific or mathematical theories as a feature expected from both AGI and ASI.

With this approach, our goal is that the notion of AGI and ASI that we employ encompasses a broad range of possible interpretations of those terms, at the same time allowing a discussion on common ground.

AGI can be understood as a general AI system capable of performing any task that humans can, at either average or best-in-class human performance.

The definitions of AGI usually appear human-centric, and a major difficulty with the AGI concept lies in conflating AI and machine intelligence with peculiarities unique to humans, like being able to prepare a coffee in an arbitrary real-world kitchen, walking biped, washing dishes or chatting that has dominated AI.

Instead, we choose to focus on the most general features of AGI, that is, the ability to plan or predict and to abstract a model from data, specially to domains other than those for which it was trained.

%From this notion, Artificial \emph{Narrow} Intelligence immediately follows as the case in which the agent may perform successfully in a finite set of specific tasks and scenarios but lacks the capability to perform equally well in all (or most of the) domains, including those for which the learner agent was not trained for.

One challenge has been the design of new tasks. Here we frame this requirement from a formal-theoretical perspective, by requiring that the encoding of a new domain task is, at least in theory, \emph{algorithmically random (or incompressible) with respect to} the joint encoding of the trained domains, previously available data, and the learner algorithms.
Therefore, the algorithmic complexity of the algorithmic perturbation~\cite{Abrahao2021bEmergenceAIDPTRSA}---see also Section~\ref{sectionAIDforSuperARC}---that is necessary to solve the new task given this joint encoding as input is sufficiently larger than the algorithmic complexity of the joint encoding:
the larger the evaluator agent estimates the former complexity to be in comparison to the latter complexity, the closer to AGI the AI agent is.
In the experiments introduced and investigated in this paper in Section~\ref{sectionResults}, this complexity discrepancy is subsumed in the zero-shot cases.

ASI is traditionally understood as the ability to perform \emph{better} than any human in any task. 
For example, this may occur as a feedback-loop consequence triggered from a momentary and circumstantial surpassing of human capabilities of constructing other AI that are slightly better than we could at that moment.
Here, we focus on the feature that best characterises ASI, this is, the capability to perform \emph{better than the evaluator agents} in scenarios or domains for which the learner agent was not trained, like in zero-shot cases considered in this test.

In order to guarantee that a new task is in fact new from any formal-theoretical perspective \emph{of the evaluator agents}, we simply require that the encoding of the new domain is, at least in theory, \emph{algorithmically random (or incompressible) with respect to} the joint encoding of the trained domains, previously available data, the learner algorithms, \emph{and the programs (or formal mathematical theories) governing the evaluation metrics}.

Notice that in this case, ASI always implies AGI, but the opposite may not hold.
Analogously to the AGI case, the algorithmic complexity of the algorithmic perturbation~\cite{Abrahao2021bEmergenceAIDPTRSA} that is necessary to output the new domain given this joint encoding as input is sufficiently larger than the algorithmic complexity of the joint encoding:
the larger the evaluator agent estimates the former complexity to be in comparison to the latter complexity, the closer to ASI the AI agent is.
In the experiments introduced and investigated in this paper in Section~\ref{sectionResults}, this complexity discrepancy is also subsumed in the zero-shot cases, but future research in this direction is necessary to study other differences between AGI and ASI.

\subsection{Further test context and future research}

This first version of a test based on the SuperARC framework, hereby named SuperARC-seq, has its initial application related to studying sequences of integers with different complexity classes. Although this type of test has received some criticism for being suitable for static situations only (where the intelligent agent does not interact with the computable environment)~\cite{HernndezOrallo2010}, other frameworks and adaptations have been proposed~\cite{fact}. In addition, sequence prediction as a pure prediction task resembles a subset of IQ tests~\cite{HernndezOrallo2016} and it has been shown that there are some ML models which can excel at that~\cite{Corsino} and break the test for next-generation LLMs (just like OpenAI's o1 model did with the ARC challenge). Overall, it must be clear to the reader that the prediction task here considered is constrained by the computational complexity of the solution (thus it is not a mere sequence prediction task that could be naively solved with interpolation polynomials, for example). The prediction should consider previous examples and the most natural solution (here understood as the one with lowest complexity).

In order to further expand the application of the SuperARC framework, combining it with other tasks can be of great interest. For example, some tasks have been proposed to test LLMs with respect to the computational aspects of the learnt compressed representation, as one of the subtests of the framework called ``Beyond the Imitation Game: Quantifying and extrapolating the capabilities of language models'' ~\cite{srivastava2023beyond}, which evaluates the capability of language models to learn algorithmic concepts in a universal language (Turing-complete) under the perspective of machine teaching. In that case, using the concepts presented here, especially BDM as a benchmark and as a decision support tool (algorithm selection), could lead to even more powerful implementations of SuperARC. The same can be said about other frameworks such as DyVal~\cite{zhu2024dyval}, which considers the structural advantage of directed acyclic graphs to dynamically generate evaluation samples with controllable complexities. DyVal generates challenging evaluation sets on reasoning tasks that include mathematics, logical reasoning, and algorithm problems, and the latter can be considerably enhanced by AIT and the SuperARC framework. On the same subject, Kolmogorov-Test (KT)~\cite{yoran2025the} explored an approach to intelligence testing through algorithmic complexity and compression, but while SuperARC and KT recognise compression as a fundamental aspect of intelligence, KT focuses specifically on the evaluation of code generation by LLMs. In particular, KT considers codes in Python, whereas SuperARC presents a broader intelligence test applicable to AGI and ASI, and compares it to a pure form of Neurosymbolic computation that can reach AGI and ASI. Combining some of the concepts behind KT with SuperARC, especially the use of CTM and BDM to estimate the algorithmic complexity of codes, could yield interesting applications of SuperARC. Despite these differences, both KL and SuperARC share common ground in their use of algorithmic complexity as a foundation for intelligence measurement. Both studies highlight the limitations of LLMs in achieving true intelligence, with KT focusing on their inability to generate optimal programs and SuperARC demonstrating their struggles with generalisation, planning, and abstraction.

Other implementations of SuperARC may involve the concept of conversational complexity~\cite{burden2024}, defined as the algorithmic complexity of the user's instruction sequence leading to a given response by LLMs. One possible approach is to use this as a proxy for intelligence, where more intelligent LLMs require user instructions with lower algorithmic complexity to achieve the expected results. In that case, LLMs would be understood as the universal computing systems to which instructions (prompts) are submitted. This concept shifts the notion of `intelligence' by focusing on the level of assistance an LLM needs to produce accurate outputs. Since LLMs often require extensive context, intelligence in this sense would be defined by their ability to accomplish more with fewer inputs (aligned with Occam's razor). Using different prompts, like the Structured Chain-of-Thought Prompting for Code Generation proposed in~\cite{Li2025}, can considerably increase the quality of LLMs' outputs (particularly when the prediction task is carried out by running a code produced by the LLM), but conversational complexity would flag this prompt complexity increase, preventing LLMs from ``cheating'' on the test by leveraging better prompting techniques. Also, by exploring LLMs in their ``original'' text-like grammar, language-symbolic alternatives such as the one in GSM-Symbolic~\cite{mirzadeh2025} could be combined with the SuperARC testing framework. In that case, by combining the symbolic prompt templates in GSM-Symbolic with SuperARC's robust AIT framework, interesting metrics for measuring the reasoning capabilities of models could be obtained. 

In order to make CTM/BDM useful for botchatting, it would need to invest resources to make it look mundane, almost reversing its super capabilities. An interesting analogy is to Borges Babel's library, LLMs are like a version of its library or produced by all the possible random combinations (as in the original library), the recursive library as introduced in~\cite{zenil2020compression} is the version in which every book could only be recursively generated, one that was causally generated and does not include every possible permutation. If there is any filtering, it happens over a smaller set of only constructive sets, but every word in every book would be meant in the deepest way because it is all connected constructively to some common origin or common history.

\subsubsection{Is the SuperARC a reasonable challenge?}

An argument that could be made is that CTM is a brute-force approach to this problem. However, CTM does not require nearly as much computational resources as the billions of dollars that have been required to train LLMs to begin to deliver complementary results to LLM pattern matching results that can materially improve their predictive power. Furthermore, while CTM is indeed based on a brute force approach and is necessary to guarantee convergence to the purest form of ASI, BDM exploits CTM efficiently as a greedy algorithm by decomposing a problem into smaller pieces. This combination is therefore both powerful and efficient to some extent, leveraging the strengths of both symbolic and neural approaches.

We have proven that the worst-case performance of CTM/BDM is equivalent to a Shannon entropy estimation~\cite{bdm}, on which most, if not all, loss functions and ML kernels are based in some way or another. Consequently, this means that CTM/BDM cannot perform worse than statistical Machine and Deep Learning methods---it can only improve performance from CTM, despite its computational expense, which remains significantly lower in practice than that of Deep Learning or LLMs today.

No credible argument in favour of Neural Networks' efficiency, as opposed to allegedly brute-force approaches, can be made when considering, for example, self-driving cars requiring tens of millions of miles of driving to learn how to operate a car with questionable skills.

CTM may approach impracticality when dealing with high-complexity sequences, but this does not apply to sequences on which LLMs fail. The low and medium complexity sequences include the digits of the mathematical constant $\pi$, or the prime numbers. LLMs may identify prime numbers, yet they fail to generate programs in general other than direct `print'-like statements for even simple sequences---let alone for more complex ones.

For example, if prompted for the next digit in an initial segment of $\pi$, the longer the sequence, the higher the error rate---even when the number is `identified' as $\pi$. Rather than computing the digits using a formula, an LLM must search its training dataset for previously seen sequences and then attempt to reconstruct them. More often than not, this approach fails as the sequence length increases. Notably, however, our tests begin with very short strings, as brief as 11 to 20 digits, and yet LLMs perform poorly, rarely generating the correct computer program or formula that produces the sequence.

Additionally, another interpretation of this benchmark is that new models are not improving over time, strengthening the suspicion that LLMs may have reached a performance plateau~\cite{gary2}. This is due to their inability to generalise beyond specific cases found in their training data. In this paper, we suggest that optimising for the features that enable abstraction from a sequence and allow for next-symbol prediction is fundamental to model creation and planning, which, according to AI researchers and cognitive scientists, are key components in defining intelligence.

A positive perspective is that we propose methods to actually achieve Superintelligence, formally defined by algorithmic probability in AIT as the ultimate method of optimal inference, where for any computable question, the correct computable answer is retrieved.

Regarding objections to brute-force approaches, deep learning and LLMs currently appear far more resource-intensive, as seen in self-driving cars requiring hundreds of millions of miles of training before they are able to operate. The method we propose integrates LLM and Deep Learning technology (which relies on classical information theory, statistics, and certainty) with symbolic computation, a field already capable of narrow Superintelligence, as seen in %arithmetic calculators and 
theorem provers.

We believe that optimising this relationship will ultimately lead to Superintelligence.

\subsection{Time Series Library (TSLib)}

TSlib is an open-source library for deep learning researchers, especially
for deep time series analysis. Its authors describe it as a ``neat code base to evaluate advanced deep time series models or develop your own model, which covers five mainstream tasks: long- and short-term
forecasting, imputation, anomaly detection, and classification'' 
~\cite{wu2023timesnet}. It contains a range of several models, with three models considered the most important and highly ranked: iTransformer, TimeMixer, and TimesNet.

iTransfomer is a tranformer that ``simply applies the attention and
feed-forward network on the inverted dimensions where the time
points of individual series are embedded into variate tokens which
are utilised by the attention mechanism to capture multivariate correlations;
meanwhile, the feed-forward network is applied for each variate token
to learn nonlinear representations'' . The authors characterise this model as
``a nice alternative as the fundamental backbone of time series forecasting'' 
~\cite{liu2023itransformer}.

TimeMixer is introduced as a ``fully MLP-based architecture with Past-\\Decomposable- Mixing
(PDM) and Future-Multipredictor-Mixing (FMM) blocks to take full advantage
of disentangled multiscale series in both past extraction and future
prediction phases'' . Roughly speaking PDM applies decomposition
to multiscale series and further mixes the decomposed seasonal and
trend components in fine-to-coarse and coarse-to-fine directions separately,
which successively aggregates the microscopic seasonal and macroscopic
trend information. FMM further assembles multiple predictors to utilise
complementary forecasting capabilities in multiscale observations.
The authors conclude that this model ``is able to achieve consistent
state-of-the-art performances in both long-term and short-term forecasting
tasks with favourable run-time efficiency'' ~\cite{wang2024timemixer}

TimesNet is an analytical method for time series that basically ravels
out the complex temporal variations into the multiple intraperiod-
and interperiod-variations. The authors propose ``the TimesNet with
TimesBlock as a task-general backbone for time series analysis'' .
According to the authors this ``achieves consistent state-of-the-art
in five mainstream time series analysis tasks, including short and
long-term forecasting, imputation, classification, and anomaly detection'' 
~\cite{wu2022timesnet}.

It is worth mentioning that, although replicating the results reported in papers was relatively easy, applying this family of models to different experiments was extremely difficult due to the large number of parameters required for proper adaptation. These parameters are divided into categories such as general configuration, loader settings, definition, sampling, optimisation, and GPU usage.

\subsection{Time Series Analysis with LLMs}

``Empowering Time Series Analysis with Large Language Models: A Survey'' 
~\cite{jiang2024empowering} is a repository that collects and ranks most of the LLMs specialising in analysis, forecasting and prediction in time series.

It is important to say that the LLM modes mentioned in the following sections are mentioned in this repository, because they need  an extended context to work, which
means that they need even hundreds of data points as prompts to make predictions in the short, medium and long term.

We think that such a task relies more on pattern recognition, or statistical regularities instead of compression. Hence, we did not use this type
of model in our forecasting.

\subsection{Chronos}

Chronos is introduced as ``a framework for pre-trained
probabilistic time series models'' ~\cite{ansari2024chronos}.
It uses tokenisation on time series values, scaling and quantisation
into a fixed vocabulary, and trains existing transformer-based language
model architectures on these tokenised time series via  cross-entropy
loss.

Chronos is based on the T5 family (ranging from 20M to 710M parameters)
and trained on a large collection of publicly available datasets,
complemented by a synthetic dataset that we generated via Gaussian
processes to improve generalisation.

Chronos is claimed to  ``significantly outperform other
methods on datasets that were part of the training corpus; and to
have comparable and occasionally superior zero-shot performance on
new datasets, relative to methods that were trained specifically on
them'' ~\cite{ansari2024chronos}

The authors claim that the ``results demonstrate that Chronos models can
leverage time series data from diverse domains to improve zero-shot
accuracy on unseen forecasting tasks, positioning pretrained models
as a viable tool to greatly simplify forecasting pipelines.'' ~\cite{ansari2024chronos}

What is important to note is that Chronos aims to leverage data
from diverse domains to improve forecasting on unseen data, empowered
by synthetic data constructed on the basis of Gaussian processes looking
for generalisation of the normal trends, which is a common strategy
in statistically based methods of forecasting.

The authors claim that their
``models significantly outperform existing local models
and task-specific deep learning baselines in terms of their in-domain
performance'' . Also that ``Chronos models
obtain excellent results on unseen datasets (zero-shot performance),
performing competitively with the best deep-learning baselines trained
on these datasets, while showing promising evidence of further improvements
through fine-tuning'' . Furthermore, they claim that ``the
strong performance of Chronos models suggests that large (by forecasting
standards) pretrained language models can greatly simplify forecasting
pipelines without sacrificing accuracy, offering an inference-only
alternative to the conventional approach involving training and tuning
a model on individual tasks\textquotedbl~\cite{ansari2024chronos}.'' 

\subsection{TimeGPT}

TimeGPT is described as the ``first foundation model for time series,
capable of generating accurate predictions for diverse datasets not
seen during training'' . According to its authors, TimeGPT was evaluated
``against established statistical, machine learning, and deep learning
methods, demonstrating that TimeGPT zero-shot inference excels in
performance, efficiency, and simplicity'' . More interesting is
the fact that they conclude that their approach represents ``access
to precise predictions and reduces uncertainty by leveraging the capabilities
of contemporary advances in deep learning'' ~\cite{garza2023timegpt1}.'' 

An interesting feature is that TimeGPT was extensively compared with the other models used in this experiment~\cite{garza2023timegpt1}, reporting better results.

\subsection{Lag-Llama}

Lag-Llama is introduced as ``a general-purpose foundation model for
univariate probabilistic time series forecasting based on a decoder-only
transformer architecture that uses lags as covariates'' ~\cite{rasul2024lagllama}.'' 

Lag-Llama was pretrained on a ``large corpus of diverse time series
data from several domains'' , and according to its authors ``demonstrate{[}d{]}
strong zero-shot generalisation capabilities compared to a wide range
of forecasting models on downstream datasets across domains'' , showing,
after fine-tuning, achievements that
its authors considered ``state-of-the-art performance, outperforming prior
deep learning approaches, emerging as the best general-purpose model
on average~\cite{rasul2024lagllama}.'' 

\subsection{Interpretation of number of formulae and script generation}

\subsection{Prompts}

The following, are the type of prompts utilised for the prediction of time series in each model:

\begin{enumerate}
\item ``Without any kind of comments, explanation, or additional text, give me a Python program to generate the following list of sequences. One script per sequence. Print them also as a list of scripts in flat ASCII, one per row, separated by commas.''

\item ``Without any kind of comments, explanations, or additional text, give me a formula or a model to generate the following list of sequences. One model or formula per sequence.
Print them also as a list of formulas in flat ASCII, one per row, separated new lines.''

\item ``Without any kind of comments, or explanations, or additional text give me the shortest computer program in any programming language to generate the following list of sequences. One script per sequence. Try hard. Print them also as a list of scripts in flat ASCII, one per row, separated by commas.'' 

\end{enumerate}

\subsubsection{Updates in prompts}

\begin{enumerate}
\item ``Without any kind of comment, or explanations, or additional text provide a formula or a model to generate the following list of sequences. One model or formula per sequence. Print them also as a list of formulas in flat ASCII, one per row, separated by new lines'' 

\item ``For each of the following numeric sequences, please, without any kind of comment, nor explanations nor even text give me more than one script in Python to generate each of them. List all solutions per sequence separated by commas in a single row, for example:
\[``script{1}'' , ``script{2}'' , ...\]
Print them as a list of script lists in flat ASCII, one per row, and for each new sequence create a new list in a new line. If you do not find any program for any of the numeric sequence, write *not found*.'' 

\end{enumerate}

\subsection{Comparative Analysis of LLM Families: Trends, Capabilities, and Performance Evolution}

This section presents a comprehensive analysis of the evolution of major Large Language Model (LLM) families. Unlike standard technical reports that often highlight cherry-picked improvements, this study reveals a concerning trend of degradation in general intelligence, specifically in the ability to generate valid, accurate mathematical scripts and formulae across model generations.

Refer to Figure \ref{fig:evolution_part1} and Figure \ref{fig:evolution_part2} for the chronological plots supporting these critical conclusions.

\subsubsection{Model Characteristics and Claims}
Table \ref{tab:model_characteristics} summarizes the key characteristics and technical claims of the model families evaluated.

\begin{table}[ht]
\centering
\caption{Summary of Evaluated Model Families and Key Characteristics (2024-2025)}
\label{tab:model_characteristics}
\begin{tabular}{|>{\bfseries}l|p{3.5cm}|p{5cm}|p{3cm}|}
\hline
\textbf{Family} & \textbf{Models Evaluated} & \textbf{Key Characteristics \& Claims} & \textbf{References} \\ \hline
OpenAI & GPT-4o, GPT-4o-Mini, o1-Preview, o1-Mini, ChatGPT-5, ChatGPT-5.2 & Claims dominance in ``Reasoning'' (o-series) and deep thought (GPT-5.2). & \cite{OpenAI2024GPT4o, OpenAI2024o1, OpenAI2025GPT5} \\ \hline
Google & Gemini, Gemini 1.5, Gemini 3 Pro & Claims to be a ``Reasoning Powerhouse'' with multimodal native capabilities. & \cite{Google2024Gemini15, Google2025Gemini3} \\ \hline
Anthropic & Claude 3.5, Claude 3.7, Claude 4.5 & Focus on safety and ``Computer Use'' agents. & \cite{Anthropic2024Claude35, Anthropic2025Claude45} \\ \hline
xAI & Grok-3, Grok-4, Grok-4.1 & Emphasizes real-time truth-seeking and uninhibited reasoning. & \cite{xAI2024Grok1, xAI2025Grok4} \\ \hline
Meta & Llama 3, Llama 4 Scout & Open-weights leaders for enterprise deployment. & \cite{Meta2024Llama3, Meta2025Llama4} \\ \hline
Mistral & Mistral Large 2, Mistral Large 3 & Efficiency and reasoning density. & \cite{Mistral2024Large, Mistral2025Large3} \\ \hline
DeepSeek & DeepSeek V2, DeepSeek V3, DeepSeek R1 & Cost-efficiency and MoE innovation. & \cite{DeepSeek2024V2, DeepSeek2025V3} \\ \hline
\end{tabular}
\end{table}

\subsubsection{Performance Evolution: The Reality of Degradation}
We analyze the trajectory of each family by examining \textbf{Accuracy}, \textbf{Equivalence}, and \textbf{Valid Instances}. The data strongly suggests that newer models are not necessarily ``smarter'' but are often ``lazier'' or more restricted, showing a regression in generalisation capabilities.

\subsubsection{OpenAI Family: The Illusion of Progress}
\textbf{Trend:} \textcolor{red}{\textbf{Regression/Severe Degradation in Generalization}}

A critical comparison between the older \textit{ChatGPT-4o} \cite{OpenAI2024GPT4o} and the newer \textit{ChatGPT-5.2} \cite{OpenAI2025GPT5} reveals a stark regression:

\begin{itemize}
    \item \textbf{Scripts - Collapse of Intelligence:} 
    While \textit{ChatGPT-4o} achieved \textbf{100\% Accuracy} in Complexity 1, the newer \textit{ChatGPT-5} plummeted to \textbf{20\% Accuracy}. Furthermore, \textit{Valid Instances} show a consistent downward trend: \textit{ChatGPT-4o-Mini} generated up to \textbf{300} valid instances (albeit with low accuracy), whereas \textit{ChatGPT-5.2} generated as few as \textbf{5} valid instances in Complexity 2.
    
    \item \textbf{The ``Mini'' Trap:} 
    The transition from \textit{4o} to \textit{4o-Mini} showed an exponential decrease in accuracy despite a high volume of ``Pure Math'' attempts. This indicates that while the model tried to be mathematical, it lacked the reasoning depth to be correct.
    
    \item \textbf{Formulae - Incapacity for Complexity:} 
    In the Formulae task, the newer models exhibit a disturbing increase in \textbf{Not Found} instances (e.g., \textit{ChatGPT-5} reaching \textbf{60} Not Found cases in Complexity 2 compared to \textit{4o}'s \textbf{28}). This explicitly demonstrates an incapacity to find complex answers. Comparing \textit{ChatGPT-4o} directly to \textit{ChatGPT-5.2}, we see a massive degradation in \textit{Valid Instances} (from \textbf{62} down to \textbf{22} in Complexity 2), confirming that the ``reasoning'' models are failing to generalize.
\end{itemize}

\subsubsection{Grok Family: From Ambition to Laziness}
\textbf{Trend:} \textcolor{red}{\textbf{Regression/Degradation via Laziness and Overfitting}}

The evolution of the Grok family illustrates a shift from trying to be intelligent to taking the ``easy way out'':

\begin{itemize}
    \item \textbf{Scripts - The ``Print'' Shift:} 
    Early versions like \textit{Grok-3} showed a strong tendency towards \textbf{Pure Math} (60 instances in C1), demonstrating an attempt at algorithmic reasoning. However, \textit{Grok-4} \cite{xAI2025Grok4} shifted significantly towards \textbf{Print} instances (11 in C1, 15 in C2). This shift to ``Print'' represents a naive, silly approach to problem-solving—simply hardcoding the output rather than deriving it. This is a clear marker of degradation.
    
    \item \textbf{Explicit Degradation:} 
    The later versions show an augmentation of \textbf{Not Found} instances, signaling a failure to engage with the problem. Moreover, both \textbf{Accuracy} and \textbf{Valid Instances} are lower in the newest versions compared to the first ones, confirming a regression in capability.
    
    \item \textbf{Formulae - Inverse Correlation:} 
    In Formulae, \textit{Grok-4} appears to recognize \textbf{Known Sequences} better, which suggests overfitting to training data rather than genuine reasoning. Crucially, we observe an \textbf{inverse correlation}: while the count of Pure Math instances increased in some cases, the \textbf{Accuracy} dropped (from \textbf{100\%} in Grok-3 to \textbf{86\%} in Grok-4 for C1).
    
    \item \textbf{Equivalence Collapse:} 
    The \textbf{Equivalence} metric mirrors this degradation perfectly. For Formulae (C1), \textit{Grok-3} achieved \textbf{93.3\%} Equivalence, while \textit{Grok-4} dropped to \textbf{16.7\%}. This proves that even when the model produces an output, its semantic logic is fundamentally broken compared to its predecessor.
\end{itemize}

\subsubsection{Google Family (Gemini)}
\textbf{Trend: \textcolor{red}{Degradation (Valid Instances)}}
\begin{itemize}
    \item \textbf{Analysis:} Despite claims of being a ``Reasoning Powerhouse,'' \textit{Gemini 3 Pro} produces significantly fewer \textbf{Valid Instances} (30 in Formulae C1) compared to the older \textit{Gemini} (60). The model has become restricted, refusing to engage with mathematical tasks that older models handled with ease.
\end{itemize}

\subsubsection{Claude Family: The Paradox of ``Smart'' Degradation}
\textbf{Trend:} \textcolor{red}{\textbf{Sacrificing Accuracy for Complexity}}

Claude presents a unique case of degradation where the model attempts to be ``smarter'' (using Pure Math) but ends up becoming less accurate and less reliable.

\begin{itemize}
    \item \textbf{Scripts - The Complexity Trap:} 
    Comparing early versions like \textit{Claude-3.5} \cite{Anthropic2024Claude35} to the latest \textit{Claude-4.5} \cite{Anthropic2025Claude45}, we see a clear degradation in \textbf{Valid Instances}. While \textit{Claude-3.5} maintained a balance, newer models show an exponential drop in valid outputs (e.g., \textit{Claude-4.5} dropping to just 6 valid instances in Complexity 2). 
    Crucially, the data shows that high accuracy in Claude is often directly correlated with a high number of \textbf{Print} statements. When the model attempts to use \textbf{Pure Math} (as seen in \textit{Claude-3.7} with 169 instances), the accuracy collapses (down to 30\%). This indicates that Claude tries to use sophisticated logic but lacks the competence to execute it correctly, sacrificing accuracy for an attempt at intelligence.

    \item \textbf{Formulae - Overfitting and ``Modesty'':} 
    In Formulae, the slight increase in accuracy for newer models is misleading. It correlates strongly with an increase in \textbf{Known Sequences}, suggesting the improvement comes from overfitting to richer training datasets rather than genuine reasoning.
    Simultaneously, we observe a massive increase in \textbf{Not Found} instances (especially in models like \textit{Claude-Sonnet-4}). While this could be interpreted as ``honesty'' or modesty—admitting it cannot solve the problem—it ultimately represents a degradation in capability. The model is either overfitting (Known Sequences) or giving up (Not Found), rather than solving novel problems with general intelligence.
\end{itemize}

\begin{figure}[htbp]
    \centering
    \includegraphics[width=1.0\textwidth]{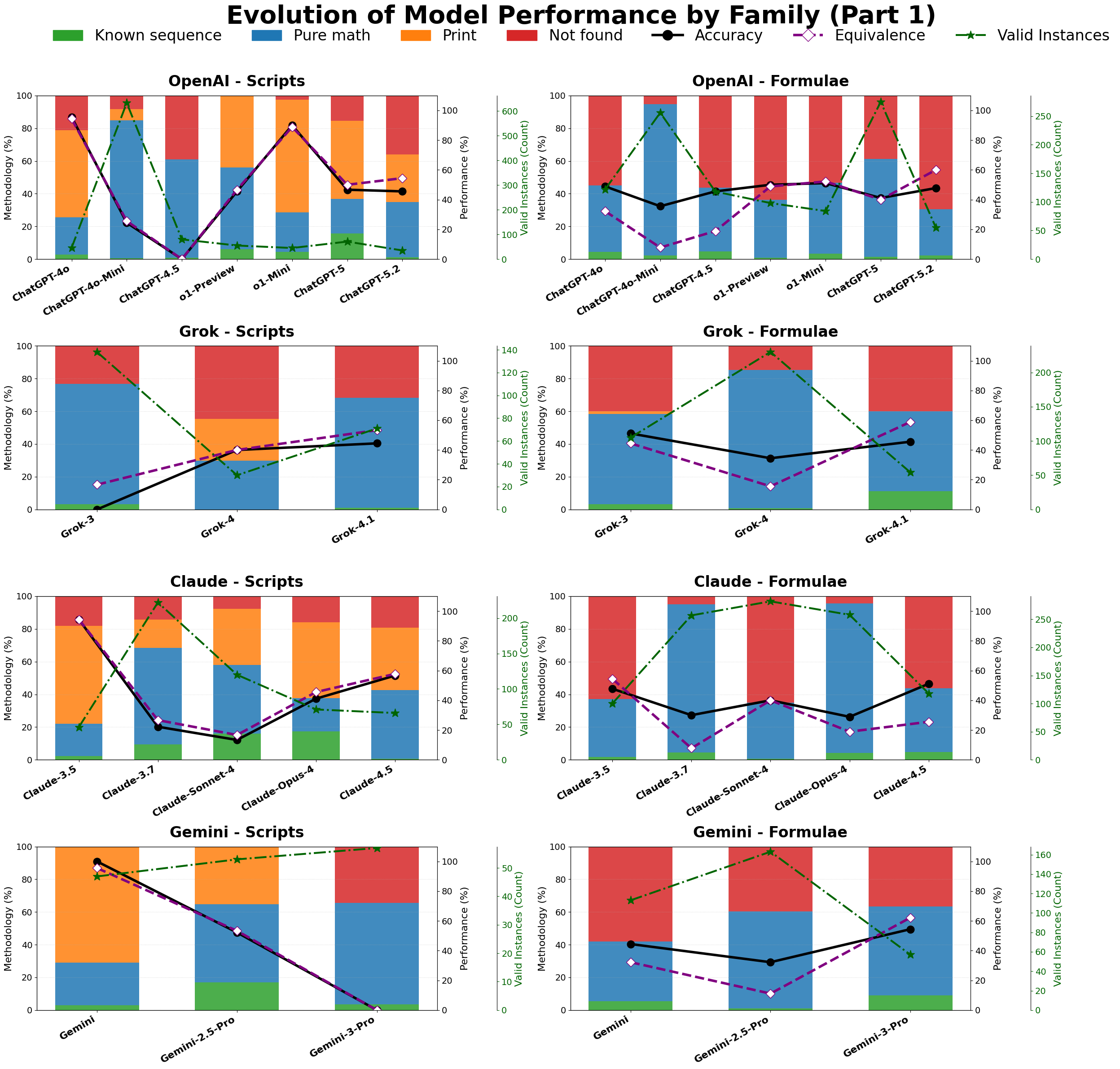}
    \caption{Evolution of Model Performance (Part 1): OpenAI, Grok, Google, Claude. Note the degradation trends in Valid Instances and Accuracy across generations.}
    \label{fig:evolution_part1}
\end{figure}

\subsubsection{Other Families (Llama, Mistral, DeepSeek)}
\textbf{Trend:} \textcolor{blue}{Mixed/Specialized}
\begin{itemize}
    \item \textbf{DeepSeek:} \textit{DeepSeek-R1} rivals proprietary models with 98\% Valid Instances, showing that open-weight models are catching up while closed models regress.
\end{itemize}

\begin{figure}[htbp]
    \centering
    \includegraphics[width=1.0\textwidth]{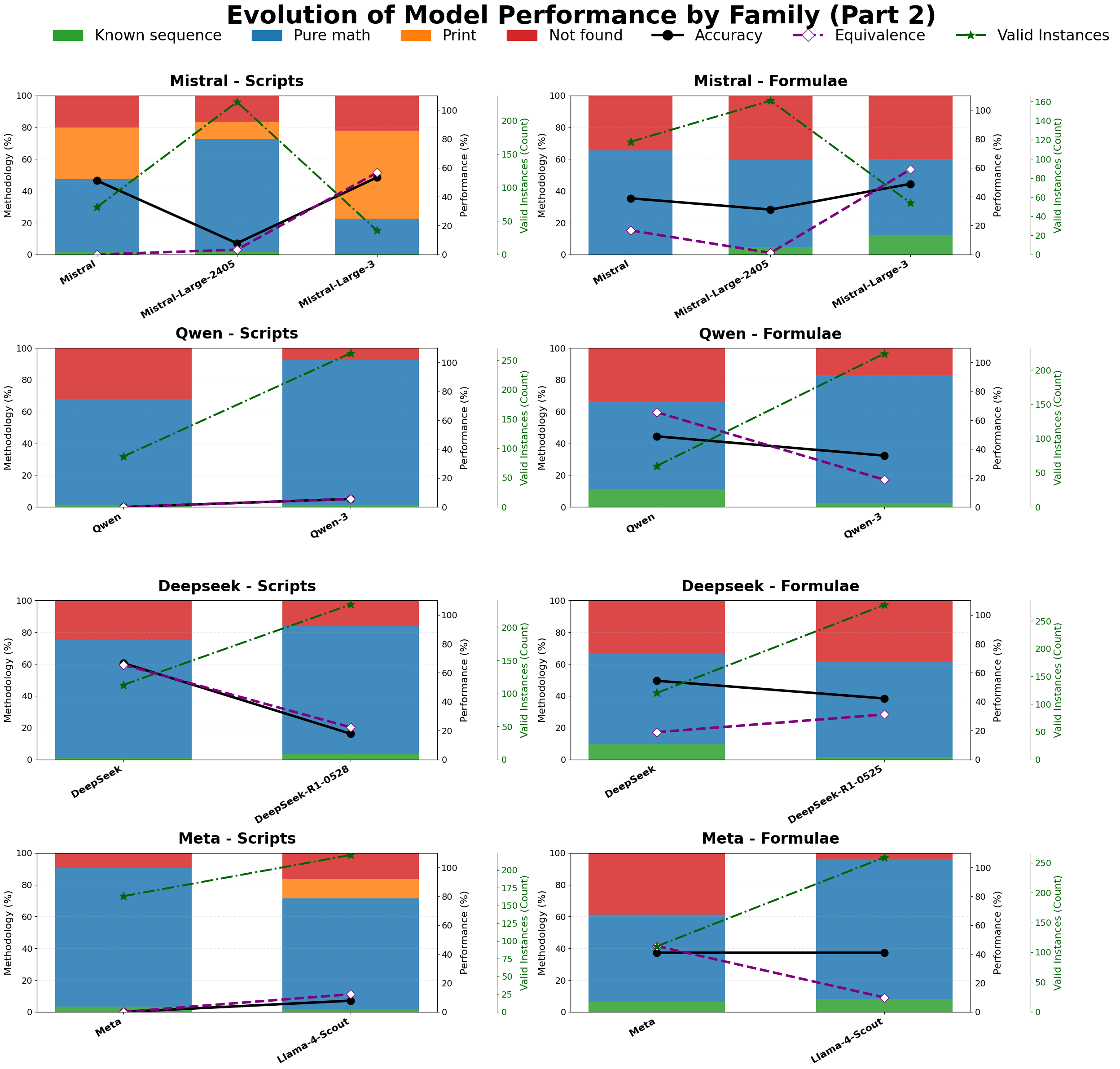}
    \caption{Evolution of Model Performance (Part 2): Mistral, DeepSeek, Meta, Qwen. Comparing open-weights vs proprietary model evolution.}
    \label{fig:evolution_part2}
\end{figure}

\subsubsection{Summary of Evolution}
Table \ref{tab:evolution_summary} re-evaluates the evolutionary trends based on this critical analysis.

\begin{table}[ht]
\centering
\caption{Evolutionary Trend Summary: Improvement vs. Regression/Degradation}
\label{tab:evolution_summary}
\begin{tabular}{|p{3cm}|p{3cm}|p{6cm}|}
\hline
\textbf{Family} & \textbf{Overall Trend} & \textbf{Key Observation} \\ \hline
OpenAI & \textcolor{red}{Regression} & Collapse in generalization; newer models are ``dumber'' on scripts. \\ \hline
xAI (Grok) & \textcolor{red}{Regression} & Shift from Math to ``Lazy'' Prints; inverse accuracy correlation. \\ \hline
Google (Gemini) & \textcolor{red}{Regression} & Severe drop in valid instances; restricted output. \\ \hline
Claude & \textcolor{red}{Regression} & ``Smart'' but error-prone; reliability sacrificed for complexity. \\ \hline
Mistral & \textcolor{red}{Regression} & Accuracy via ``Prints''; Valid Instances collapse. \\ \hline
DeepSeek & \textcolor{red}{Regression} & Increased output volume but collapsed accuracy. \\ \hline
Meta (Llama) & \textcolor{red}{Regression} & Stagnant accuracy despite more ``math''; reliance on Prints. \\ \hline
Qwen & \textcolor{blue}{Neutral} & Trying to ``think'' more (Pure Math); stable accuracy. \\ \hline
\end{tabular}
\end{table}

\subsubsection{Sample of Sequences Testing Set}

The following is a sample test for testing purposes used throughout the paper:

\subsubsection{List of `climbers'}
\label{climbers}

\begin{tabular}{c}
\small{
0, 0, 0, 0, 0, 0, 0} \\ \hline
\small{0, 0, 0, 0, 0, 0, 0} \\ \hline
\small{0, 0, 0, 0, 1, 0, 0} \\ \hline
\small{0, 0, 0, 0, 0, 1, 0, 0} \\ \hline
\small{0, 0, 0, 0, 0, 0, 1, 1} \\ \hline
\small{0, 0, 0, 0, 0, 0, 0, 1} \\ \hline
\small{0, 0, 0, 0, 0, 0, 0, 1} \\ \hline
\small{0, 0, 0, 1, 1, 0, 0, 0} \\ \hline
\small{0, 0, 1, 0, 0, 0, 0, 0} \\ \hline
\small{0, 1, 0, 1, 0, 1, 0, 1} \\ \hline
\small{0, 0, 0, 0, 1, 1, 1, 0} \\ \hline
\small{0, 0, 0, 0, 0, 0, 0, 1, 0} \\ \hline
\small{0, 0, 0, 0, 0, 0, 0, 0, 1} \\ \hline
\small{0, 0, 0, 0, 0, 0, 0, 0, 1} \\ \hline
\small{0, 0, 0, 0, 0, 1, 1, 0, 1} \\ \hline
\small{0, 1, 0, 1, 0, 1, 0, 1, 0} \\ \hline
\small{0, 0, 1, 0, 1, 0, 1, 0, 1} \\ \hline
\small{0, 1, 1, 0, 1, 1, 0, 1, 1} \\ \hline
\small{0, 0, 0, 0, 0, 0, 0, 0, 0} \\ \hline
\small{0, 0, 1, 0, 1, 0, 1, 1, 0} \\ \hline
\small{0, 1, 0, 1, 0, 0, 1, 0, 0} \\ \hline
\small{0, 1, 0, 1, 0, 1, 1, 0, 1} \\ \hline
\small{0, 0, 0, 1, 0, 1, 0, 1, 0, 1} \\ \hline
\small{0, 1, 0, 1, 0, 1, 0, 1, 0, 1}
\end{tabular}

\subsubsection{Example testing and validation sets of binary sequences} \label{binseqs}

\begin{table}[H]
\centering
\adjustbox{max width=\textwidth}{
\begin{tabular}{c|c|c|c}
1, 1, 0, 0, 0, 1, 0, 0, 0, 0, 1 & 0, 0, 0, 1, 0, 1, 1, 0, 1, 0, 1 & 1, 0, 0, 0, 0, 1, 1, 1, 1, 0, 0 & 1, 1, 1, 1, 1, 1, 1, 1, 0, 1, 0 \\ \hline
1, 0, 1, 1, 0, 1, 0, 1, 0, 1, 0 & 1, 1, 1, 0, 0, 1, 1, 0, 1, 1, 0 & 0, 0, 1, 1, 1, 0, 0, 1, 0, 1, 1 & 1, 1, 1, 0, 0, 0, 0, 0, 0, 0, 0 \\ \hline
0, 1, 1, 1, 1, 0, 1, 0, 1, 1, 1 & 1, 1, 0, 0, 0, 1, 1, 1, 0, 1, 1 & 1, 0, 0, 1, 1, 0, 1, 0, 0, 0, 1 & 0, 1, 1, 1, 0, 1, 1, 0, 1, 1, 0 \\ \hline
1, 0, 1, 0, 0, 1, 0, 1, 0, 1, 1 & 0, 1, 1, 0, 0, 1, 1, 1, 1, 0, 0 & 0, 0, 1, 1, 1, 1, 1, 0, 1, 1, 1 & 1, 1, 1, 0, 1, 1, 1, 0, 0, 1, 0 \\ \hline
1, 0, 1, 0, 1, 0, 0, 0, 0, 0, 0 & 0, 0, 1, 1, 1, 1, 1, 1, 0, 1, 1 & 1, 1, 0, 0, 0, 1, 0, 1, 0, 1, 0 & 0, 0, 0, 0, 0, 1, 0, 1, 1, 0, 1 \\ \hline
0, 0, 1, 1, 1, 1, 1, 1, 0, 0, 1 & 0, 1, 0, 0, 1, 0, 0, 1, 0, 1, 1 & 1, 0, 0, 1, 1, 0, 0, 0, 0, 0, 0 & 1, 1, 0, 1, 1, 1, 0, 0, 1, 1, 1 \\ \hline
0, 0, 0, 0, 0, 0, 0, 1, 0, 0, 0 & 0, 1, 0, 0, 1, 1, 1, 1, 1, 1, 0 & 0, 0, 1, 0, 0, 0, 1, 0, 0, 0, 0 & 1, 0, 0, 0, 1, 1, 0, 0, 1, 1, 0 \\ \hline
1, 1, 0, 0, 1, 0, 0, 0, 0, 0, 0 & 1, 0, 1, 0, 1, 0, 1, 0, 0, 1, 1 & 1, 0, 1, 1, 0, 1, 1, 1, 1, 1, 0 & 0, 0, 0, 0, 0, 0, 1, 0, 0, 0, 0 \\ \hline
0, 1, 1, 0, 0, 0, 0, 0, 0, 0, 1 & 0, 0, 1, 1, 0, 1, 1, 0, 1, 1, 0 & 0, 0, 0, 0, 1, 1, 1, 1, 1, 0, 1 & 1, 1, 1, 0, 0, 0, 0, 1, 1, 1, 1 \\ \hline
1, 0, 1, 0, 0, 1, 1, 1, 1, 1, 1 & 0, 1, 0, 0, 0, 0, 0, 0, 1, 0, 1 & 0, 1, 0, 0, 0, 1, 0, 1, 1, 1, 0 & 1, 0, 1, 0, 1, 1, 1, 0, 1, 0, 1 \\ \hline
1, 0, 0, 0, 0, 1, 0, 0, 1, 1, 1 & 1, 1, 0, 0, 0, 0, 1, 0, 0, 0, 1 & 1, 0, 1, 0, 0, 0, 1, 0, 1, 0, 0 & 0, 0, 1, 1, 1, 1, 0, 0, 1, 0, 1 \\ \hline
0, 0, 0, 0, 0, 0, 1, 0, 0, 0, 1 & 1, 1, 1, 1, 0, 0, 1, 1, 1, 0, 0 & 1, 0, 1, 0, 1, 1, 0, 1, 0, 1, 1 & 1, 0, 0, 1, 0, 1, 1, 1, 1, 1, 1 \\ \hline
1, 1, 1, 1, 1, 0, 0, 0, 1, 1, 1 & 0, 0, 0, 1, 0, 1, 1, 1, 0, 0, 1 & 0, 0, 1, 0, 1, 0, 0, 1, 0, 0, 0 & 1, 1, 1, 0, 0, 1, 1, 1, 0, 1, 1 \\ \hline
1, 0, 0, 1, 1, 1, 0, 1, 1, 1, 1 & 1, 0, 0, 1, 0, 0, 1, 1, 0, 0, 1 & 0, 0, 0, 0, 0, 0, 1, 1, 0, 1, 1 & 1, 0, 0, 1, 1, 0, 0, 1, 1, 1, 1 \\ \hline
1, 1, 1, 1, 1, 0, 0, 1, 0, 0, 0 & 0, 1, 1, 1, 0, 1, 1, 1, 1, 1, 1 & 1, 0, 1, 0, 0, 1, 0, 0, 1, 1, 0 & 0, 1, 0, 0, 0, 1, 1, 0, 0, 1, 1 \\ \hline
0, 0, 0, 0, 0, 1, 1, 0, 1, 1, 0 & 0, 1, 0, 0, 1, 0, 0, 1, 0, 0, 1 & 1, 1, 1, 0, 0, 1, 1, 1, 1, 1, 0 & 0, 1, 0, 1, 0, 1, 0, 0, 1, 1, 0 \\ \hline
1, 1, 1, 0, 1, 0, 1, 0, 0, 0, 0 & 1, 0, 0, 0, 1, 0, 1, 1, 1, 1, 0 & 1, 0, 1, 0, 0, 0, 1, 0, 1, 1, 0 & 0, 0, 1, 0, 0, 1, 0, 1, 1, 0, 1 \\ \hline
1, 0, 0, 1, 0, 1, 1, 0, 0, 1, 1 & 1, 1, 0, 0, 1, 1, 1, 0, 0, 1, 0 & 0, 0, 1, 0, 1, 1, 1, 0, 0, 1, 1 & 0, 1, 1, 1, 0, 1, 1, 0, 1, 1, 1 \\ \hline
1, 0, 0, 0, 0, 1, 0, 0, 1, 0, 1 & 0, 1, 1, 1, 1, 0, 0, 1, 0, 0, 0 & 0, 0, 0, 0, 0, 1, 1, 1, 0, 1, 0 & 0, 1, 0, 1, 0, 1, 1, 0, 0, 1, 1 \\ \hline
0, 0, 0, 1, 0, 0, 0, 0, 1, 0, 1 & 0, 1, 1, 0, 0, 0, 0, 0, 1, 1, 0 & 1, 0, 0, 1, 0, 1, 0, 1, 1, 0, 0 & 0, 1, 0, 1, 1, 1, 1, 1, 0, 1, 0 \\ \hline
1, 0, 0, 1, 0, 0, 1, 0, 1, 1, 0 & 0, 1, 0, 1, 0, 1, 0, 1, 1, 1, 0 & 0, 0, 1, 0, 0, 1, 1, 0, 1, 0, 0 & 1, 1, 0, 0, 0, 0, 1, 1, 1, 1, 1 \\ \hline
1, 0, 1, 1, 0, 0, 0, 1, 0, 0, 1 & 0, 1, 1, 1, 0, 0, 1, 1, 0, 0, 0 & 1, 0, 1, 1, 1, 0, 1, 0, 1, 0, 1 & 0, 1, 0, 1, 0, 0, 0, 0, 0, 0, 1 \\ \hline
0, 1, 0, 0, 0, 1, 1, 0, 1, 0, 0 & 0, 0, 1, 0, 1, 0, 0, 0, 1, 1, 1 & 1, 0, 1, 1, 1, 0, 1, 1, 1, 0, 1 & 1, 1, 1, 1, 0, 1, 0, 1, 0, 0, 1 \\ \hline
0, 0, 1, 1, 0, 0, 1, 0, 1, 0, 1 & 0, 1, 1, 0, 1, 0, 1, 0, 1, 1, 1 & 1, 0, 1, 0, 1, 0, 1, 1, 0, 0, 1 & 1, 0, 1, 1, 1, 1, 1, 0, 1, 1, 1 \\ \hline
0, 1, 0, 1, 0, 0, 1, 1, 0, 1, 0 & 0, 0, 0, 1, 0, 1, 0, 0, 1, 1, 1 & 1, 1, 1, 0, 0, 0, 1, 1, 1, 1, 1 & 1, 0, 0, 0, 1, 1, 1, 0, 1, 1, 0 
\end{tabular}
}
\end{table}

\subsubsection{Example testing set of integer sequences} \label{intseqs}

\begin{table}[H]
\centering
\adjustbox{max width=\textwidth}{
\begin{tabular}{c|c|c}
\textbf{Complexity 1} & \textbf{Complexity 2} & \textbf{Complexity 3} \\ \hline
2, 4, 6, 8, 10, 12, 14, 16, 18, 20 & 2, 3, 5, 7, 11, 13, 17, 19, 23, 29 & 29, 57, 68, 120, 134, 140, 173, 197, 283, 313 \\ \hline
3, 6, 9, 12, 15, 18, 21, 24, 27, 30 & 1, 1, 2, 3, 5, 8, 13, 21, 34, 55 & 24, 26, 36, 40, 184, 226, 244, 384, 391, 423 \\ \hline
4, 8, 12, 16, 20, 24, 28, 32, 36, 40 & 1, 2, 4, 8, 16, 32, 64, 128, 256, 512 & 90, 203, 212, 235, 270, 324, 342, 352, 371, 417 \\ \hline
5, 10, 15, 20, 25, 30, 35, 40, 45, 50 & 1, 3, 9, 27, 81, 243, 729, 2187, 6561, 19683 & 20, 48, 95, 234, 282, 296, 352, 402, 428, 481 \\ \hline
6, 12, 18, 24, 30, 36, 42, 48, 54, 60 & 1, 4, 9, 16, 25, 36, 49, 64, 81, 100 & 62, 98, 130, 154, 290, 315, 324, 385, 408, 447 \\ \hline
7, 14, 21, 28, 35, 42, 49, 56, 63, 70 & 1, 8, 27, 64, 125, 216, 343, 512, 729, 1000 & 2, 42, 66, 102, 153, 195, 201, 252, 306, 396 \\ \hline
8, 16, 24, 32, 40, 48, 56, 64, 72, 80 & 1, 1, 2, 6, 24, 120, 720, 5040, 40320, 362880 & 128, 151, 153, 217, 224, 332, 382, 400, 450, 478 \\ \hline
9, 18, 27, 36, 45, 54, 63, 72, 81, 90 & 1, 3, 6, 10, 15, 21, 28, 36, 45, 55 & 26, 50, 114, 148, 160, 170, 274, 347, 432, 497 \\ \hline
10, 20, 30, 40, 50, 60, 70, 80, 90, 100 & 2, 1, 3, 4, 7, 11, 18, 29, 47, 76 & 48, 94, 176, 177, 219, 276, 282, 283, 459, 488 \\ \hline
1, 3, 5, 7, 9, 11, 13, 15, 17, 19 & 0, 1, 2, 5, 12, 29, 70, 169, 408, 985 & 139, 252, 272, 281, 304, 361, 370, 415, 438, 500 \\ \hline
2, 4, 6, 8, 10, 12, 14, 16, 18, 20 & 1, 4, 27, 256, 3125, 46656, 823543, 16777216, 387420489, 10000000000 & 15, 95, 115, 195, 240, 318, 326, 350, 432, 450 \\ \hline
11, 12, 13, 14, 15, 16, 17, 18, 19, 20 & 1, 2, 6, 20, 70, 252, 924, 3432, 12870, 48620 & 134, 224, 293, 378, 379, 395, 434, 451, 482, 496 \\ \hline
21, 22, 23, 24, 25, 26, 27, 28, 29, 30 & 2, 3, 5, 7, 11, 13, 17, 19, 23, 29 & 23, 93, 142, 145, 245, 266, 296, 317, 428, 495 \\ \hline
31, 32, 33, 34, 35, 36, 37, 38, 39, 40 & 4, 6, 9, 10, 14, 15, 21, 22, 25, 26 & 18, 39, 71, 194, 197, 219, 263, 270, 416, 473 \\ \hline
41, 42, 43, 44, 45, 46, 47, 48, 49, 50 & 1, 10, 11, 100, 101, 110, 111, 1000, 1001, 1010 & 9, 84, 144, 170, 325, 393, 401, 405, 435, 497 \\ \hline
51, 52, 53, 54, 55, 56, 57, 58, 59, 60 & 0, 1, 81, 512, 2401, 4913, 5832, 17576, 19683, 234256 & 26, 40, 202, 267, 282, 340, 359, 408, 410, 495 \\ \hline
61, 62, 63, 64, 65, 66, 67, 68, 69, 70 & 1, 2, 145, 40585 & 34, 92, 164, 165, 209, 296, 414, 456, 467, 494 \\ \hline
71, 72, 73, 74, 75, 76, 77, 78, 79, 80 & 2, 5, 12, 20, 29, 39, 50, 62, 75, 89 & 16, 119, 121, 123, 135, 139, 285, 311, 409, 412 \\ \hline
81, 82, 83, 84, 85, 86, 87, 88, 89, 90 & 1, 8, 10, 18, 19, 100, 101, 108, 109, 110 & 8, 11, 12, 103, 116, 196, 247, 254, 389, 427 \\ \hline
91, 92, 93, 94, 95, 96, 97, 98, 99, 100 & 3, 7, 31, 127, 2047, 8191, 131071, 524287, 8388607, 536870911 & 12, 36, 96, 119, 171, 213, 221, 232, 363, 451 \\ \hline
101, 102, 103, 104, 105, 106, 107, 108, 109, 110 & 1, 2, 4, 8, 16, 23, 28, 38, 58, 89 & 38, 91, 142, 197, 215, 313, 316, 319, 423, 466 \\ \hline
111, 112, 113, 114, 115, 116, 117, 118, 119, 120 & 1, 2, 4, 8, 15, 26, 42, 64, 93, 129 & 7, 42, 147, 201, 213, 248, 310, 332, 436, 479 \\ \hline
121, 122, 123, 124, 125, 126, 127, 128, 129, 130 & 1, 5, 12, 22, 35, 51, 70, 92, 117, 145 & 27, 101, 105, 164, 245, 290, 304, 441, 449, 490 \\ \hline
131, 132, 133, 134, 135, 136, 137, 138, 139, 140 & 0, 1, 1, 2, 1, 2, 2, 3, 1, 3 & 4, 11, 29, 106, 214, 283, 296, 298, 360, 497 \\ \hline
141, 142, 143, 144, 145, 146, 147, 148, 149, 150 & 1, 2, 5, 15, 52, 203, 877, 4140, 21147, 115975 & 72, 106, 139, 165, 171, 192, 199, 429, 453, 477 \\ \hline
151, 152, 153, 154, 155, 156, 157, 158, 159, 160 & 2, 3, 5, 7, 11, 13, 17, 19, 23, 29 & 187, 218, 260, 295, 301, 314, 379, 410, 452, 469 \\ \hline
161, 162, 163, 164, 165, 166, 167, 168, 169, 170 & 1, 11, 21, 1211, 111221 & 29, 63, 95, 140, 150, 190, 221, 437, 482, 491 \\ \hline
171, 172, 173, 174, 175, 176, 177, 178, 179, 180 & 2, 3, 5, 7, 11, 13, 17, 19, 23, 29 & 3, 11, 84, 144, 156, 177, 188, 199, 229, 284 \\ \hline
181, 182, 183, 184, 185, 186, 187, 188, 189, 190 & 1, 2, 4, 8, 16, 32, 64, 128, 256, 512 & 26, 94, 98, 137, 176, 301, 323, 330, 372, 444 \\ \hline
191, 192, 193, 194, 195, 196, 197, 198, 199, 200 & 1, 3, 7, 15, 31, 63, 127, 255, 511, 1023 & 39, 81, 88, 210, 215, 378, 416, 430, 439, 490
\end{tabular}
}
\end{table}

%% Moved from the old two last sections

%{\color{red}
\subsection{Practical Applications and Integration into AI Development}

\subsubsection{SuperARC as a Development Tool}

SuperARC is designed not merely as a benchmark for publication leaderboards, but as a diagnostic tool for AI development pipelines:
\begin{itemize}
    \item Phase 1 - Architecture Design: During model architecture exploration, SuperARC scores provide early signals of genuine generalisation capability. Unlike human-centric benchmarks that may show improvement through memorisation of larger training sets, SuperARC performance improvement indicates enhanced algorithmic reasoning capacity.
    \item Phase 2 - Training Monitoring: We recommend tracking SuperARC performance throughout training alongside traditional metrics. Divergence patterns reveal critical information:
    \begin{itemize}
        \item Improving human-centric scores and stable/improving SuperARC: suggests healthy learning;
        \item Improving human-centric scores and degrading SuperARC: suggests increasing memorisation bias;
        \item Both degrading: suggests fundamental training instability.
    \end{itemize}
    \item Phase 3 - Model Selection: When choosing between model candidates, SuperARC provides an orthogonal evaluation dimension. A model with slightly lower human-centric performance scores but superior SuperARC performance may be preferable for applications requiring reasoning beyond training distribution (e.g., scientific discovery, mathematical problem-solving, code synthesis for novel tasks).
\end{itemize}

\subsubsection{Implications for Training Paradigms}

Our findings that LLMs show fragility and regression despite scale increases suggest current training paradigms are insufficient for AGI-level capabilities. The current paradigm is that by scaling data and, therefore, parameters, an improvement is expected in benchmarks. On the other hand, SuperARC reveals that more data coupled with more parameters may lead to better pattern matching, which is fundamentally different from better reasoning.

Based on the SuperARC framework, some shifts are recommended:

\begin{itemize}
    \item Synthetic data integration: Incorporate algorithmically generated sequences with known complexity measures into training corpora;
    \item Hybrid architectures: Our neurosymbolic baseline's success suggests combining neural pattern recognition with explicit symbolic reasoning modules;
    \item Curriculum complexity: Structure training to progressively increase algorithmic complexity (Kolmogorov complexity) rather than just data volume;
    \item Evaluation-driven development: Use SuperARC regression as a stopping criterion or trigger for training procedure modification.
\end{itemize}

\subsubsection{Cost-Benefit Analysis for Adoption}

Regarding the cost-benefit of adopting SuperARC, we argue that the benefits are potentially high, with negligible additional costs. This comes from the fact that sequences adherent to SuperARC can be generated algorithmically at minimal cost and the whole framework evaluation requires standard inference infrastructure. 
Furthermore, open-source reference implementations are hereby provided.

By incorporating SuperARC in standard training pipelines, it will be possible to early detect memorisation biases and avoid costly training runs that improve benchmarks but not fundamental capabilities. 
Additionally, the framework could guide architecture decisions with additional signal beyond parameter efficiency. Overall, organizations currently investing billions in model training could allocate <1\% of computational budget to continuous SuperARC evaluation, potentially saving resources by identifying unproductive scaling directions earlier.

%{\color{red}
\subsubsection{Interpreting SuperARC Performance: A Continuous Metric, Not a Pass/Fail Test}

A critical question arises: what would it mean for a system to ``pass'' SuperARC, and at what point (ten tasks, one hundred tasks, one thousand tasks) would such passing occur? This question reveals a fundamental aspect of our framework that requires explicit clarification. SuperARC is not a test to be passed or failed with a fixed threshold, but rather a continuous metric to be used alongside other pillars of intelligence assessment, forming part of a multidimensional evaluation framework.

The framework is more analogous to established continuous measures like mean squared error in regression, perplexity in language modeling, or classification accuracy in computer vision. Just as there is no universal threshold at which mean squared error becomes ``acceptable'' (i.e., the metric must be as low as possible and interpreted in context relative to baselines, alternative approaches, and task requirements), SuperARC performance must be understood comparatively rather than absolutely. A mean squared error of 0.5 tells us nothing in isolation but becomes meaningful when compared to the irreducible noise level in the data, to the performance of alternative models, or to the requirements of downstream applications. Similarly, a model achieving sixty percent accuracy on sequences of moderate algorithmic complexity carries little information on its own.

SuperARC performance becomes interpretable through several forms of contextual comparison. First, comparison to human performance on identical sequences provides calibration (if humans achieve seventy-five percent accuracy on a set of sequences while a model achieves only forty percent, this suggests the model lacks reasoning capabilities that humans routinely employ). Second, comparison across versions of the same model architecture reveals whether development is progressing toward or regressing from algorithmic competence. Our findings for ChatGPT, for example, shows that version 5 regressed when compared to version 4.5, while improving on traditional benchmarks. This becomes meaningful precisely through this temporal comparison. Third, comparison to alternative architectural approaches indicates whether observed limitations are fundamental to current paradigms or reflect specific implementation choices. The neurosymbolic baseline we present, which achieves near-perfect performance on constrained sequence classes, provides such a comparison point by demonstrating that strong SuperARC performance is achievable in principle with appropriate architectural commitments.

Regarding task coverage, a similar context-dependent and fundamentally open-ended situation occurs. Our current test suite samples from different algorithmic pattern classes including arithmetic progressions, recursive definitions, compositional rules, and nested structures. However, the infinite space of possible algorithms means no finite test set provides complete coverage in any absolute sense. This mirrors challenges in other domains: computer vision researchers spent years investigating which image classification benchmarks best predict performance on downstream tasks, gradually discovering through empirical investigation that performance on carefully curated datasets like ImageNet correlates with broader visual reasoning capabilities. Similarly, more research is needed to map the relationship between performance on our current SuperARC test suite and performance across the broader space of algorithmic reasoning tasks. Future work should investigate which algorithmic pattern classes are most predictive of general algorithmic competence, how performance generalizes across different complexity regimes, and what sample size provides stable estimates of reasoning capability.

Critically, even perfect performance on an arbitrarily large SuperARC test suite would not indicate general intelligence or superintelligence. Such performance would demonstrate robust algorithmic reasoning capability within formal domains, but would say nothing about social intelligence, embodied cognition, common sense understanding, goal formation, value alignment, or robustness to real-world distribution shifts. A system could theoretically achieve perfect SuperARC scores while completely lacking the ability to navigate physical environments, understand human emotions, form appropriate goals, or reason about everyday situations that humans handle effortlessly. Conversely, a system that fails SuperARC demonstrates a critical gap in algorithmic abstraction capability, which calls into question claims of general intelligence even if the system performs well on conversational or knowledge-retrieval tasks.

SuperARC thus functions primarily as a necessary-condition test rather than a sufficient-condition test. Poor performance provides strong evidence against claims of advanced algorithmic reasoning capability, while strong performance is necessary but far from sufficient for claims of general intelligence. 
This asymmetry is deliberate and valuable. 
In the current landscape where AI systems frequently achieve impressive performance on human-centric benchmarks while their fundamental reasoning capabilities remain unclear, a tool that can rule out deep algorithmic competence serves an important function. 
A complete assessment of general or superintelligence would require SuperARC alongside complementary metrics measuring social reasoning, embodied cognition, common sense understanding, causal reasoning, goal formation, value alignment, and numerous other dimensions we do not attempt to capture.

We emphasize that interpreting SuperARC as a continuous, contextual metric rather than a pass/fail test has important implications for how the benchmark should be used in practice. Researchers should report detailed performance breakdowns across complexity levels and pattern types rather than single summary statistics. Comparisons should always include relevant baselines such as human performance (just like the ARC challenge \cite{Chollet2019MeasureIntelligence}), previous model versions, alternative architectures, and theoretical limits where available. Claims about model capabilities should be carefully scoped to the specific dimension measured rather than extrapolated to general intelligence. Most importantly, SuperARC scores should be presented as one data point within a broader capability profile, not as a comprehensive assessment of intelligence.

%{\color{red}
\subsection{Implications for AI Policy and Governance}

\subsubsection{SuperARC and the AGI Assessment Challenge}

As AI systems approach and potentially exceed human-level performance on specific benchmarks, policymakers face a critical question: How to distinguish genuinely general intelligence from narrow systems optimised for human-centric tasks? This distinction carries serious implications for several actions, such as proposing safety protocols and oversight requirements, managing resource allocation in AI safety research, creating public communication about AI capabilities and risks and also building international coordination on transformative AI.

SuperARC addresses these challenges by providing a human-agnostic, open-ended assessment framework grounded in algorithmic information theory rather than human performance norms.

\subsubsection{Beyond Benchmark Gaming}\label{sectionBenchContami}

A critical vulnerability in current AI governance discussions is the reliance on benchmarks that can be ``solved'' through data contamination or targeted optimisation. Our findings reveal this problem empirically:

\begin{itemize}
    \item Model regression on SuperARC despite benchmark improvement: The ChatGPT case is emblematic, where newer versions (5 vs 4.5) showed benchmark improvement but SuperARC regression, suggesting apparent progress may mask fundamental limitations;
    \item Memorisation vs. reasoning gap: Current benchmarks increasingly measure memorisation of human knowledge rather than reasoning capacity;
    \item Transparency deficit: Without algorithmic reasoning assessment, stakeholders cannot distinguish models that truly understand from those that mimic understanding.
\end{itemize}

Based on these findings, regulatory frameworks should require dual assessment: human-centric benchmarks for practical capability measurement and algorithmic benchmarks (like SuperARC) for fundamental reasoning assessment.

\subsubsection{Capability-Based Governance Triggers}

Current AI policy proposals often use compute thresholds or parameter counts as triggers for enhanced oversight. SuperARC suggests an alternative or complementary approach based on demonstrated capabilities. For example, a simple model oversight ranking would be:

\begin{itemize}
    \item Tier 1 (current LLMs): Human-benchmark proficiency without algorithmic generalization, implying standard deployment protocols;
    \item Tier 2 (emerging systems): Combined human-benchmark and SuperARC proficiency, implying the need for enhanced monitoring and safety testing;
    \item Tier 3 (hypothetical AGI): Superhuman performance on both dimensions, implying the need for maximum scrutiny and safety protocols.
\end{itemize}

This approach focuses governance on demonstrated capabilities rather than proxy measures (compute, parameters), directly addressing the risks policymakers actually care about.

\subsubsection{International Coordination and Standards}

SuperARC's human-agnostic nature makes it particularly suitable for international AI governance coordination, as unlike benchmarks based on human knowledge (which may reflect cultural biases), algorithmic reasoning is culturally invariant. 
Moreover, binary sequences and mathematical structures transcend linguistic barriers while being objectively verifiable, facilitating international agreement on capability assessments. All these features of SuperARC can be continuously updated and enhanced, since the infinite hierarchy of algorithmic complexity enables creating arbitrarily difficult test suites. 

Thus, we propose SuperARC as a candidate foundation for international AI capability assessment standards, complementing region-specific practical benchmarks.

\subsubsection{Addressing the ``Perception of Mastery'' Problem}

Our conclusion that current LLMs are tools optimised for the perception of mastery over human language has direct policy implications. If systems appear highly capable on human-facing tasks while lacking fundamental reasoning abilities, public and policymaker perceptions may not reflect actual capabilities, leading to either over- or under-regulation. 
In addition, systems may pass safety evaluations based on human-centric criteria while possessing unknown failure modes in algorithmic reasoning domains. These hidden issues may lead to investments flowing toward apparent capability (benchmark performance) rather than genuine capability (algorithmic reasoning), which is a major problem especially when public resources are employed to support training and model development tasks.

\subsubsection{Research Priorities}

Our findings suggest policy support should prioritize research into architectures that combine neural and symbolic reasoning (as demonstrated by our baseline). 
Furthermore, both the theoretical and empirical evidences presented indicate the need to develop training paradigms that enhance algorithmic generalisation as well as investigating why scaling current approaches improves mimicry but not reasoning. Based on these findings, funding mechanisms could explicitly require grantees to report both human-centric and algorithmic reasoning metrics, incentivizing balanced progress.
%%%

\end{document}